%% file: main.tex
\documentclass{article} 
\usepackage{arxiv}

\usepackage[utf8]{inputenc} 
\usepackage[T1]{fontenc}    
\usepackage{hyperref}       
\usepackage{url}            
\usepackage{booktabs}       
\usepackage{amsfonts}       
\usepackage{nicefrac}       
\usepackage{microtype}      
\usepackage[shortlabels]{enumitem}
\usepackage[title]{appendix}
\usepackage[numbers]{natbib}
\setlist[enumerate]{leftmargin=*}
\setlist[description]{leftmargin=*}
\usepackage{fonts}
\usepackage{multirow}
\usepackage{todonotes}
\usepackage{xcolor}
\input{common_packages.tex}
\input{common_commands.tex}

\newcommand{\idmat}{\mathrm{Id}}
\newcommand{\infnet}{I_\mathrm{net}}

\newcommand{\numconv}{\mathfrak{n}_c}
\newcommand{\paren}[1]{\left ( #1\right)}

\newcommand{\mattinnorm}[1]{\norm{#1}_n}
\newcommand{\mattidnorm}[1]{\norm{#1}_D}
\newcommand{\mattimnorm}[1]{\norm{#1}_m}

\newcommand{\mattidifmat}{\mathscr{D}}

\newcommand{\nnet}{\mathcal{NN}}
\newcommand{\Prob}{\mathbb{P}}
\renewcommand{\vec}[1]{\boldsymbol{#1}}

\newtheorem{lemma}{Lemma}
\newtheorem{theorem}{Theorem}
\newtheorem{corollary}{Corollary}
\newtheorem{proposition}{Proposition}

\theoremstyle{definition}
\newtheorem{definition}{Definition}
\newtheorem{example}{Example}

\theoremstyle{remark}
\newtheorem{remark}{Remark}

\usepackage{hyperref}
\hypersetup{
  colorlinks,
  citecolor=blue,
  linkcolor=red,
  urlcolor=blue}

\usepackage{url}

\newcommand{\MAPrm}[1]{}

\newcommand{\bmat}[1]{\begin{bmatrix} #1 \end{bmatrix}}

\definecolor{udpink}{rgb}{0.98431373,0.60392157,0.6}
\definecolor{udgreen}{rgb}{0.69803922,0.8745098,0.54117647}
\definecolor{udblue}{rgb}{0.65098039,0.80784314,0.89019608}
\definecolor{udbrown}{rgb}{0.69411765,0.34901961,0.15686275}
\definecolor{udviolet}{rgb}{0.41568627, 0.23921569, 0.60392157}
\definecolor{udorange}{rgb}{1.0,0.49803922,0.0}

\title{Globally Injective ReLU Networks}

\author{%
  Michael Puthawala \\
  Rice University \\
  \texttt{map19@rice.edu} \\
  \And
  Konik Kothari \\
  University of Illinois at\\
  Urbana-Champaign \\
  \texttt{kkothar3@illinois.edu} \\
  \AND
  Matti Lassas \\
  University of Helsinki \\
  \texttt{matti.lassas@helsinki.fi} \\
  \And
  Ivan Dokmani\'c \\
  University of Basel \\
  \texttt{ivan.dokmanic@unibas.ch} \\
  \And
  Maarten de Hoop \\
  Rice University \\
  \texttt{mdehoop@rice.edu} \\
}

\begin{document}

\maketitle

\begin{abstract}
Injectivity plays an important role in generative models where it enables inference; in inverse problems and compressed sensing with generative priors it is a precursor to well posedness. We establish sharp characterizations of injectivity of fully-connected and convolutional ReLU layers and networks. First, through a layerwise analysis, we show that an expansivity factor of two is necessary and sufficient for injectivity by constructing appropriate weight matrices. We show that global injectivity with iid Gaussian matrices, a commonly used tractable model, requires larger expansivity between 3.4 and 10.5. We also characterize the stability of inverting an injective network via worst-case Lipschitz constants of the inverse. We then use arguments from differential topology to study injectivity of deep networks and prove that any Lipschitz map can be approximated by an injective ReLU network. Finally, using an argument based on random projections, we show that an end-to-end---rather than layerwise---doubling of the dimension suffices for injectivity. Our results establish a theoretical basis for the study of nonlinear inverse and inference problems using neural networks.
\end{abstract}

\section{Introduction}

Many applications of deep neural networks require inverting them on their range. Given a neural network $N : \cZ \to \cX$, where $\cX$ is often the Euclidean space $\R^m$ and $\cZ$ is a lower-dimensional space, the map $N^{-1} : N(\cZ) \to \cZ$ is only well-defined when $N$ is injective. The issue of injectivity is particularly salient in two applications: generative models and (nonlinear) inverse problems. 

Generative networks model a complicated distribution $p_X$ over $\cX$ as a pushforward of a simple distribution $p_Z$ through $N$. Given an $x$ in the range of $N$, inference requires computing $p_Z(N^{-1}(x))$ which is well-posed only when $N$ is injective. 
In the analysis of inverse problems~\citep{arridge2019solving}, uniqueness of a solution is a key concern; it is tantamount to injectivity of the forward operator. Given a forward model that is known to yield uniqueness, a natural question is whether we can design a neural network that approximates it arbitrarily well while preserving uniqueness. Similarly, in compressed sensing with a generative prior $N$ and a possibly nonlinear forward operator $A$ injective on the range of $N$, we seek a latent code $z$ such that $A(N(z))$ is close to some measured $y = A(x)$. This is again only well-posed when $N$ can be inverted on its range \citep{balestriero2020max}. Beyond these motivations, injectivity is a fundamental mathematical property with numerous implications. We mention a notable example: certain injective generators can be trained with sample complexity that is polynomial in the image dimension \citep{bai2018approximability}.

\subsection{Our Results}

In this paper we study injectivity of neural networks with ReLU activations. Our contributions can be divided into layerwise results and multilayer results.

\paragraph{Layerwise results.} For a ReLU layer $f : \R^n \to \R^m$ we derive sufficient and necessary conditions for invertibility on the range. For the first time, we construct deterministic injective ReLU layers with minimal expansivity $m = 2n$. We then derive specialized results for convolutional layers which are given in terms of filter kernels instead of weight matrices. We also prove upper and lower bounds on minimal expansivity of globally injective  layers with iid Gaussian weights. This generalizes certain existing pointwise results (Theorem \ref{theorem-injGauws} and Appendix \ref{sec:proof-of-thm-2}). We finally derive the worst-case inverse Lipschitz constant for an injective ReLU layer which yields stability estimates in applications to inverse problems. 

\paragraph{Multilayer results.} A natural question is whether injective models are sufficiently expressive. Using  techniques from differential topology we prove that injective networks are \textit{universal} in the following sense: if a neural network $N_1:\cZ\rightarrow \Rea^{2n+1}$ models the data, $\cZ\subset \Rea^n$, then we can approximate $N_1$ by an injective neural network $N_2:\cZ\rightarrow \Rea^{2n+1}$. As $N_2$ is injective, the image set $N_2(\cZ)$ is a Lipschitz manifold. We then use an argument based on random projections to show that an \emph{end-to-end} expansivity by a factor of $\approx 2$ is enough for injectivity in ReLU networks, as opposed to layerwise 2-expansivity implied by the layerwise analysis.

We conclude with preliminary numerical experiments to show that imposing injectivity improves inference in GANs while preserving expressivity.

\begin{figure}
    \centering
    \includegraphics[width=.7\linewidth]{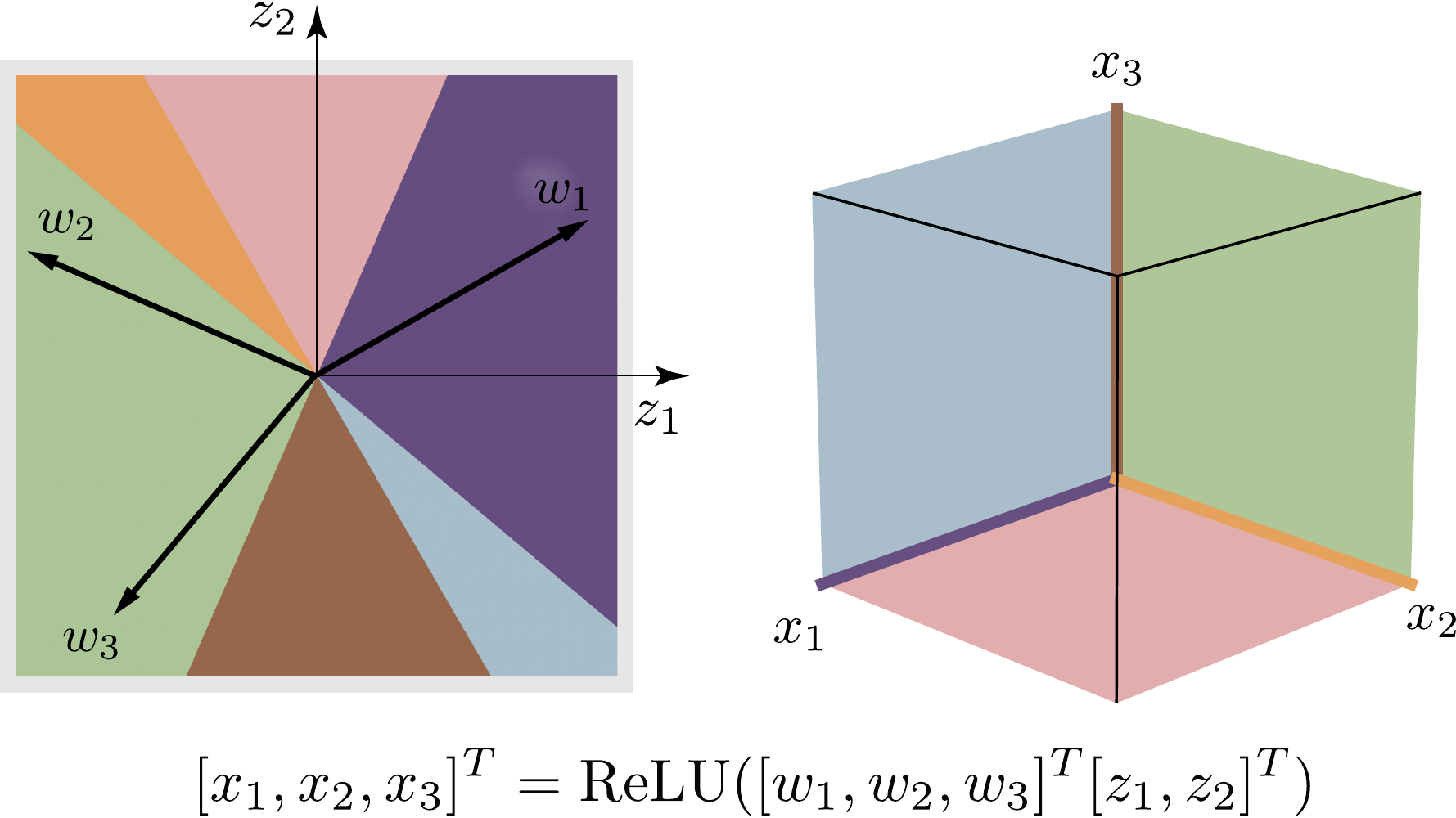}
    \caption{An illustration of a ReLU layer $N: \R^2 \to \R^3$, $x = N(z)$, that is not globally injective. Differently colored regions in the $z$-space are mapped to regions of the same color in the $x$-space. While $N$ is locally injective in the \textcolor{udpink}{pink}, \textcolor{udblue}{blue} and \textcolor{udgreen}{green} wedges in $z$-space, the \textbf{}\textcolor{udorange}{orange}, \textcolor{udbrown}{brown}, and \textcolor{udviolet}{violet} wedges are mapped to coordinate axes. $N$ is thus \emph{not} injective on these wedges. This prevents construction of an inverse in the range of $N$.}
    \vspace{-3mm}
    \label{fig:nonglobinj}
\end{figure}

\subsection{Why Global Injectivity?}

The attribute ``global'' relates to global injectivity of the map $N : \cZ \to \R^m$ on the low-dimensional latent space $\cZ$, but it does not imply global invertibility over $\R^m$, only on the range $N(\cZ) \subset \R^m$. If we train a GAN generator to map iid normal latent vectors to real images from a given distribution, we expect that any sampled latent vector generates a plausible image. We thus desire that any $N(z)$ be produced by a unique latent code $z \in \cZ$. This is equivalent to global injectivity, or invertibility on the range. Our results relate to the growing literature on using neural generative models for compressed sensing~\citep{bora2017compressed,daskalakis2020constant}. They parallel the related guarantees for sparse recovery where the role of the low-dimensional latent space is played by the set of all $k$-sparse vectors. One then looks for matrices which map \textit{all} $k$-sparse vectors to distinct measurements \citep{10.5555/2526243}. As an example, in the illustration in Figure \ref{fig:nonglobinj} images coresponding to latent codes in \textcolor{udorange}{orange}, \textcolor{udbrown}{brown}, and \textcolor{udviolet}{violet} wedges cannot be compressively sensed. Finally, a neural network is often trained to directly reconstruct an image $x$ from its (compressive) low-dimensional measurements $y = A(x)$ without introducing any generative models. In this case, whenever $A$ is Lipschitz, it is immediate that the learned inverse must be injective.

\subsection{Related Work}

Closest to our work are the papers of \citep{bruna2013signal}, \citep{hand2018phase} and \citep{lei2019inverting}. \citep{bruna2013signal} study injectivity of pooling motivated by the problem of signal recovery from feature representations. They focus on $\ell^p$ pooling layers; their Proposition 2.2 gives a criterion similar to the DSS (Definition \ref{def:directed-spanning-set}) and bi-Lipschitz bounds for a ReLU layer (similar to our Theorem \ref{thm:inv-lip:global}). Unlike Theorems \ref{thm:relu-w-injectivity} and \ref{thm:inv-lip:global}, their criterion and Lipschitz bound are in some cases not precisely aligned with injectivity; see Appendix \ref{sec:comparison-w-bruna}. 

Compressed sensing with GAN priors requires inverting the generator on its range \citep{bora2017compressed, shah2018solving, wu2019deep, mardani2018deep, hand2018phase, aberdam2020and,daskalakis2020constant}. \citep{lei2019inverting,daskalakis2020constant} replace the end-to-end inversion by the faster and more accurate layerwise inversion when each layer is injective. They show that with high probability a ReLU layer with an iid normal weight matrix can be inverted about a fixed point if the layer expands at least by a factor of 2.1, and that the inverse about a fixed point can be computed with the algorithm from \citep{bora2017compressed}. This result is related to our Theorem \ref{theorem-injGauws} which gives conditions for global injectivity or layers with random matrices. \citep{hand2017global} show that when the weights of a ReLU network obey a certain weighted distribution condition, the loss function for the inversion has a strict descent direction everywhere except in a small set. The condition is in particular satisfied by random matrices with expansivity $n_j = \Omega(n_{j - 1} \log n_{j - 1})$, where $n_{j}$ is the output dimension of layer $j$.

 A continuous analogy of our convolutional construction (Definition \ref{def:discrte-convolution}) was considered by \citep{mallat2018phase}. They show that $\relu$ acts as a phase filter and that the layer is bi-Lipschitz and hence injective when the filters have a diverse phase and form a frame. Discretizing their model gives a statement related to Corollary \ref{cor:minimal-expansivity} and Theorem \ref{thm:suff-conv-injectivity}.  

Injectivity is automatic in invertible neural networks such as normalizing flows \citep{grover2018flow,kingma2018glow,grathwohl2018ffjord}. Specialized architectures with simple Jacobians give easy access to the likelihood \citep{dinh2014nice, dinh2016density, gomez2017reversible}, which facilitates application to inverse problems \citep{ardizzone2018analyzing}. Inference with GANs can be achieved by jointly training a generative network and its inverse \citep{donahue2016adversarial,dumoulin2016adversarially}, which is well-defined when the generator is injective. Relaxed injective probability flows resemble GAN generators but are trained via approximate maximum likelihood \citep{kumar2020regularized}. Injectivity is promoted by keeping the smallest singular value of the Jacobian away from zero at the training examples, a necessary but not  sufficient condition. In general, Jacobian conditioning improves GAN performance \citep{heusel2017gans, odena2018generator}. Finally, lack of injectivity interferes with disentanglement \citep{chen2016infogan,lin2019infogan}. In this context, injectivity seems to be a natural heuristic to increase latent space capacity without increasing its dimension \citep{brock2016neural}.

\subsection{Notation}
\label{sec:notation}
Given a matrix $W \in \Rea^{m \times n}$ we define the notation $w \in W$ to mean that $w \in \Rea^{n}$ is a row vector of $W$. For a matrix $W \in \Rea^{m \times n}$ with rows $\{w_j\}_{j=1}^m$ and $x \in \Rea^n$, we write
\begin{align}
    S(x,W) \coloneqq \left\{ j \in [[m]] \colon \innerprod{w_j}{x} \geq 0 \right\}
\end{align}
and $S^c(x, W)$ for its complement, with $[[m]] = \{1, \ldots, m \}$. Further, given a matrix $W$ and index set $\cI$ we define the notation $\restr{W}{\cI}$ to be a matrix whose $i$'th row is the same as $W$ if $i \in \cI$, and a row of zeroes otherwise. We let $\nnet(n,m,L,\vec n)$ be the family of functions $N_{\theta}:\Rea^n\to \Rea^m$ of the form 
\begin{equation}
    \label{eqn:deep-network-definition}
    N(z) = W_{L + 1}\phi_L(W_L \cdots \phi_2(W_2\phi_1(W_1 z + b_1) + b_2) \cdots + b_L)
\end{equation}
Indices $\ell = 1,\ldots, L$ index the network layers,  $b_{\ell} \in \Rea^{n_{\ell+1}}$ are the bias vectors, $W_{\ell} \in \Rea^{n_{\ell+1} \times n_{\ell}}$ are the weight matrices with $n_1=n$, $n_{L}=m$, and $\phi_\ell$ are the nonlinear activation functions. We will denote $\relu(x) = \max(x, 0)$. We write $\vec n=(n_1,n_2,\dots n_{L-1})$ and $\theta=(W_1,b_1,\dots,W_L,b_L)$ for the parameters that determine the function $N_{\theta}$. We also write $\nnet(n,m,L)=\bigcup_{\vec n\in \mathbb{Z}^{L-1}}\nnet(n,m,L,\vec n)$ and  $\nnet(n,m)=\bigcup_{L\in \mathbb{Z}}\nnet(n,m,L)$.

In Section \ref{sec:concolutional-layer} we use multi-index notation. Multi-indices are $p$-tuples denoted by the capitol letters, such as $I, J, K, N$ and $M$. In the case that $p = 1$, the operations between multi-indices is equivalent to the equivalent operations for scalars. If $I$ is an $p$-dimensional multi-index, then
\begin{align}
    I \coloneqq (I_1,\dots,I_p) \in \Nat^p
\end{align}
where $I_p$ refers to the $p$'th index of $I$. We use the notation $c \in \Rea^N$ to refer to $c \in \Rea^{N_1} \times \dots \times \Rea^{N_p}$, and
\(
    \sum_{I = 1}^N c_I = \sum_{I_1 = 1}^{N_1}\dots\sum_{I_p = 1}^{N_p} c_{I_1,\dots,I_p},
\) for $I = (I_1, \dots, I_p)$. The symbol $1$ can refer to number $1$ or a $p$-tuple $(1,\dots, 1)$, and is clear from context. We likewise use the notation $I \leq J$ to mean that $I_k \leq J_k$ for all $k = 1,\dots,p$ and likewise for $\geq, <$ and $>$. If $I \not \leq J$, then there is at least one $k$ such that $I_k > J_k$.

We use the notation $\norm{\cdot}_{p,m}$ and $\norm{\cdot}_{p,n}$ to denote the standard $l_p$ norms for vectors in $\Rea^m$ and $\Rea^n$ respectively.

\section{Layerwise Injectivity of ReLU Networks}

For a one-to-one activation function such as a leaky $\relu$ or a $\tanh$, it is easy to see that injectivity of $x \mapsto W_i x$ implies the injectivity of the layer. We therefore focus on non-injective ReLU activations.

\subsection{Directed Spanning Set}
\label{sub:dss}

\begin{figure}
    \centering
    \begin{subfigure}{.6\linewidth}
        \centering
        \includegraphics[width=\linewidth]{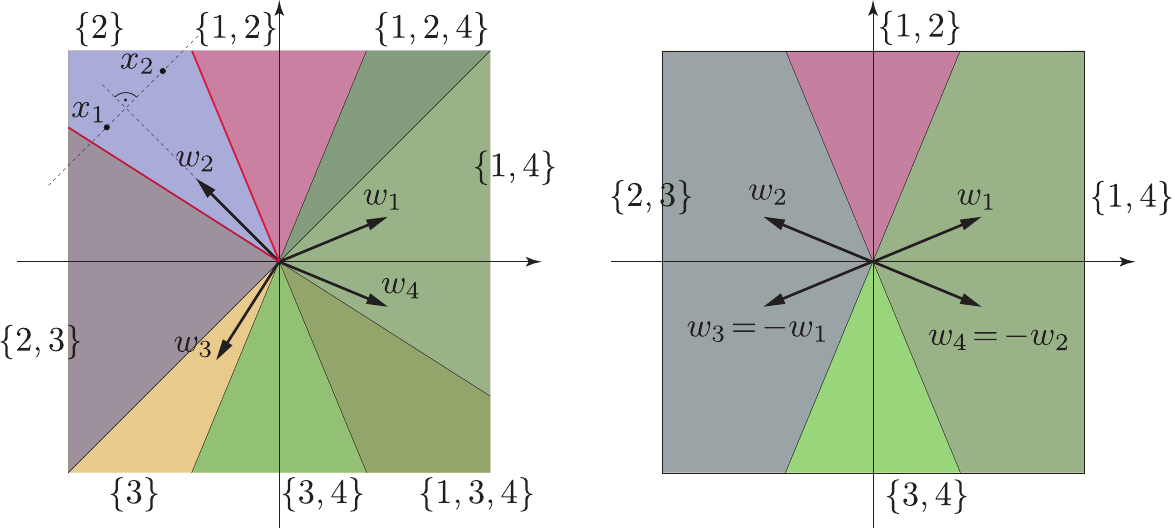}
        \label{fig:dss-illustration}
    \end{subfigure}
    \begin{subfigure}{.36\linewidth}
        \centering
        \includegraphics[width=\linewidth]{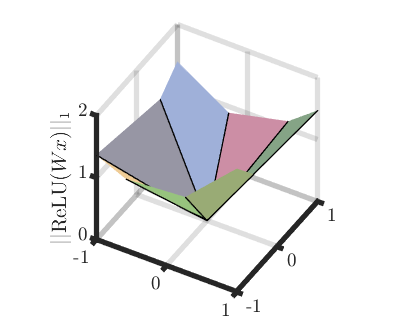}
        \label{fig:folding-landscape}
    \end{subfigure}
    \caption{Illustration of the DSS definition. \textit{Left:} A configuration of 4 vectors in $\R^2$ that do not have a DSS w.r.t. all $x\in \Rea^2$. In this case, the vectors do not generate an injective layer. The set of labels indicate which $w_j$ have positive inner product with vectors in the wedge; there are two wedges with only one such $\set{w_j}$; we have $\relu(Wx_1) = \relu(Wx_2)$. \textit{Center:} A configuration where four vectors have a DSS for all $x\in \Rea^2$. These vectors correspond to a minimally-expansive injective layer; see Corollary \ref{cor:minimal-expansivity}. \textit{Right:} A plot of $\norm{\relu(Wx)}_1$ where $W$ is given as in the left figure. Note that $x \mapsto \relu(Wx)$ is linear in every wedge.}
    \label{fig:inverse-lip-condition}
\end{figure}

Unlike in the case of a one-to-one activation, when $\phi(x) = \relu(x)$, $x \mapsto \phi(Wx)$ cannot be injective for all $x \in \Rea^n$ if $W$ is a square matrix. Noninjectivity of $\relu(Wx)$ occurs in some simple situations, for example if $Wx$ is not full rank or if $Wx \leq 0$ element-wise for a nonzero $x$, but these two conditions on their own are not sufficient to characterize injectivity. Figure \ref{fig:inverse-lip-condition} gives an example of a weight matrix that yields a noninjective layer that has these two properties. In order to facilitate the full characterization of injectivity of ReLU layers, we define a useful theoretical device:

\begin{definition} [Directed Spanning Set]\label{def:directed-spanning-set}
    Let $W \in \Rea^{n \times m}$. We say that $W$ has a directed spanning set (DSS) of $\Omega \subset \Rea^n$ with respect to a vector $x \in \Rea^n$ if there exists a $\hat W_x$ such that each row vector of $\hat W_x$ is a row vector of $W$,
    \begin{align}
        \innerprod{x}{w_i} \geq 0 \quad \text{for all}~w_i \in \hat W_x,
    \end{align}
    and $\Omega \subset \Span(\hat W_x)$. When omitted, $\Omega$ is understood to be $\Rea^n$.
\end{definition}

For a particular $x \in \Rea^n$, it is not hard to verify if a given $W = \{w_i\}_{i = 1,\dots,m}$ has a DSS of $\Omega$. One can simply let 
\(
    \hat W_x  = \{ w_i \in \R^n \colon \innerprod{x}{w_i} \geq 0\}.
\)
Then, $W$ has a DSS of $\Omega$ with respect to $x$ if and only if $\hat W_x$ spans $\Omega$, which can be checked efficiently. Note that having full rank is necessary for it to have a DSS of $\Rea^n$ w.r.t. \textit{any} $x \in \Rea^n$. It is however not sufficient. For example, $\idmat_n \in \Rea^{n \times n}$, the identity matrix in $\Rea^n$, is clearly full rank, but it doesn't have a DSS of $\Rea^n$ w.r.t. $x = (-1, 0, \dots, 0) \in \Rea^n$. We also note that in this paper we focus on globally injective networks, and so $\Omega$ is always taken to be either $\Rea^n$ or $\Rea^p$, where $p < n$. Further, this latter case is only discussed to eventually produce a statement showing that a matrix has a DSS on the full space $\Rea^n$. This only comes up in the context of our discussion on convolution layers in Section \ref{sec:concolutional-layer}, specifically Theorem \ref{thm:suff-conv-injectivity} and Lemma \ref{lem:domain-decomp}.

To check whether $W$ has a DSS for all $x \in \R^n$, note that $W$ partitions $\R^n$ into open wedges $S_k$,  $\R^n = \bigcup_k S_k$, with constant sign patterns. That is, for $x_1, x_2 \in S_k$, $\sign(Wx_1) = \sign(Wx_2)$. (See also the proof of Theorem \ref{theorem-injGauws} in Appendix \ref{sub:minexpGauss}.) Checking whether $W$ has a DSS for all $x$ is equivalent to checking that for every wedge there are at least $n$ vectors $W_k \subset W$ with positive sign, $\langle w, x \rangle > 0$ for $x \in S_k$, and that $W_k$ spans $\R^n$. Since the number of wedges can be exponential in $m$ and $n$ \citep{winder1966partitions} this suggests an exponential time algorithm. We also note the connection between DSS and spark in compressed sensing \citep{10.5555/2526243}, defined as the size of the smallest set of linearly dependent vectors in $W$. If every wedge has $n$ or more positive signs, then full spark $n + 1$ is sufficient for $W$ to have a DSS w.r.t all $x \in \R^n$. Computing spark is known to be NP-hard \citep{tillmann2013computational}; whether one can do better for DSS remains an open question.

\subsection{Fully Connected Layer}

The notion of a DSS immediately leads to our main result for fully connected layers.

\begin{theorem}[Conditions for Injectivity of $\relu(Wx)$] \label{thm:relu-w-injectivity}
    Let $W \in \Rea^{m \times n}$ where $1 < n \leq m$ be a matrix with row vectors $\{w_j\}_{j = 1}^m$, and $\relu(y) = \max(y, 0)$. The function $\relu(W(\cdot)) \colon \Rea^n \rightarrow \Rea^m$ is injective if and only if $W$ has a DSS w.r.t every $x \in \Rea^n$. 
\end{theorem}

The question of injectivity in the case when $b = 0$ and when $b \neq 0$ are very similar and, as Lemma \ref{lem:relu-w-plus-b-injectivity} shows, the latter question is equivalent to the former on a restricted weight matrix.

\begin{lemma}[Injectivity of $\relu(Wx + b)$] \label{lem:relu-w-plus-b-injectivity}
    Let $W \in \Rea^{m \times n}$ and $b \in \Rea^m$. The function $\relu(W(\cdot) + b) \colon \Rea^n \rightarrow \Rea^m$ is injective if and only if $\relu(W|_{b\geq 0}\cdot)$ is injective, where $W|_{b\geq 0} \in \Rea^{m \times n}$ is row-wise the same as $W$ where $b_i \geq 0$, and is a row of zeroes when $b_i < 0$.
\end{lemma}

\begin{remark}[Injectivity of $\relu(Wx)$, Positive $x$]
    Without normalization strategies, inputs of all but the first layer are element-wise non-negative. One can ask whether $\relu(Wx)$ is injective when $x$ is restricted to be element-wise non-negative. Following the same argument as in Theorem \ref{thm:relu-w-injectivity}, we find that $W$ must have a DSS w.r.t. $x \in \Rea^n$ for every $x$ that is element-wise non-negative.
    With normalizations strategies however (for example batch renormalization \citep{ioffe2017batch}) the full power of Theorem \ref{thm:relu-w-injectivity} and Lemma \ref{lem:relu-w-plus-b-injectivity} may be necessary. In Section \ref{sec:norm-training-runtime} we show that common normalization strategies such as batch, layer, or group normalization do not interfere with injectivity.
\end{remark}

\begin{remark}[Global vs Restricted Injectivity]
    Even when the conditions of Theorem \ref{thm:relu-w-injectivity} and Lemma \ref{lem:relu-w-plus-b-injectivity} are not satisfied the network might be injective in some $\cX \subset \Rea^n$. Indeed, $N(x) = \relu(\idmat_n \cdot)$ is in general not injective, but it is injective in $\cX = \{ x \in \Rea^n \colon x_i >0 \text{ for } i = 1,\dots,n\}$, a convex and open set. Theorem \ref{thm:relu-w-injectivity} and Lemma \ref{lem:relu-w-plus-b-injectivity} are the precise criteria for single-layer injectivity for all $x \in \Rea^n$.
\end{remark}

The above layerwise results imply a sufficient condition for injectivity of deep neural networks. If Theorem \ref{thm:relu-w-injectivity} applies to each layer of a network then the entire network must be injective end-to-end.

Note that layer-wise injectivity is sufficient but not necessary for injectivity of the entire network. Consider for example an injective layer $\relu(W)$, and 
\(
    N(x) = \relu( \idmat_m \cdot \relu(W(x))),
\)
where $I$ is the identity matrix. Clearly,
\begin{align}
    \relu( \idmat_m \cdot \relu(W(x))) = \relu(\relu(W(x))) = \relu(W(x))
\end{align}
so $N$ is injective, but it fails to be injective at the $\relu(I \cdot)$ layer.

An important implication of Theorem \ref{thm:relu-w-injectivity} is the following result on minimal expansivity.

\begin{corollary}[Minimal Expansivity] \label{cor:minimal-expansivity}
    For any $W \in \Rea^{m \times n}$, $\relu(W\cdot)$ is non-injective if $m < 2\cdot n$.  If $W \in \Rea^{2n \times n}$ satisfies Theorem \ref{thm:relu-w-injectivity} and Lemma \ref{lem:relu-w-plus-b-injectivity}, then up to row rearrangement $W$ can be written as 
    \begin{align}\label{eqn:cor:twice-expansive-factorization}
        W = 
        \begin{bmatrix}
            B\\
            -DB
        \end{bmatrix}
    \end{align}
    where $B, D \in \Rea^{n \times n}$ and $B$ is a basis and $D$ a diagonal matrix with strictly positive diagonal entries.
\end{corollary}

\begin{figure}
    \begin{center}
        \includegraphics[width=0.5\linewidth]{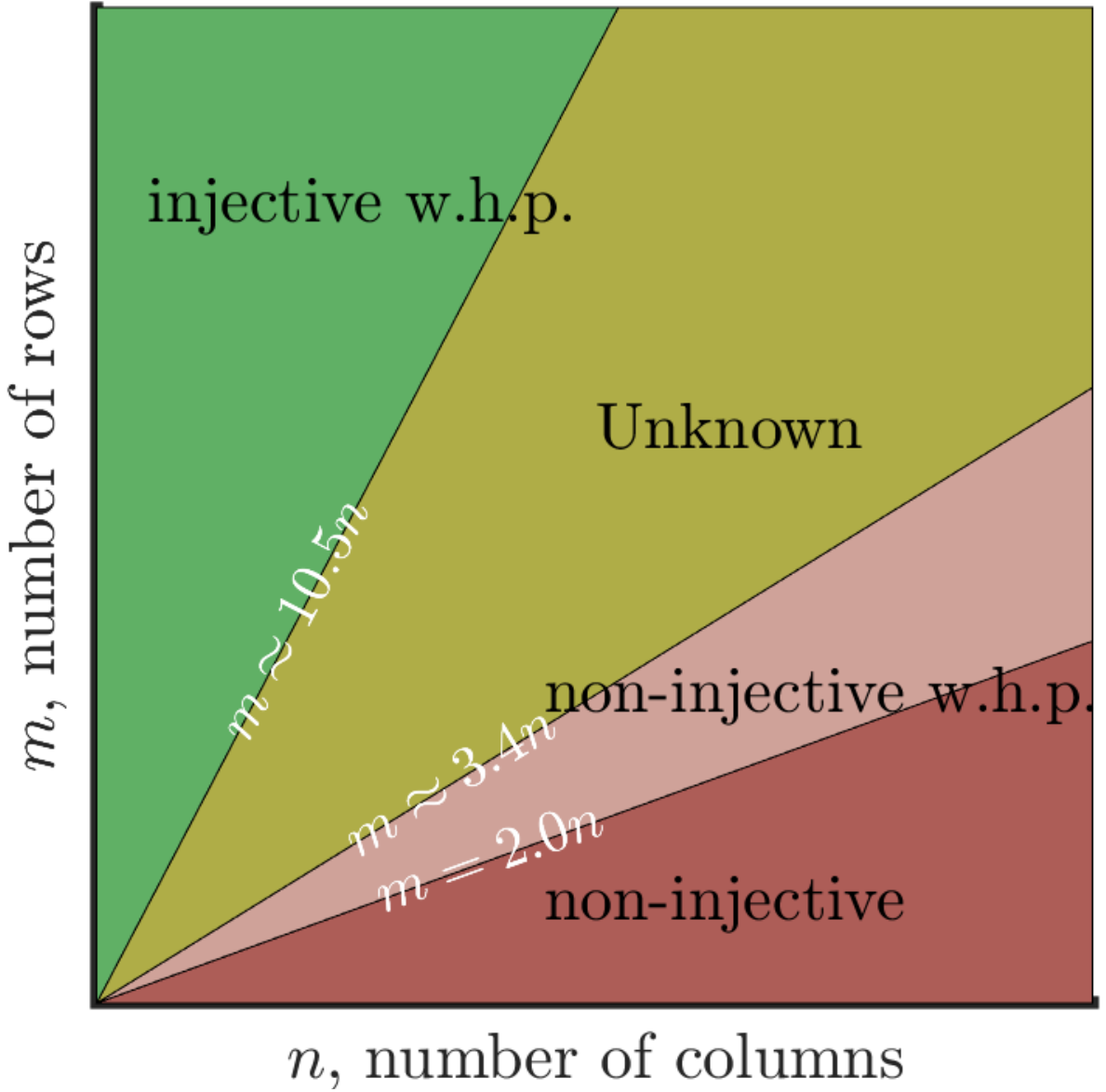}
    \end{center}
    \caption{A visualization of regions where asymptotics of $\cI(m,n)$ in Theorem \ref{theorem-injGauws} are valid.}
    \vspace{-6mm}
    \label{fig:injgausregion}
\end{figure} 

While $m < 2n$ immediately precludes injectivity, Corollary \ref{cor:minimal-expansivity} gives a simple recipe for the construction of minimally-expansive injective layers with $m = 2n$. To the best of our knowledge this is the first such result. Further, we can also use Corollary \ref{cor:minimal-expansivity} to build weight matrices for which $m > 2n$. For example, the matrix 
\begin{align*}
    W = \begin{bmatrix}
        B\\
        -DB\\
        M
    \end{bmatrix},
\end{align*}
where $B \in \Rea^{n \times n}$ has full rank, $D \in \Rea^{n \times n}$ is a positive diagonal matrix, and $M \in \Rea^{(m - 2n)\times n}$ is arbitrary, has a DSS w.r.t. every $x \in \Rea^n$. This method can be used in practice to generate injective layers by using standard spectral regularizers to ensure that $B$ is a basis \citep{cisse2017parseval}. In Section \ref{sec:universality} we will show that, in fact, an \textit{end-to-end} rather than layerwise doubling of dimension is sufficient for injectivity.

Previous work that sought injectivity uses random-weight models. \citep{lei2019inverting} show that a layer is invertible \emph{about a point} in the range provided that $m \geq 2.1 n$ and $W$ is iid Gaussian. In Appendix \ref{app:expansivity} we show that $m=2.1n$ is not enough to guarantee injectivity.  In fact, we show that with high probability an iid Gaussian weight matrix $W$ yields a \textit{globally} invertible ReLU layer, when $W$ is sufficiently expansive, and conversely does not satisfy Theorem \ref{thm:relu-w-injectivity} if it is not expansive by at least a factor of 3.4:

\begin{theorem}[Injectivity for Gaussian Weights] \label{theorem-injGauws}
    Let $\cI(m,n) = \mathbb{P}\{x \mapsto \relu(Wx)~\text{is injective} \}$ with the entries of $W$ iid standard normal and $c = m/n$ fixed. Then as $n \rightarrow \infty,$
    \begin{align*}
        \cI(m,n) \rightarrow 1 \text{ if } c \geq 10.5 \ \
        \text{and} \ \
        \cI(m,n) \rightarrow 0 \text{ if } c \leq 3.3
    \end{align*}
\end{theorem}

The parameter regions defined in Theorem \ref{theorem-injGauws} are illustrated in Figure \ref{fig:injgausregion}. We note that the study of neural networks with random weights is an important recent research direction \citep{pennington2017nonlinear, benigni2019eigenvalue}. \cite{pennington2017nonlinear} articulate this clearly, arguing that large complex systems, of which deep networks are certainly an example, can be profitably studied by approximating their constituent parts as random variables.

\subsection{Inverse Lipschitz Constant for Networks}

Because $\norm{\relu(Wx) - \relu(Wy)} \leq \norm{W}\norm{x - y}$, it is clear that $\relu(W\cdot)$ is Lipschitz with constant $\norm{W}$; whether the inverse is Lipschitz, and if so, with what Lipschitz constant, is less obvious. We can prove the following result (see Appendix \ref{sec:proof-of-fun-approx}):

\begin{theorem}[Global Inverse Lipschitz Constant] \label{thm:inv-lip:global}
    Let $W \in \Rea^{m \times n}$ have a DSS w.r.t. every $x \in \Rea^n$. There exists a $C(W) > 0$ such that for any $x_0, x_1 \in \Rea^n$,
    \begin{align}
        \label{eqn:inv-lip:global}
        \norm{\relu(Wx_0) - \relu(Wx_1)}_2 \geq C(W) \norm{x_0 - x_1}_2
    \end{align}
    where $C(W) = \tfrac{1}{\sqrt{2m}} \min_{x \in \Rea^n}\sigma(W|_{S(x,W)}) $, and $\sigma$ denotes the smallest singular value.
\end{theorem}

This result immediately yields stability estimates when solving inverse problems by inverting (injective) generative networks. We note that the $(2m)^{-1/2}$ factor is essential for our estimate; the naive ``smallest singular value'' estimate is too optimistic.
\begin{remark}
     We note here that that Theorem \ref{thm:inv-lip:global} gives a bound on the stability of leaky $\relu$ layers. Layers of the form $\lrelu_\alpha(x) = \relu(Wx) - \alpha\relu(-Wx)$ are always injective, provided that $W$ is full rank, but may be unstable if $\alpha \approx 0$. To be explicit, the concept of a DSS still applies to these activation functions, and if a leaky relu layer satisfies the criterion of Theorem \ref{thm:relu-w-injectivity}, then it automatically inherits the stability estimate of Theorem \ref{thm:inv-lip:global} which is independent of $\alpha$.
\end{remark}

\subsection{Convolutional Layer}
\label{sec:concolutional-layer}

\begin{figure}
    \begin{center}
        \includegraphics[width=0.5\linewidth]{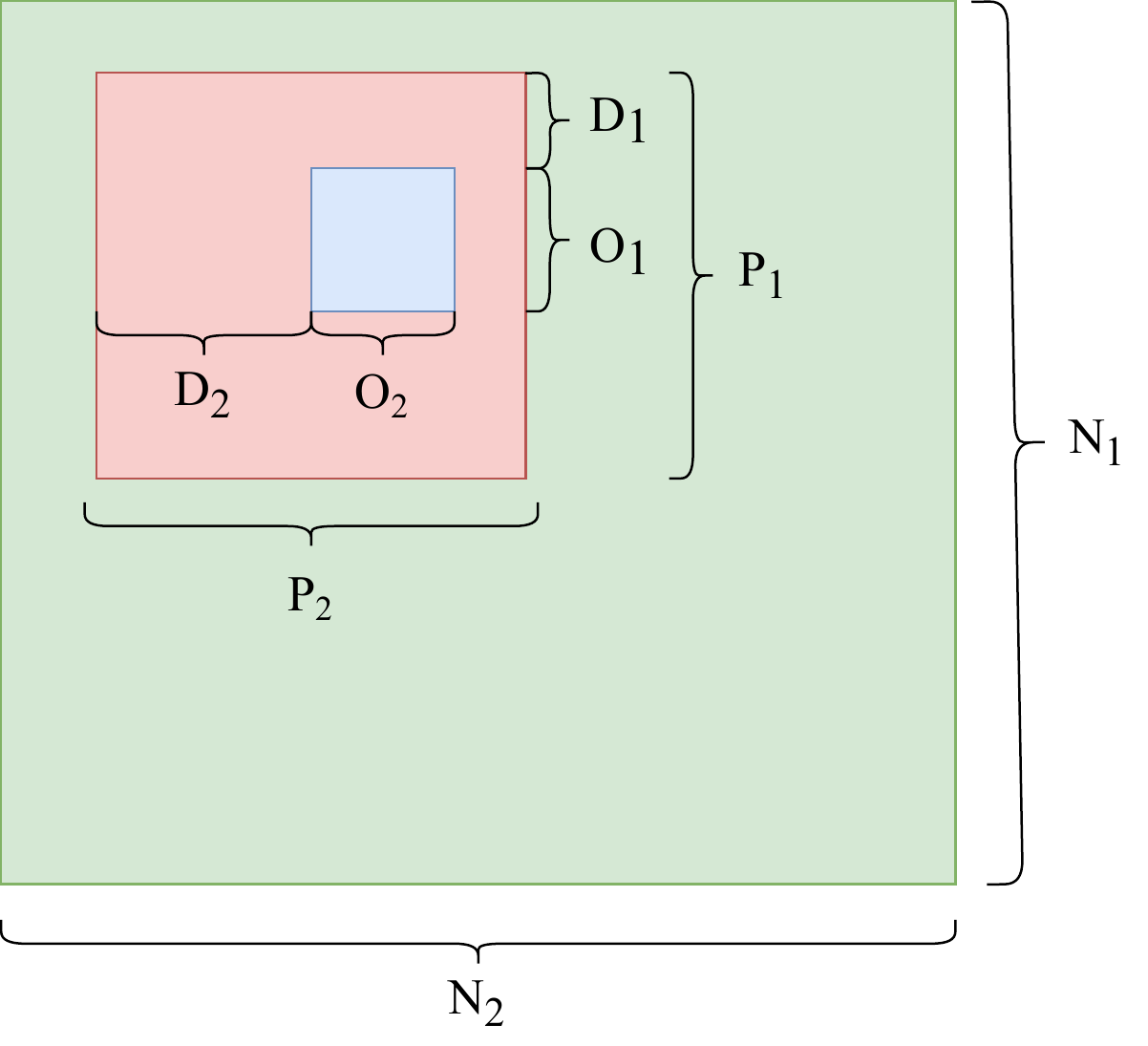}
    \end{center}
    \caption{A visualization of the indices in Definition \ref{def:discrte-convolution} in two dimensions. The \textcolor{udblue}{blue} region is a kernel $c$ of width $O$, the \textcolor{udpink}{pink} region is the zero-padded box of width $P$, and $D$ is the offset of the kernel $c$ in the \textcolor{udpink}{pink} box. The entire signal is the \textcolor{udgreen}{green} box of width $N$.}
    \label{fig:conv-proof-aid}
\end{figure} 

Since convolution is a linear operator we could simply apply Theorem \ref{thm:relu-w-injectivity} and Lemma \ref{lem:relu-w-plus-b-injectivity} to the convolution matrix. It turns out, however, that there exists a result specialized to convolution that is much simpler to verify. In this section we make heavy use of multi-index notation, indicated by capital letters, see Section \ref{sec:notation}. 

\begin{definition} [Convolution Operator] \label{def:conv-operator}
    Let $c \in \Rea^O$. We say that $c$ is a convolution kernel of width $O$. Given $c$ and $x \in \Rea^N$, we define the convolution operator $C \in \Rea^{N \times N}$ with stride 1 as,
    \begin{align} \label{eqn:conv-operator:def}
        (Cx)_{J} = \sum_{I = 1}^O c_{O-I+1}x_{J+I} = \sum_{I' = 1+J}^{O+J}c_{O+J-I'+1}x_{I'}.
    \end{align}
    When $1 \not \leq K$ or $K \not \leq O$ we will set $c_K = 0$. We do not specify the boundary condition on $x$ (zero-padded, periodic, or otherwise), as our results hold generally.
    
\end{definition}

\begin{definition} [Convolutional Layer] \label{def:discrte-convolution}
    We say that a matrix $W \in \Rea^{M \times N}$ is a convolution operator, with $\numconv$ convolutions, provided that $W$ can be written (up to row rearrangement) in the form
    \( 
        W = \begin{bmatrix}
        C_1^T & 
        C_2^T &
        \cdots &
        C_{\numconv}^T
        \end{bmatrix}^T
    \)
    where each $C_k$ is a convolution operator for $k = 1,\dots, \numconv$. A neural network layer for which $W$ is a convolution operator is called a convolution layer.
\end{definition}

Definitions \ref{def:conv-operator} and \ref{def:discrte-convolution} automatically model the standard multi-channel convolution used in practice. For 2D images of size $512 \times 512$, $n_c$ input channels, and $3 \times 3$ convolutions, we simply let $x \in \Rea^N$ with $N = (512, 512, n_c)$ and $c \in \Rea^O$ with $O = (3, 3, n_c)$.

To state our main result for convolutions we also need to define the set of zero-padded kernels for $c \in \Rea^O$. To aid understanding we first give an example of a 2-dimensional zero-padded kernel, as visualized by Figure \ref{fig:conv-proof-aid}. The task is to zero-pad a kernel $c = \bmat{1&-1\\-2&2}$ (the \textcolor{udblue}{blue} box) of dimensions $O = (O_1,O_2)= (2,2)$ so that it occupies the region (the \textcolor{udpink}{blue} box) of size $P = (P_1,P_2) = (3,3)$. This can be done in four different ways, which we call $\cP_{(3,3)}(c)$, given by
\begin{align}
    \label{eqn:zero-padding-aid}
    \cP_{(3,3)}\paren{\bmat{1&-1\\-2&2}} = \set{\bmat{1&-1&0\\-2&2&0\\0&0&0}, \bmat{0&0&0\\1&-1&0\\-2&2&0}, \bmat{0&1&-1\\0&-2&2\\0&0&0}, \bmat{0&0&0\\0&1&-1\\0&-2&2}}.
\end{align}
Further, there is a natural way to order the elements of $\cP_{P}(c)$ for any $c$ and $P$. We can let $D = (D_1,D_2)$ denote the offset of $c$ from the top-left corner. In this way, we denote the (respective) matrices on the r.h.s. of Eqn. \ref{eqn:zero-padding-aid} by $\hat c^{(0,0)}, \hat c^{(1,0)},\hat c^{(0,1)},\hat c^{(1,1)}$. We use the hat notation to remind ourselves that $\hat c^D$ is defined on a different domain than $c$. This example can be straight-forwardly extended to any dimension and kernels of any size as given below.

\begin{definition}[Zero-Padded Kernel] \label{def:zero-padding:kernal}
    Let $c$ be a convolution kernel of width $O$, and let $P$ be another multi-index. Given $O$ and $P$, let $\cD(O,P)$ be the set of admissible offsets $D$ where $D \in \cD(O,P)$ if $0 \leq D  \leq P - O$. Each different choice of $D \in \cD(O,P)$ corresponds to a different way to pad a kernel with zeros. Note that if $O \not \leq P$, then $\cD(O,P) = \emptyset$. Likewise, let $\cT(P)$ be the set of admissible indices $T$ where $T \in \cT(P)$ if $0 \leq T \leq P$. We define the set of zero-padded kernels\footnote{With many convolutional neural networks, padding refers to the act of padding the image (with e.g. zeroes), but the convolutional kernels are not padded. For our results, the variable $P$ refers to the padding of the kernels, not the image.} of $c$ as the set
    \begin{align}
        \cP_P(c) = \left \{ \hat c^D \in \Rea^P \colon D \in \cD(O,P) \text{ for each } T \in \cT(P),
        \paren{\hat c^D}_{T} = \begin{cases}
            c_{D - T+O+1} &\text{ if } 1 \leq T - D \leq O\\
            0 & \text{ otherwise}
        \end{cases}\right \}.
    \end{align}
\end{definition}

\begin{theorem}[Sufficient Condition for Injectivity of Convolutional Layer] \label{thm:suff-conv-injectivity}
    Suppose that $W \in \Rea^{M \times N}$ is a convolution layer with convolutions $\{C_k\}_{k = 1}^{\numconv}$, and corresponding kernels $\{c_k\}_{k = 1}^{\numconv}$. If for any $P$,
    \begin{align} \label{eqn:suff-conv-injectivity:restricted-set}
        W|_{\cP_P} \coloneqq \bigcup_{k = 1}^{\numconv} \cP_P(c_k)
    \end{align}
    has a DSS for $\Rea^P$ with respect to all $x \in \Rea^P$, then $\relu(W\cdot)$ satisfies Theorem \ref{thm:relu-w-injectivity}.
\end{theorem}
We note that Theorem \ref{thm:suff-conv-injectivity} applies if $\bigcup_{k = 1}^{\numconv} \cP_P(c_k)$ is a DSS for $\Rea^P$ for \emph{any} $P$. The proof of the theorem shows that if it holds for any $P$ then the entire operator taken on the whole domain must be injective. See Example \ref{ex:conv-injectivity-1} in Appendix \ref{sec:conv-thm} for an example of applying Theorem \ref{thm:suff-conv-injectivity} in one-dimension.

Theorem \ref{thm:suff-conv-injectivity} applies to multi-channel convolution of width $(O, n_c) = (O_1, \dots, O_p, n_c)$ provided that we choose a $(P, n_c)$ such that $O \leq P$. This theorem shows that if a convolution operator $W$ has a DSS of the \textcolor{udpink}{pink} region, w.r.t. vectors with support in the \textcolor{udpink}{pink} region in Figure \ref{fig:conv-proof-aid}, then it has a DSS of the entire \textcolor{udgreen}{green} region w.r.t. vectors supported on the \textcolor{udgreen}{green} region as well.

To apply Theorem \ref{thm:suff-conv-injectivity} to a convolution layer with $\numconv$ kernels of width $O$, we must choose $P$ so that $W|_{\cP_P}$ has at least the minimal number $2 \prod_{j=1}^p P_j$ of vectors to have a DSS of a vector space of dimension $|P| = \prod_{j=1}^p P_j$. Some algebra gives that
\(
    \numconv \geq 2\prod^p_{j = 1} \frac{1}{1 - {O_j}/{P_j}} \text{ and } P_j \geq \frac{O_j}{1 - ({2}/{\numconv})^{1/p}},
\)
where the last inequality holds only when $\frac{P_j}{O_j}$ is independent of $j$.

In this section we also present a lemma which is key in the proof of Theorem \ref{thm:suff-conv-injectivity}. We believe that this lemma in particular will be useful for future work, as it simplifies the process of showing that a matrix is a DSS of a large domain, by showing it is DSS of smaller ones.

\begin{lemma}[Domain Decomposition] \label{lem:domain-decomp}
    Suppose that $\Rea^n = \Span\{\Omega_{1}, \dots, \Omega_{K}\}$ where\footnote{This is a slight abuse of traditional spanning notation. Here we mean that there is a set of vectors of $\Omega_1$, $\Omega_2$, $\dots$ such that their union spans $\Rea^n$.} each $\Omega_k$ is a subspace and for each $k = 1,\dots, K$ we have a 
    \begin{align}
        W_k = \begin{bmatrix}
            w_{k, 1}^T, \dots, w_{k, N_k}^T
        \end{bmatrix}^T 
        \quad \text{ and } \quad W = \begin{bmatrix} 
            W_1^T, \dots, W_K^T
        \end{bmatrix}
    \end{align}
    such that $w_{k, \ell} \in \Omega_k$ and $W_k$ has a DSS of $\Omega_k$ w.r.t. every $x \in \Omega_k$. Then $W$ has a DSS of $\Rea^n$ w.r.t. every  $x \in \Rea^n$.
\end{lemma}

\subsection{Normalization During Runtime}
\label{sec:norm-training-runtime}

Normalization strategies such as batch \citep{ioffe2015batch}, layer \citep{ba2016layer}, instance \citep{ulyanov2016instance}, group \citep{wu2018group}, weight \citep{salimans2016weight} and spectral \citep{miyato2018spectral} promote convergence during training and encourage low generalization error.  Normalization is a many-to-one operation. In this section we show that batch, weight, and spectral normalization do not interfere with injectivity provided that the network is injective without normalization. In the interest of making this work self-contained, we will not concern ourselves here with the detailed questions of how a network is trained, thus we are only concerned with compatibility between injectivity of a \textit{trained} network. The question of how a network might be trained in such a was as to be injective after training is an important, but beyond the scope of this work.

Let $\{x_i\}_{i = 1,\dots,m}$ represent the inputs to a given layer over a mini-batch. The batch normalization adds two learnable parameters ($\gamma, \beta$) and transforms $x_i$ to $y_i$ as
\begin{align}
    \mu_\mathscr{B} = \frac{1}{m} \sum^m_{i = 1}x_i, \quad \quad
    \sigma^2_\mathscr{B} = \frac{1}{m} \sum^m_{i = 1}(x_i - \mu_\mathscr{B})^2, \quad \quad
    \hat{x}_i = \frac{x_i - \mu_\mathscr{B}}{\sqrt{\sigma_\mathscr{B}^2 + \epsilon}}, \quad \quad
    y_i = \gamma \hat{x}_i + \beta.
\end{align}
Although for a given $\gamma, \beta$ the relationship between $x_i$ and $y_i$ is not injective
, batch normalization is usually present during training but not at test time. Provided that the learned weights satisfy Theorem \ref{thm:relu-w-injectivity} and Lemma \ref{lem:relu-w-plus-b-injectivity}, batch normalization does not spoil injectivity.

In batch renormalization \citep{ioffe2017batch} it is still desirable to whiten the input into each layer. During run time there may be no mini-batch, so running averages of the $\sigma_i$ and $\mu_i$ (denoted $\hat{\sigma}, \hat{\mu})$ computed during training are used. Batch renormalization at test time is then $\hat{x} = \frac{x - \hat \mu}{\sqrt{\hat \sigma^2 + \epsilon}}$, $y = \gamma \hat{x} + \beta$. Since, importantly, $\hat \sigma$ and $\hat \mu$ are not functions of $x$, the mapping between $x$ and $y$ is one-to-one (when $\gamma \neq 0$), thus such normalization in an injective network does not spoil the injectivity.

In weight normalization \citep{salimans2016weight} the coefficients of the weight matrices are normalized to have a given magnitude. This normalization is not a function of the input or output signals, but rather of the weight matrices themselves. This plays no role in injectivity.

Layer, instance and group normalization all take place during both training and execution and, unlike weight and spectral normalization, the normalization is done on the inputs/outputs of layers instead of the weight matrix, thus injectivity is lost. In Appendix \ref{sec:layer-instance-group-normalization}, we present a more general notion of injectivity (Scalar-Augmented Injective Normalization) in order to preserve some of the properties of injectivity in networks that use these normalization strategies.

\section{Universality and Expansivity of Deep Injective Networks}

\label{sec:universality}

We now consider general properties of deep injective networks. Note that a neural network $N_{\theta}\in \mathcal N\mathcal N(n,m)$ is Lipschitz smooth: there is $L_0>0$  such that $|N_{\theta}(x)-N_{\theta}(y)|\le L_0 |x-y|$  for all $x,y\in \Rea^n$. If $N_{\theta}:\Rea^n\to \Rea^m$ is also injective, results of basic topology imply that for any bounded and closed set $B\subset \Rea^n$ the map $N_{\theta}:B\to N_\theta(B)$ has a continuous inverse $N_{\theta}^{-1}:N_\theta(B)\to B$ and the sets $B$ and its image $N_\theta(B)$ are homeomorphic. Thus, for example, if $Z$  is a random variable supported on the cube $[0,1]^n\subset \Rea^n$, the sphere $\mathbb S^{n-1} \subset \Rea^n$, or the torus $\mathbb T^{n-1} \subset \Rea^n$, we see that $N_\theta(Z)$ is a random variable supported on a set in $\Rea^m$ that is homeomorphic to an $n$-dimensional cube or an $n-1$ dimensional sphere or a torus, respectively. This means that injective neural networks can model random distributions on surfaces with prescribed topology. Moreover, if $N_{\theta}:\Rea^n\to \Rea^m$ is a decoder, all objects $x\in N_\theta(\Rea^n)$  correspond to the unique code $z$ such that $x=N_\theta(z)$.

A fundamental property of neural networks is that they can uniformly approximate any continuous function $f:\cC\to \Rea^m$
defined in a bounded set $\cC\subset \Rea^n$. For general dimensions $n$ and $m$,
the injective neural networks do not have an analogous property. For instance, if $f:[-\pi,\pi]\to \R$ is the trigonometric function $f(x)=\sin(x)$, it is easy to see that there 
is no injective neural network (or any other injective function) $N_\theta:[-\pi,\pi]\to \R$ 
such that $|f(x)-N_{\theta}(x)|<1$.
Consider, however, the following trick: add two dimensions in the output vector and consider the map $F:[-\pi,\pi]\to \Rea^3$ given by $F(x)=(0,0,\sin(x))\in \R^3$. When $f$ is approximated by a one-dimensional, non-injective ReLU network $f_\theta:[-\pi,\pi]\to \R$ and $\alpha>0$ is small, then the neural network
$N_\theta(x)=(\alpha\relu\paren{x},\alpha\relu\paren{-x},f_\theta(x))$ is an injective map that approximates $F$.

In general, as ReLU networks are piecewise affine maps, it follows in the case $m=n$ that if a ReLU network $N_\theta:\Rea^n\to \Rea^n$ is
injective, it has to be surjective \cite[Thm. 2.1]{scholtes1996homeomorphism}. This is a limitation of injective neural networks when $m=n$. Building on these ideas we show that when the dimension of the range space is sufficiently large, $m\ge 2n+1$, injective neural networks become sufficiently expressive to universally model arbitrary continuous maps.
In other words, we show that any continuous map $f:\Rea^n\to \Rea^m$, where the dimension of the image space  satisfies $m\ge 2n+1$, can be approximated uniformly in compact sets by injective neural networks. We emphasize that the map $f$ does not need to be injective.

\begin{theorem}[Universal Approximation with Injective Neural Networks]\label{thm:fun-approx-by-injective-nn}
    Let  $f:\Rea^n\to \Rea^m$ be a continuous function, where $m\ge 2n+1$, and $L\ge 1$. Then for any $\varepsilon>0$ and compact subset $\cC \subset \Rea^n$ there exists a  neural network $N_{\theta}\in \mathcal N\mathcal N(n,m)$  of depth $L$ such that $N_{\theta}:\Rea^n\to \Rea^m$ is injective and
    \begin{align}\label{approximation errror}
        \mattimnorm{f(x)-N_{\theta}(x)}\leq \varepsilon_1,\quad\hbox{for all }x\in \cC.
    \end{align}
    Further we have the following. Let $N_\theta$ be such that there exists open and disjoint $\Omega_j \subset \Rea^n$ for $j = 1,\dots,J$ where $\Rea^n = \bigcup^J_{j = 1}\overline{\Omega}_j$ and $\restr{N_\theta}{\Omega_j}$ is affine for each $j$. Then there exist $C_1,C_2, \varepsilon_2(n,m) > 0$, such that for any $\varepsilon_2 \in \epsilon_2(n,m)$ and $R > 0$ where $B(R) \subset \cC$ there exists a $N'_{\theta'} \in \cN\cN(n,m)$ where
    \begin{align}
        \mattimnorm{f(x)-N'_{\theta'}(x)}\leq \varepsilon_1 + \varepsilon_2,\quad\hbox{for all }x\in B(R)
    \end{align}
    and for any $x,y \in B(R)$,
    \begin{align}
        \mattimnorm{N'_{\theta'}(x) - N'_{\theta'}(y)} \geq C \mattinnorm{x - y}
    \end{align}
    where
    \begin{align}
        \label{formula for a}
        C= \frac {C_1^{(n+m)n}}{(2n + \sqrt{m+n})^n(\norm{N_\theta}_{C(B(R))}+R)^n}
         \frac {\varepsilon_2^n}{J^{2\ln( 1+n/(m-2n))+C_2}}.
    \end{align}
\end{theorem}

Before describing the proof, we note that this result also holds for leaky $\relu$ networks with no modification. The key fact used in the proof is that the network is a piecewise affine function.

To prove this result, we combine the approximation results for neural networks  (e.g., Pinkus's density result for shallow networks \citep[Theorem 3.1]{pinkus1999approximation} or Yarotsky's result for deep neural networks  \citep{yarotsky2017error}), with the Lipschitz-smooth version of the generic projector technique. This technique from differential topology is used for example to prove the easy version of the Whitney's embedding theorem \citep[Chapter 2, Thm.\ 3.5]{hirsch2012differential} and Whitney's injective immersion theorem \cite[Thm.\ 2.4.3]{Mukherjee}.
Related random projection techniques have been used earlier in machine learning and compressed sensing \citep{Broomhead2,Broomhead,Baraniuk-Wakin,Hedge-Wakin-Baraniuk,Iven-Maggioni}.

The proof of Theorem \ref{thm:fun-approx-by-injective-nn} is based
on applying the above classical results to locally approximate the function $f:\R^n\to \R^m$ by some (possibly non-injective) ReLU-based neural network
$F_\theta:\R^n\to \R^m$, and augment it by adding additional variables, so that
$H_\theta(x)=(x,F_\theta(x))$ is an injective map $H_\theta:\R^n\to \R^{n+m}$. The image of this map, $M=H_\theta(\R^n)$, is an $n$-dimensional, Lipschitz-smooth submanifold 
of $\R^{n+m}$. The dimension of $m+n$ is larger than $2n+1$, which implies that for a randomly chosen projector
$P_1$ that maps $\R^{n+m}$ to a subspace of dimension $m+n-1$, the restriction of $P_1$ on the submanifold $M$, that is, $P_1:M\to P_1(M)$, is injective. By applying $n$ random projectors, $P_1,\dots,P_n$, we have that $N_\theta=P_n\circ P_{n-1}\circ \dots \circ P_1\circ H_\theta$ is an injective map whose image is in a $m$  dimensional linear space. These projectors can be multiplied together, but are generated sequentially. By choosing the projectors $P_j$ in a suitable way, the obtained neural network $N_\theta$ approximates the map $f$.

We point out that in the proof of Theorem \ref{thm:fun-approx-by-injective-nn} it is crucial that we first approximate function $f$ by a neural network and then apply to it a generic projection to make the neural network injective.
Indeed, doing this in the opposite order may fail as an arbitrarily small deformation  (in $C(\Rea^n)$)
of an injective map may produce a map that is non-injective.

\begin{figure}
    \begin{center}
        \includegraphics[width=0.43\textwidth]{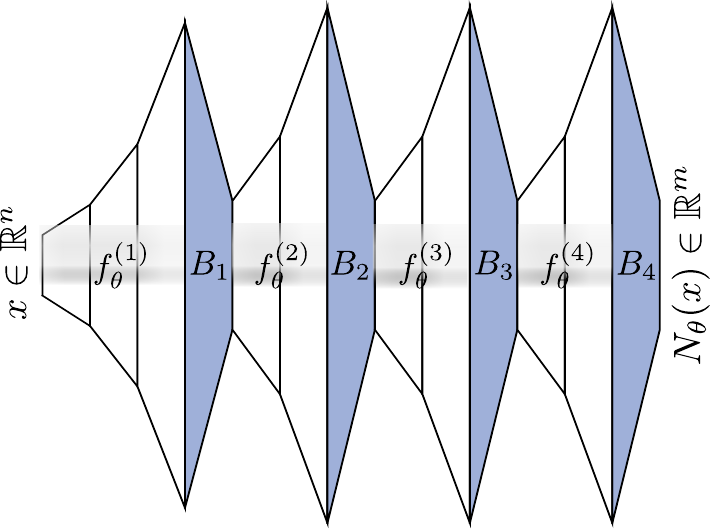}
    \end{center}
    \caption{An illustration of an injective deep neural network that avoids expansivity as described by Corollary \ref{cor: decreasing dimension}. White trapezoids are expansive weight matrices satisfying Theorem \ref{thm:relu-w-injectivity}, and the blue trapezoids are random projectors that reduce dimension while preserving injectivity.}
    \label{fig:xmas-tree}
\end{figure}

Note that $J$  is estimated in theory of hyperplane arrangements \cite[Section 3.11]{stanley2011enumerative}. Indeed, when $m$ hyperplanes $\mathcal A=\{T_1,\dots,T_m\}$ divide $\R^n$ to disjoint regions, the number of the regions $J=r(\mathcal A)$ is bounded by
\begin{align}
    J\leq 1+m+\binom{m}{2}+\dots+\binom{m}{n}.
\end{align}
The proof of Theorem \ref{thm:fun-approx-by-injective-nn}  implies the following corollary on cascaded neural networks where the dimensions of the hidden layers can both increase and decrease (see Figure \ref{fig:xmas-tree}):

\begin{corollary} \label{cor: decreasing dimension} 
  Let $n,m,d_{j}\in \mathbb Z_+$, $j=0,1,\dots,2k$ be such that $d_0=n$,
  $d_{2k}=m\ge 2n+1$ and $d_{j}\ge 2n+1$ for even indexes $j\ge 2$. Let
$$
F_k=B_k\circ  f^{(k)}_\theta\circ B_{k-1}\circ  f^{(k-1)}_\theta\circ\dots\circ B_1\circ  f^{(1)}_\theta
$$
where $f^{(j)}_\theta :\R^{d_{2j-2}}\to \R^{d_{2j-1}}$ are injective neural networks
and $B_j\in\Rea^{d_{2j-1}}\to\Rea^{d_{2j}}$ are random matrices whose joint  distribution is absolutely continuous with respect to the Lebesgue measure of $\prod_{j=1}^k(\Rea^{d_{2j}\times d_{2j-1}})$.
Then the neural network $F_k:\Rea^n\to \Rea^{m}$ is injective almost surely.
\end{corollary}

Observe that in Corollary \ref{cor: decreasing dimension}  the weight matrices $B_j$ and $B_{j'}$ may not be independent. Moreover, the distribution of the matrices $B_j$ may be supported in an arbitrarily small set, that is, a small random perturbation of the weight matrices make the function $F_k$ injective. Corollary \ref{cor: decreasing dimension} characterizes deep globally injective networks, while avoiding the exponential growth in dimension as indicated by Corollary \ref{cor:minimal-expansivity}.

\section{Numerical Experiments}

\subsection{Experiments testing the minimal expansivity threshold} \label{app:expansivity}

Theorem \ref{theorem-injGauws} implies that an expansivity factor of $2.1$ as suggested by \citep{lei2019inverting} is not enough to ensure global injectivity with Gaussian weights. Insufficient expansivity leads to failure cases in inverse problems, as we show in this section. We use a single ReLU layer, $f(x) = \relu(Wx)$, and choose $x$ to be an MNIST digit for ease of visualization of the ``latent'' space. We choose $W\in \R^{cn\times n}$ with iid Gaussian entries and set $c = m/n = 2.1$. Given $y=f(x)$, our goal is to reconstruct $x$ using, for example, gradient descent. This requires at least $n$ of the $cn$ entries of $f(x)$ to be non-zero since $x\in\R^n$. As shown in Figure \ref{fig:gaussian_numerics}, $x_0 = -\frac{1}{m}\sum_j w_j / \norm{w_j}$ violates this condition with high probability when $c = 2.1$.

We choose $2$ data samples, $x_1$ and $x_2$ to be $2$ ``latent'' codes. Note that $W$ has a DSS with respect to $x_1$. This ensures that given $f(x_1)$ one can reconstruct $x_1$. Now, in order to show that injectivity around a few points is insufficient for global injectivity, we linearly interpolate between $x_0$ and $x_1$ and $x_0$ and $x_2$ in a $2$D grid. We show the intermediate $x$s as images in Figure $\ref{fig:dss_single_relu_layer}$. The $x$s shown in red 
cannot be uniquely recovered given $f(x)$, i.e., there exist infinitely many samples close to the red samples that all map to the same output.

To further illustrate this, we also run a simple experiment where we have $y= f(x)+\eta$, with a  $10$dB signal-to-noise ratio. We invert $f$ to estimate $x$ for 3 different samples in the domain of $f$. We can easily see that non-injective points of the domain give poorer reconstructions due to the lower number of positive inner products with $W$. In order to avoid problems of non-convexity in optimization we report the best reconstruction out of 10 in 4 different trials with different random seeds, similar to experiments of \citep{bora2017compressed}. While these results show reconstructions with single layer, the issues are only exacerbated with multiple layers.

\begin{figure}
\centering
\includegraphics[width=0.8\textwidth]{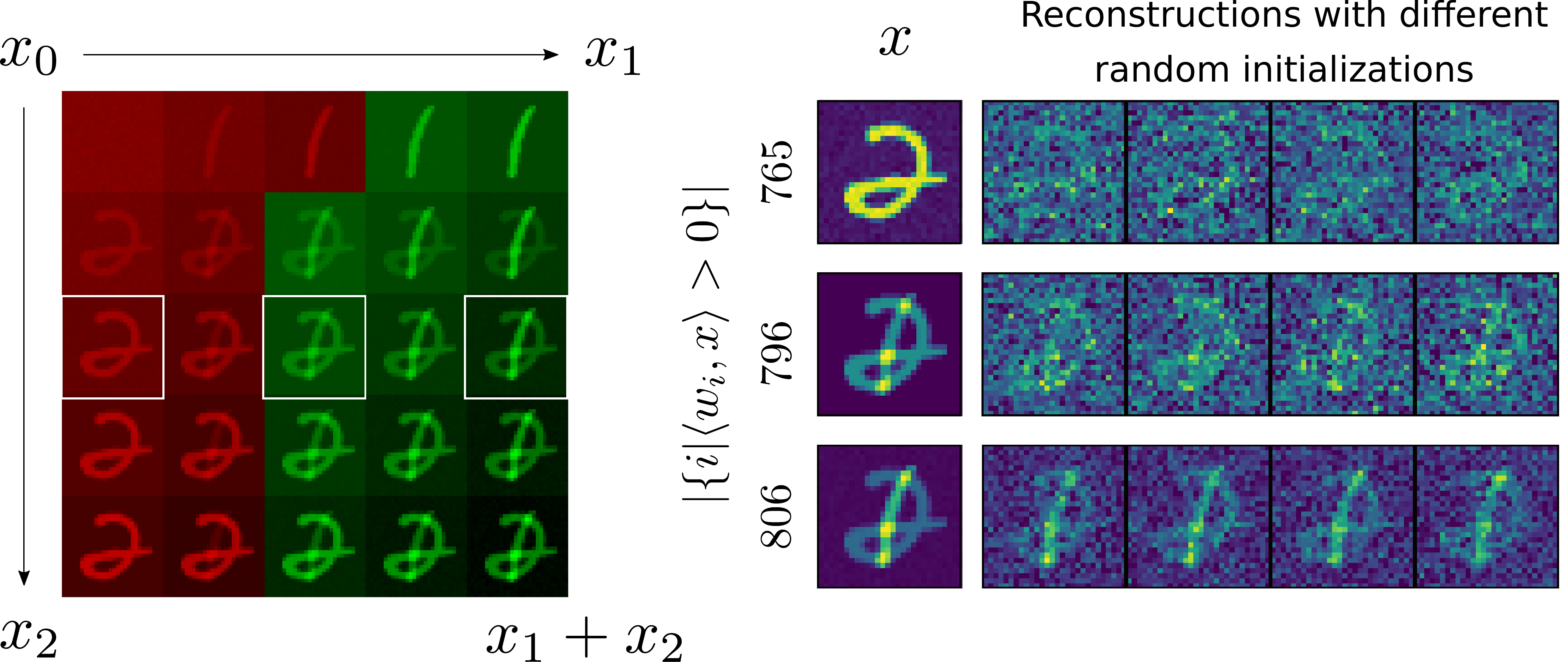}
\caption{The mapping $\phi(x):=\relu(Wx),\ x \in \R^n$ when $W\in\R^{2.1n\times n}$ as suggested in \citep{lei2019inverting} is not injective as the samples shown in red do not have $n$ positive inner products with the rows of $W$. This implies that around the red samples $\phi$ cannot be inverted. For the 3 samples delineated by white boxes on the left, we do a simple denoising test where $y = \phi(x) + \eta$ and $\eta$ corresponds to $10$dB noise. Each reconstruction reported is best out of 10 trials based on MSE error.}
\label{fig:dss_single_relu_layer}
\end{figure}

\subsection{Injective GANs Improve Inference}

We devise a simple experiment to demonstrate that 1) one can construct injective networks using Corollary \ref{cor:minimal-expansivity} that 2) perform as well as the corresponding non-injective networks, while 3) being better at inference problems. One way to do inference with GANs is to train a so-called inference network $\infnet  : \cX \to \cZ$ jointly with the generator $G : \cZ \to \cX$ so that $I(G(z)) \approx z$ \citep{donahue2016adversarial,dumoulin2016adversarially}; $\infnet$ is trained to invert $G$ on its range. This is used to evaluate the likelihood of an image $x$ in the range of $G$ as $p_Z(\infnet(x))$. However, for $p_Z(\infnet(G(z)) = p_Z(z)$ to hold, $G$ must be injective---otherwise the distribution of $\infnet(G(z))$ will be different from that of $z$. One would thus hope that the distribution of $\infnet(G(z))$ will be closer to that of $z$ if we enforce injectivity. If that is the case and the sample quality is comparable, then we have a practical tool to improve inference.

We use the DCGAN~\citep{radford2015unsupervised} architecture with the same hyperparameters, layer sizes and normalization strategies for the regular GAN and the injective GAN; see Appendix \ref{app:arch_dets}. Note that the original DCGAN uses several $\relu$ conv layers with an expansivity of $2$ per layer. By Corollary \ref{cor:minimal-expansivity} these layers are non-injective. Therefore, the filters of the DCGAN are modified to be of the form $[C; -s^2 C]$ in order to build an injective GAN. We train the networks to generate $64 \times 64 \times 3$ images and draw the latent code $z \in \R^{256}$ iid from a standard normal distribution. We test on CelebA~\citep{liu2015faceattributes} and FFHQ~\citep{karras2019style} datasets. To get a performance metric, we fit Gaussian distributions $\mathcal{N}(\mu, \Sigma)$ and $\mathcal{N}(\mu_{\text{inj}}, \Sigma_{\text{inj}})$ to $G(z)$ and $G_{\text{inj}}(z)$. We then compute the Wasserstein-2 distance $\mathcal{W}_2$ between the distribution of $z \sim \mathcal{N}(0, I_{256})$ and the two fitted Gaussians using the closed-form expression for $\mathcal{W}_2$ between Gaussians,
\(
\mathcal{W}^2_2(\mathcal{N}(\mu_1,\Sigma_1),\mathcal{N}(\mu_2,\Sigma_2)) = \|\mu_1-\mu_2\|^2 + \mathrm{Tr}(\Sigma_1 + \Sigma_2 - 2(\Sigma_1\Sigma_2)^{1/2}).
\)
We summarize the results in Table \ref{tab:metrics}. 
Despite the restrictions on the weights of injective generators, their performance on popular GAN metrics---Fr\'echet inception distance (FID)~\citep{heusel2017gans} and inception score (IS) ~\citep{salimans2016improved}---is comparable to the standard GAN while inference improves. That the generated samples are indeed comparable to the standard GAN can be also gleaned from the from the figure in Appendix \ref{app:arch_dets}.
\begin{table}[ht]
\renewcommand{\arraystretch}{1.} 
\centering
\caption{Injectivity improves inference without sacrificing performance.}
\begin{tabular}{@{}llccc}
\toprule
\multicolumn{1}{@{}l}{Dataset} & \multicolumn{1}{c}{Type of $G$} & Inception Score$\uparrow$ & Fr\'echet Inception Distance$\downarrow$ & $\mathcal{W}_2^2(\Prob_{\hat{z}},\Prob_{z})\downarrow$ \\ \midrule
\multirow{2}{*}{\textit{CelebA}}              & Injective & $2.24\pm 0.09$ & $39.33 \pm 0.41$& $\mathbf{18.59}$        \\\vspace{1mm}
                                     & Regular & $2.22\pm 0.16$ & $50.56 \pm 0.52$& $33.85$        \\ 
\multirow{2}{*}{\textit{FFHQ}}                & Injective & $2.56\pm 0.15$ & $61.22 \pm 0.51$& $\mathbf{9.87}$        \\\vspace{1mm}
                                     & Regular & $2.57\pm 0.16$ & $47.23 \pm 0.90$& $19.63$\\
\bottomrule
\end{tabular}
\label{tab:metrics}
\end{table}

\section{Conclusion}

We derived and explored conditions for injectivity of ReLU neural networks. In contrast to prior work which looks at random weight matrices, our characterizations are deterministic and derived from first principles. They are also sharp in that they give sufficient and necessary conditions for layerwise injectivity. Our results apply to any network of the form \eqref{eqn:deep-network-definition}---they only involve weight matrices but make no assumptions about the architecture. We included explicit constructions for minimally expansive networks that are injective; interestingly, this simple criterion already improves inference in our preliminary experiments. 
The results on universality of injective neural networks further justify their use in applications; they also implicitly justify the various Jacobian conditioning strategies when learning to generate real-world data. Further, injective neural networks are topology-preserving homeomorphisms which opens applications in computational topology and establishes connections to tools such as self-organizing maps. Analysis of deep neural networks has an analogue in the analysis of inverse problems where one studies uniqueness, stability and reconstruction. Uniqueness coincides with injectivity, quantitative stability with the Lipschitz constant of the inverse, and, following \citep{lei2019inverting}, from a linear program we get a reconstruction in the range. 

We also remark here that we believe that many of the theoretical results from this paper can be applied to networks with other activation functions with minimal changes. For example we believe that, Theorems \ref{thm:relu-w-injectivity}, \ref{thm:inv-lip:global}, \ref{thm:suff-conv-injectivity} and \ref{thm:injec-ofgaussian-lower-bound} can be made to apply to arbitrary activation functions that are injective on a half plane. We also believe that Theorem \ref{thm:fun-approx-by-injective-nn} can be applied to any model of injective networks provided that the activation functions are piece-wise affine.

\section*{Acknowledgments}

We are grateful to Daniel Paluka and Charles Clum for pointing out a problem with the bound on the sum of binomial coefficients used to prove Theorem \ref{theorem-injGauws}. We are also grateful to Teemu Saksala for the many useful conversations had, especially in the early stages of this project.


I.D. was supported by the European Research Council Starting Grant 852821---SWING. M.L. was  supported by Academy of Finland, grants 284715, 312110. M.V.dH. gratefully acknowledges support from the Department of Energy under grant DE-SC0020345, the Simons Foundation under the MATH + X program, and the corporate members of the Geo-Mathematical Imaging Group at Rice University.

\bibliographystyle{plain}
\bibliography{references.bib}

\begin{appendices}

\section{Proofs from Section 2}

\subsection{Proofs from Subsection 2.2} \label{sec:proof-of-telu-injec}

\begin{proof}[Proof of Theorem \ref{thm:relu-w-injectivity}]
    Suppose that $W$ is such that the conditions of Theorem \ref{thm:relu-w-injectivity} hold, and that $\relu(W x_1) = \relu(W x_2) = y$. If for $j \in [[m]]$, $y|_j > 0$ then both $\innerprod{w_j}{x_1} > 0$ and $\innerprod{w_j}{x_2} > 0$. Similarly, if $y|_j \leq 0$ then both $\innerprod{w_j}{x_1} \leq 0$ and $\innerprod{w_j}{x_2}  \leq 0$. In particular, this implies that
    \begin{align}
        \label{eqn:lem:relu-w-injec:1}
        \innerprod{w_j}{x_1} > 0 \iff \innerprod{w_j}{x_2} > 0 \text{ and } \innerprod{w_j}{x_1} \leq 0 \iff \innerprod{w_j}{x_2} \leq 0.
    \end{align}
    If we then consider $x_\alpha \coloneqq (1 - \alpha)x_1 + \alpha x_2$ where $\alpha \in (0, 1)$, then 
    \begin{align}
        \relu(W x_\alpha) = y = \relu(W x_1) = \relu(W x_2) .
    \end{align}
    If $\innerprod{w_j}{x_\alpha} > 0$ then at least one of $\innerprod{w_j}{x_1} > 0$ or $\innerprod{w_j}{x_2} > 0$. (\ref{eqn:lem:relu-w-injec:1}) implies that both must hold, therefore $\innerprod{w_j}{x_1} = \innerprod{w_j}{x_2} > 0$. If $\innerprod{w_j}{x_\alpha} = 0$ then $\innerprod{w_j}{x_1} = \innerprod{w_j}{x_2} = 0$ (otherwise (\ref{eqn:lem:relu-w-injec:1}) is violated), thus
    \begin{align}
        \relu(W|_{S(x_\alpha, W)} x_1) = \relu(W|_{S(x_\alpha, W)} x_2) \implies W|_{S(x_\alpha, W)} x_1 = W|_{S(x_\alpha, W)} x_2
    \end{align}
    and so because $W|_{S(x_\alpha, W)}$ is full rank, this implies that $x_1 = x_2$. This proves one direction.
    
    The other direction follows from the following. Suppose that there exists a $x$ such that $W|_{S(x,W)}$ don't span $\Rea^n$. If $S(x,W) = \emptyset$ non-injectivity trivially follows, so suppose w.l.o.g. that $S(x,W) \neq \emptyset$. Let $x^{\perp} \in \ker(W|_{S(x,W)})$ and $\alpha \in \Rea^{+}$ such that $\alpha < \min_{j \in S^c(x,W)} \frac{-\innerprod{x}{w_j}}{|\innerprod{x^{\perp}}{w_j}|}$\footnote{If $\innerprod{x^{\perp}}{w_j} = 0$ for all $j \in S^c(x,W)$, then any $\alpha > 0$ will do.}. Then for $j = 1,\dots, m$ one of the following two hold
    \begin{align}
        \text{ if $j \in S(x,W)$ then}& \innerprod{w_{j}}{x + \alpha x^{\perp}} = \innerprod{w_{j}}{x} + \alpha \innerprod{w_j}{x^{\perp}} = \innerprod{w_j}{x}\\
        \text{ if $j \in S^c(x,W)$ then}& \innerprod{w_{j}}{x + \alpha x^{\perp}} = \innerprod{w_{j}}{x} + \alpha \innerprod{w_j}{x^{\perp}} < 0 .
    \end{align}
    Thus, as $\relu$ acts pointwise (row-wise in $W$), we have that
    \begin{align}
        \label{eqn:relu-w-injectivity:non-injective}
        \relu(W(x +\alpha x^{\perp})) = \relu(Wx)
    \end{align}
    and, hence, $\relu(W \cdot)$ is not injective. 
\end{proof}

\begin{proof}[Proof of Lemma \ref{lem:relu-w-plus-b-injectivity}]
    First, we show that if $\relu(W|_{b\geq 0}\cdot)$ is injective, then so is $\relu(W \cdot + b)$. Clearly if $\relu(Wx_1 + b) = \relu(Wx_2 + b)$ then $\relu(W|_{b \geq 0}x_1 + b|_{b \geq 0}) = \relu(W|_{b \geq 0}x_2 + b|_{b \geq 0})$
    as well. If we apply Lemma \ref{lem:relu-arithmetic:pos-addition} to each component of the above equation, then we obtain that $\relu(W|_{b \geq 0}x_1) = \relu(W|_{b \geq 0}x_2)$
    which, given the injectivity of $\relu(W|_{b \geq 0}\cdot)$, implies that $x_1 = x_2$.
    
    Now suppose that $\relu(W|_{b \geq 0})$ is not injective. Let $x \in \Rea^n$ be such that $W|_{b \geq 0}$ doesn't have a DSS of $\Rea^n$ w.r.t. $x$. Let $\beta > 0$ be small enough so that $W|_{b < 0}(\beta x) + b|_{b < 0} < 0$
    component-wise, and let $x^{\perp} \in \Rea^n$ such that (as in (\ref{eqn:relu-w-injectivity:non-injective})) $\relu(W|_{b \geq 0}(\beta x + x^{\perp})) = \relu(W|_{b \geq 0} \beta x)$. Further let $\alpha < 1$ be small enough such that $W|_{b < 0}(\beta x + \alpha x^{\perp}) + b|_{b < 0} < 0$.
    By a component-wise analysis, we have that 
    \begin{align}
        \relu(W(\beta x + \alpha x^\perp) + b)
        =& \relu(W\beta x + b)|_{b \geq 0}
        = \relu(W\beta x + b) ;
    \end{align}
    thus, $\relu(W\cdot + b)$ is not injective.
\end{proof}

\begin{proof}[Proof of Corollary \ref{cor:minimal-expansivity}]
    If $W \in \Rea^{m \times n}$ is injective then consider a plane $p$ in $\Rea^n$ that none of the rows of $W$ lie in. Apply Theorem \ref{thm:relu-w-injectivity} to both normals of the plane. The corresponding DSS' for each normal are disjoint, thus there must be at least $2n \geq m$, so $m < 2 \cdot n$ implies non-injectivity.
    
    Now we show that if $W$ satisfies Theorem \ref{thm:relu-w-injectivity}, then $W$ is of the form given by (\ref{eqn:cor:twice-expansive-factorization}). Suppose that there is a row vector $w_i$ such that there are no row vectors pointing in the $-w_i$ direction. Let $p$ be a plane through the origin such that $w_i \in p$, but $w_{i'} \not \in p$ for $i \neq i'$. By Theorem \ref{thm:relu-w-injectivity}, there must be at least $n$ columns that lie on each (closed) halfspace of $p$. Not counting $w_i$, one side must have $n-1$ vectors, and the other $n$. Indeed one of the sides must have exactly $n$ vectors on it (including $w_i$). Let $\theta_i$ be the principle angle between $w_i$ and $p$. If we rotate $p$ by an angle $\phi < \min_{i \neq i'} \theta_i$ then we can obtain a new plane $p_\phi$ with $w_i$ on the side with $n$ vectors, so that the other side of $p_\phi$ still only contains $n-1$ vectors. Hence there is no DSS for $p_\phi$ plane's normal on that side.
    
    Thus for $\relu{W}$ to be injective, for every $i \in [[2n]]$ there must be a different $i' \in [[2n]]$ such that $w_i$ and $w_{i'}$ are anti-parallel. This can only happen if for every $w \in W$ there is a corresponding $-dw \in W$, so $W$ must have the form of (\ref{eqn:cor:twice-expansive-factorization}).
\end{proof}

\subsection{Proof of Theorem~\ref{theorem-injGauws}} \label{sec:proof-of-thm-2}

\subsubsection{Upper Bound on Minimal Expansivity}

Let $(\Omega, \cF, \mathbb{P})$ be a probability space and let $w_i : \Omega \to \Rea^n$, $1 \leq i \leq m$ be $m$ iid Gaussian vectors on $\Omega$ stacked in a matrix $W : \Omega \to \Rea^{m \times n}$.

We aim to find conditions for a ReLU layer with the matrix $W$ to be injective. Injectivity fails at $p$ the half-space $\set{x \ : \ \langle x, p \rangle \geq 0}$ contains fewer than $n$ vectors $w_i$. Equivalently, injectivity fails if there is a half-space which contains more than $m - n$ vectors $w_i$ (which implies that its opposite half-space has fewer than $n$).

Let $k = m - n + 1$ and denote by $A \in \Rea^{k \times n}$ some $k \times n$ submatrix of $W$ (for example, the first $k$ rows). Our strategy is to first bound the probability that for a fixed subset of $k$ rows of $W$, there exists an $x$ having positive inner products will all $k$ rows (which signals non-injectivity per above discussion). Second, since there are $\binom{m}{k}$ subsets of $k$ rows, we use the union bound to get an upper bound on the probability of non-injectivity. 

For the first part we follow the proof of \citep[Theorem 13.6]{burgisser2013condition}, parts of which we reproduce for the reader's convenience.  For a sign pattern $\sigma \in \set{-1, 0, 1}^k$, we denote by $R_A(\sigma)$ the set of all $x \in \Rea^n$ which produce the sign pattern $\sigma$ (they belong to the $\sigma$-``wedge'', possibly empty),
\begin{equation}
    R_A(\sigma) = \set{x \in \Rea^n \ : \ \sign(\innerprod{x}{a_i}) = \sigma_i, \ i \in \set{1, \ldots, k}}.
\end{equation}

For $\sigma \in \set{-1, +1}^k = \Sigma$, we define the event $\cE_\sigma$ as
\begin{equation}
    \cE_\sigma = \set{\omega : R_{A(\omega)}(\sigma) \neq \emptyset}.
\end{equation}
We are interested in $\sigma_0 = (1, \ldots, 1)$, meaning that all the inner products are positive. Note that the probability of $\cE_\sigma$ is the same for all $\sigma$ due to the symmetry of the Gaussian measure. Further, note that $\sum_{\sigma \in \Sigma} \mathbf{1}_{\cE_\sigma}(\omega) = |\set{\sigma \ : \ R_{A(\omega)}(\sigma) \neq \emptyset}|$ is the number of wedges defined by $A$. Then
\begin{equation}
    \label{eq:P_E_sigma0}
    \mathbb{P} (\cE_{\sigma_0}) 
    = \frac{1}{2^k} \sum_{\sigma \in \Sigma} \mathbb{P} (\cE_{\sigma}) 
    = \frac{1}{2^k} \sum_{\sigma \in \Sigma} \bE(\mathbf{1}_{\cE_\sigma})
    = \frac{1}{2^k} \bE \left( \sum_{\sigma \in \Sigma} \mathbf{1}_{\cE_\sigma} \right)
    = \frac{1}{2^{k-1}} \sum_{i = 0}^{n - 1} \binom{k-1}{i},
\end{equation}
by \citep[Cor. 1]{winder1966partitions}. In \cite{winder1966partitions} the corollary holds for the generic case but the equality case is generic, and so holds with probability 1, thus we use equality here. 

\begin{figure}
    \centering
    \includegraphics[width=\textwidth]{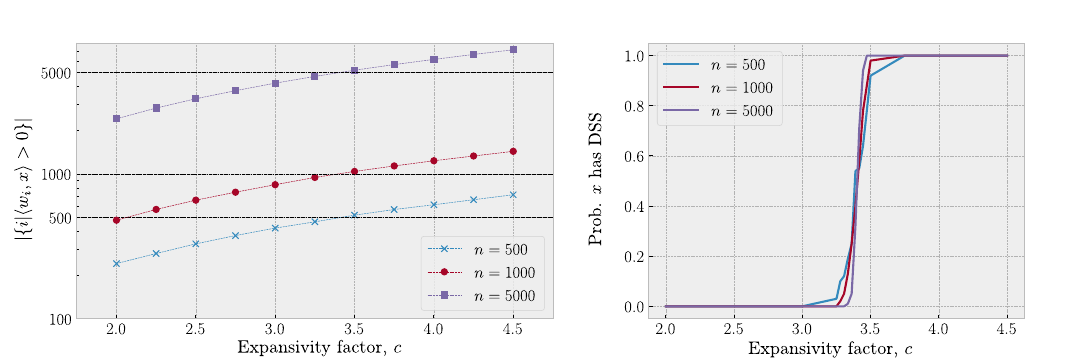}
    \caption{We empirically show that for random Gaussian matrices $\mathbf{W} \in \Rea^{m \times n}\, m=cn$, a critical oversampling factor of at least $c>3.5$ is required for injectivity of ReLU. We choose $x = -\frac{1}{m}\sum_j w_j / \norm{w_j}$. \textit{Left:} a plot of the number of elements of $W$ that point in the direction of $x$ as a function of expansivity for several different choices of $n$. If this quantity is less than $n$, then $W$ cannot contain a DSS of $\Rea^n$ w.r.t. $x$. \textit{Right:} a plot of the empirical probability that a Gaussian matrix $W$ contains a DSS of $\Rea^n$ w.r.t. $x$ as a function of expansivity for several different values of $n$.}
    \label{fig:gaussian_numerics}
\end{figure}

We now have the probability that for a subset of $k$ vectors $w_i$, there exists an $x \in \Rea^n$ which has positive inner products with all $k$ vectors. We are interested in the following event which implies non-injectivity,
\begin{equation}
    \cE_{\text{NI}} = \set{\omega \ : \ W(\omega)~\text{has a subset of $k$ rows}~B(\omega)~\text{such that}~R_{B(\omega)}(\sigma_0) \neq \emptyset}.
\end{equation}
Conversely, $\omega \in \cE_{\text{NI}}^c$ implies almost sure injectivity.

Since there are $\binom{m}{k} = \binom{m}{m - n + 1} = \binom{m}{n - 1}$ different subsets of $k$ rows, we can bound the probability of $\cE_{\text{NI}}$ as
\begin{align}
    \mathbb{P}(\cE_{\text{NI}}) 
    & \leq \binom{m}{n - 1} \mathbb{P}(\cE_{\sigma_0}) = \frac{1}{2^{k-1}}\binom{m}{n - 1}\sum^{n-1}_{i=0}\binom{k-1}{i}\\
    &\leq\paren{\frac{me}{n-1}}^k2^{1-k}2^{H\left(\frac{n-1}{k-1}\right)(k-1)}\\
    &\leq \left(\frac{me}{n - 1}\right)^n 2^{- (m - n) } 2^{(m - n) H\left(\frac{n - 1}{m - n}\right)} \\
    & \lessapprox (ce)^n  2^{n(c - 1)[ H((c - 1)^{-1}) - 1]} \\ 
    \label{eqn:upper-bound-gauss-inj}
    &= 2^{-n \left[-\log_2(ce) - (c - 1)(H((c - 1)^{-1}) - 1) \right]},
\end{align}
where\footnote{\url{https://en.wikipedia.org/wiki/Binomial_coefficient\#Sums_of_binomial_coefficients}} $\binom{n}{k} \leq \paren{\frac{ne}{k}}^k$. and $\sum^k_{i=1}\binom{n}{i} \leq 2^{H(k/n)n}$, where $H(\epsilon) = -\epsilon \log_2(\epsilon) - (1 - \epsilon)\log_2(1 - \epsilon)$ is the binary entropy function. Equation \ref{eqn:upper-bound-gauss-inj} converges to zero as $n\rightarrow \infty$ precisely if $\left[-\log_2(ce) - (c - 1)(H((c - 1)^{-1}) - 1) \right]$ is positive. Numerically we can see that the bound is positive if $c \geq 10.5$.

\subsubsection{Lower Bound on Minimal Expansivity of Layers with Gaussian Weights}
\label{sub:minexpGauss}

In this appendix we prove that large random Gaussian weight matrices $W \in \Rea^{m\times n}$ yield non-injective $\relu$ layers with high probability if $m < c^*n$, where $c^*\approx 3.4$. Note that injectivity fails if there exists a half-space in $\R^n$ which contains less than $n$ rows of $W$. Equivalently, injectivity fails if there is a half-space with more than $m - n$ rows of $W$, since then the opposite halfspace has less than $n$. The core idea of the proof is to make an educated guess for a half-space which has many vectors, and compute the probability that it has more than $m - n$. A good such guess is to take the halfspace defined by the average direction of a row of $W$. Equivalently, we study the size of $S(\overline{w}, W)$ where $\overline{w} = \frac{1}{m} \sum_{k = 1}^m w_k$, the row average of a matrix $W$. If $|S(\overline{w}, W)| > m - n$, then $|S(-\overline{w}, W)| < n$, and $W$ cannot have a DSS w.r.t. $-\overline{w}$. As the Figure \ref{fig:gaussian_numerics} shows, when $m < c^*n$, $W$ does not contain a DSS  w.r.t. $-\overline{w}$ with high probability. 

Let $W \in \Rea^{m\times n}$ be a matrix such that the rows of $W$ are i.i.d. random vectors distributed as $\cN(0, I_n)$. Define the event $E_i$ as
\begin{equation}
    E_i := \set{\omega \ : \ \innerprod{w_i(\omega)}{\sum_{k = 1}^m w_k(\omega)} \geq 0}.
\end{equation}

\begin{lemma}
    \label{lem:concentration}
    Let $m, n \to \infty$ so that $\frac{m}{n} = c$. Then $\mathbb{P}(E_i) \to \frac{1}{2} \mathrm{erfc} \left( - \frac{1}{\sqrt{2c}} \right)$, where $\mathrm{erfc}(z) = 1 - \frac2{\sqrt{\pi}}\int^z_0e^{-t^2}dt$.
\end{lemma}

\begin{proof}
    We have
    \begin{align*}
        \mathbb{P} \left(E_i\right)
        = \mathbb{P} \left( \norm{w_i}^2 + \innerprod{w_i}{\sum_{i \neq j} w_j} \geq 0 \right)= \mathbb{P} \left( \norm{w_i}^2 + \norm{w_i}Y \geq 0\right)
    \end{align*}
    where $Y \sim \cN(0,m-1)$. We now use the fact that $\norm{w_i}^2$ concentrates around $n$. Define the event
    \begin{align}
        D_i = \set{ \omega \ : \ \norm{w_i(\omega)}^2 \in \left[ (1-\epsilon) n, n/(1-\epsilon) \right] }.
    \end{align} It holds by standard concentration arguments (e.g. \cite[Page 6, Corollary 2.3] {barvinok2005math}) that $\mathbb{P}(D_i) \geq 1 - 2 \exp(-\epsilon^2 n / 4)$, and so by the law of total probability,
    \begin{align}
        \label{eq:totprob}
        \mathbb{P} \left\{ \norm{w_i}^2 + \norm{w_i}Y \geq 0 \right\}
        & =
        \mathbb{P} \left\{ \norm{w_i} + Y \geq 0 \, | \, D_i \right\} \mathbb{P} \{ D_i \}
        +
        \mathbb{P} \left\{ \norm{w_i} + Y \geq 0 \, | \, D_i^c \right\} \mathbb{P} \left\{ D_i^c\right\}.
    \end{align}
    Choosing $\epsilon = \epsilon(n) = n^{-1/4}$ yields
    \[
        \mathbb{P} \left\{ \sqrt{n - n^{3/4}} + Y \geq 0 \right\}
        \leq
        \mathbb{P} \left\{ \norm{w_i} + Y \geq 0 \, | \, D_i \right\}
        \leq
        \mathbb{P} \left\{ \sqrt{n / (1 - n^{-1/4})} + Y \geq 0 \right\}.
    \]
    Finally, substituting $Z = \frac{Y}{\sqrt{m - 1}} \sim \cN(0,1)$ yields
    \[
        \mathbb{P} \left\{ Z \geq - \sqrt{\frac{n - n^{3/4}}{m - 1}} \right\}
        \leq
        \mathbb{P} \left\{ \norm{w_i} + Y \geq 0 \, | \, D_i \right\}
        \leq
        \mathbb{P} \left\{ Z \geq - \sqrt{\frac{n / (1 - n^{-1/4})}{m - 1} }  \right\}
    \]
    Both sandwiching probabilities converge to $\mathbb{P} \left\{ Z \geq - \frac{1}{\sqrt{c}} \right\} = \frac{1}{2} \mathrm{erfc} \left( - \frac{1}{\sqrt{2c}} \right)$. Noting that $\mathbb{P} \left\{ D_i \right\} \to 1$ and $\mathbb{P} \left\{ D_i^c \right\} \to 0$ we finally have from \eqref{eq:totprob} that
    \[
        \mathbb{P} \left\{ E_i \right\} \to \frac{1}{2} \mathrm{erfc} \left( - \frac{1}{\sqrt{2c}} \right).
    \]
\end{proof}

\begin{lemma} \label{lem:concentration-intersect}
    Under the same conditions as Lemma \ref{lem:concentration} when $i \neq j$, 
    \begin{align*}
        \Prob\{E_i \cap E_j\} \rightarrow \frac{1}{4}\mathrm{erfc}\left(-\frac{1}{\sqrt{2c}}\right)^2.
    \end{align*}
\end{lemma}

We first prove a helper lemma.

\begin{lemma}
\label{lem:inner-product}
Let $X, Y \sim \mathcal{N}(0, I_n)$ be independent Gaussian random vectors. Then
\[
    \Prob \left\{ |\langle X, Y\rangle| \geq \sqrt{n^{3/2}/(1-\epsilon)} \right\} 
    \leq e^{-\sqrt{n}/2} + 2e^{-\epsilon^2 n / 4}.
\]
\end{lemma}

\begin{proof}
The distribution of $\langle X, Y \rangle $ is the same as that of $\norm{X} \langle e_1, Y \rangle$ = $\norm{X} Z$, with $Z \sim \mathcal{N}(0, 1)$ independently from $X$ and $e_1 = (1, 0, \ldots, 0)$. Then
\begin{align}
    \Prob \left\{ \norm{X} Z \geq \sqrt{n^{3/2}/(1-\epsilon)} \right\}
    \leq
    \Prob \left\{ Z \geq n^{1/4} \right\} + \Prob \left\{ \norm{X} \geq \sqrt{n / (1-\epsilon)} \right\}.
\end{align}
The second term is bounded by $e^{-\epsilon^2 n / 4}$. For the first term we have
\begin{align}
    \Prob \left\{ Z \geq n^{1/4}  \right\} = \frac{1}{2} \text{erfc} \left( n^{1/4} / \sqrt{2} \right) \leq \frac{1}{2} e^{- \sqrt{n}/2}.
\end{align}
The claim follows by noting that $Z$ is symmetric around 0.
\end{proof}

\begin{proof}{of Lemma \ref{lem:concentration-intersect}}
\begin{align}
    &\Prob \left\{ \left\langle w_i, \sum_{k =1}^m w_k \right\rangle \geq 0 \ \text{ and } \  \left\langle w_j, \sum_{k = 1}^m w_k \right\rangle \geq 0 \right\} \\
    &= 
    \Prob \left\{ \norm{w_i}^2 + \langle w_i, w_j \rangle + \langle w_i, X \rangle \geq 0 \ \text{ and } \  \norm{w_j}^2 + \langle w_i, w_j \rangle + \langle w_j, X \rangle\geq 0 \right\} \\
    &= (*).
\end{align}
where $X \sim \mathcal{N}(0, m - 2)$ is independent from $w_i$ and $w_j$. Fixing the orientation of $X$ does not change the probability so we choose $X = \norm{X}(1, 0, \ldots, 0)$ to get
\begin{equation}
    \label{eq:EiEj-1}
    (*) =     \Prob \left\{ \norm{w_i}^2 + \langle w_i, w_j \rangle + \norm{X} w_{i, 1} \geq 0 \ \text{ and } \  \norm{w_j}^2 + \langle w_i, w_j\rangle + \norm{X} w_{j, 1} \geq 0 \right \} 
\end{equation}
Similarly as in the proof of Lemma \ref{lem:concentration} define the event
\[
    D_i = \{ \omega \ : \ \norm{w_i(\omega)}^2 \in \left[ (1-\epsilon) n, n/(1-\epsilon) \right] \}
\]
and
\[
    D_X = \set{\omega \ : \ \norm{X(\omega)}^2 \in \left[ (1-\epsilon) n(m-2), n(m-2)/(1-\epsilon) \right] },
\]
which individually hold with probability at least $1 - 2\exp(-\epsilon^2 n / 4)$.
Let also 
\[
    D_{\innerprod{i}{j}} = \left\{ \omega \ : \ \abs{\langle w_i(\omega), w_j(\omega)\rangle} \leq \sqrt{ n^{3/2} / (1 - \epsilon)} \right\}.
\]
We have from Lemma \ref{lem:inner-product} that 
\[
    \Prob \left\{ D_{\innerprod{i}{j}} \right\} \geq 1 - \exp(\sqrt{n}/2) - 2 \exp(-\epsilon^2 n / 4).
\]
Choosing $\epsilon = \epsilon(n) = n^{-1/4}$ we get via union bound that
\[
    \Prob(D_i \cap D_j \cap D_X \cap D_{\innerprod{i}{j}}) \geq 1 - C_1 \exp(-C_2 \sqrt{n} ),
\]
for some $C_1, C_2 > 0$.

Let $D = D_i \cap D_j \cap D_X \cap D_{\innerprod{i}{j}}$ and denote
\[
    \alpha = \Prob \left\{ \norm{w_i}^2 + \langle w_i, w_j \rangle + \norm{X} w_{i, 1} \geq 0 \ \text{ and } \  \norm{w_j}^2 + \langle w_i, w_j \rangle + \norm{X} w_{j, 1} \geq 0 \, \big| \, D \right\}.
\]
Dividing the left-hand sides of inequalities in the argument of $\Prob$ by $m$ and denoting $\kappa = c^{-1}$, we have for large enough $n$ (so that the parenthesis is negative),
\begin{align}
    \nonumber
    \alpha \leq 
    \Prob &\left\{ w_{i, 1} \geq \sqrt{\frac{m}{m-2}} \frac{\kappa^{-1/2}}{1 - n^{-1/4}} \left( - \frac{\kappa}{1 - n^{-1/4}} - \kappa \sqrt{\frac{n^{-1/2}}{1 - n^{-1/4}}} \right) \right.  \\ 
    &\text{and} \left.
    w_{j, 1} \geq \sqrt{\frac{m}{m-2}}\frac{\kappa^{-1/2}}{1 - n^{-1/4}} \left( - \frac{\kappa}{1 - n^{-1/4}} - \kappa \sqrt{\frac{n^{-1/2}}{1 - n^{-1/4}}} \right)
    \, \bigg| \, D \right\}
    \label{eq:alpha-up}
\end{align}
(Note that unlike in the proof of Lemma \ref{lem:concentration}, conditioning affects the distribution of $w_i, w_j$ so we need to keep it.) Similarly,
\begin{align}
    \nonumber
    \alpha \geq 
    \Prob &\left\{ w_{i, 1} \geq \sqrt{\frac{m}{m-2}} \kappa^{-1/2} (1 - n^{-1/4}) \left( - \kappa (1 - n^{-1/4}) + \kappa \sqrt{\frac{n^{-1/2}}{1 - n^{-1/4}}} \right) \right.  \\ 
    &\text{and} \left.
    w_{j, 1} \geq \sqrt{\frac{m}{m-2}} (1 - n^{-1/4}) \left( - \kappa (1 - n^{-1/4}) + \kappa \sqrt{\frac{n^{-1/2}}{1 - n^{-1/4}}} \right)
    \, \bigg| \, D \right\}.
    \label{eq:alphadown}
\end{align}

Denote the conjunctions in the arguments of the probabilities in \eqref{eq:alpha-up} and \eqref{eq:alphadown} by $A_u$ and $A_\ell$. We have 
\begin{align*}
    \Prob \left\{ A_u \, | \, D \right\} - \Prob \left\{ A_u \right\} 
    &= \frac{\Prob \left\{ A_u \cap D \right\} }{\Prob \left\{ D \right\}} - \Prob \left\{ A_u \right\} \\
    &\leq \frac{\Prob \left\{ A_u \right\}}{\Prob \left\{ D \right\}} - \Prob \left\{ A_u \right\} \\
    &= \frac{\Prob \left\{ A_u \right\}}{\Prob \left\{ D \right\}} (1 - \Prob \left\{ D \right\}) \to 0, 
\end{align*}
since $\Prob \left\{ D \right\} \to 1$, and similarly for $A_\ell$, so we can remove the conditioning in the limit. Combining all of the above, we have that
\begin{align}
    \alpha \to \Prob \left\{ w_{i, 1} \geq  -\sqrt{\kappa} \ \text{ and } \ w_{j, 1} \geq -\sqrt{\kappa} \right \} = \left( \frac{1}{2} \text{erfc} \left( -\sqrt{\frac{\kappa}{2}} \right)  \right)^2
\end{align}
by independence of $w_{i, 1}$ and $w_{j, 1}$.
\end{proof}

\begin{theorem}\label{thm:injec-ofgaussian-lower-bound}
   Given a Gaussian weight matrix $W \in \Rea^{m \times n}$,  the layer $\relu(Wx)$ with $m / n \to c$ is not injective with probability $\to 1$ as $n \to \infty$ when $c < c^*$, where $c^*$ is the unique positive real solution to 
    \[
        \frac{1}{2} \mathrm{erfc} \left(\frac{1}{\sqrt{2c}} \right) = \frac{1}{c}.
    \]
    The numerical value of $c^*$ is $\approx 3.4$.
\end{theorem}

\begin{proof}
    Let $X_i = \mathbf{1}_{\innerprod{w_i}{\frac{1}{m} \sum_{j = 1}^{m} w_j} \geq 0}$, the indicator function of the event $E_i$. The expected number of $w_i$ with a positive inner product with $\sum_{i = 1}^m w_i$ is by the linearity of expectation equal to
    \[
        \mathbb{E} \left( \sum_{i = 1}^m X_i \right ) = m \cdot \mathbb{P} \left\{ E_i \right\} =: mp.
    \]
    By Chebyshev's inequality 
    \[
        \mathbb{P} \left( \abs{ \frac{1}{m} \sum_{i=1}^m X_i - p} \geq t \right) \leq \frac{\sigma^2}{t^2}
    \]
    where $\sigma^2 = \mathbb{V} \left(\tfrac{1}{m} \sum_{i=1}^m X_i \right)$ is the variance of the sum. We compute
    \[
        m^2 \sigma^2 = \mathbb{V} \left( \sum_{i = 1}^m X_i \right) = \mathbb{E} \left( \sum_{i=1}^m X_i\right)^2 - \left( \mathbb{E} \, \sum_{i=1}^m X_i \right)^2 = mp + \sum_{i \neq j} \mathbb{P}(E_i \cap E_j) - (mp)^2.
    \]
    Since $\mathbb{P} (E_i \cap E_j) \to p^2$ by Lemmas \ref{lem:concentration} and \ref{lem:concentration-intersect}, we have for any $i$ and $j \neq i$
    \begin{align*}
        \sigma^2 &= \frac{mp + m(m - 1) \mathbb{P}(E_i \cap E_j) - (mp)^2}{m^2} \to 0.
    \end{align*}
    Thus indeed for any $t > 0$
    \[
        \mathbb{P} \left( \abs{ \frac{1}{m} \sum_{i=1}^m X_i - p } \geq t \right) \to 0.
    \]
    Finally,
    \begin{align}
    \mathbb{P}(\text{noninjectivity}) \geq \mathbb{P} \left( \sum_{i=1}^m X_i > m - n \right) = \mathbb{P} \left( \frac{1}{m} \sum_{i=1}^m X_i > 1 -  \frac{1}{c} \right).
    \end{align}
    Combining with Lemma \ref{lem:concentration} we get that
    \[
        \frac{1}{m} \sum_{i = 1}^m X_i \stackrel{\mathbb{P}}{\to} \frac{1}{2} \mathrm{erfc} \left( - \frac{1}{\sqrt{2c}} \right) 
    \]
    Thus\begin{equation}
        \mathbb{P}(\text{noninjectivity}) \geq \mathbb{P} \left( \frac{1}{m} \sum_{i=1}^m X_i > 1 -  \frac{1}{c} \right) \to
        \begin{cases}
            0 & c > c^* \\
            1 & c < c^*.
        \end{cases}
    \end{equation}
\end{proof}


\subsection{Proof of Theorem \ref{thm:inv-lip:global}}
\label{sub:proof-inv-lip}

\begin{figure}
    \centering
    \begin{subfigure}{.4\linewidth}
        \centering
        \includegraphics[width=\linewidth]{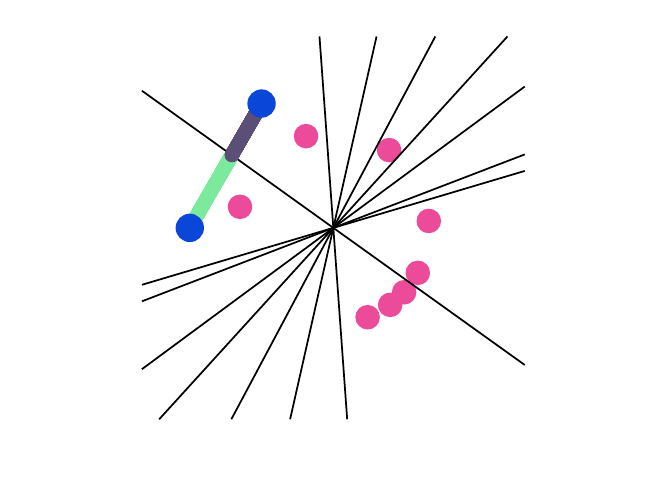}
        \subcaption{}
        \label{fig:dss-aid:adjacent}
    \end{subfigure}
    \begin{subfigure}{.4\linewidth}
        \centering
        \includegraphics[width=\linewidth]{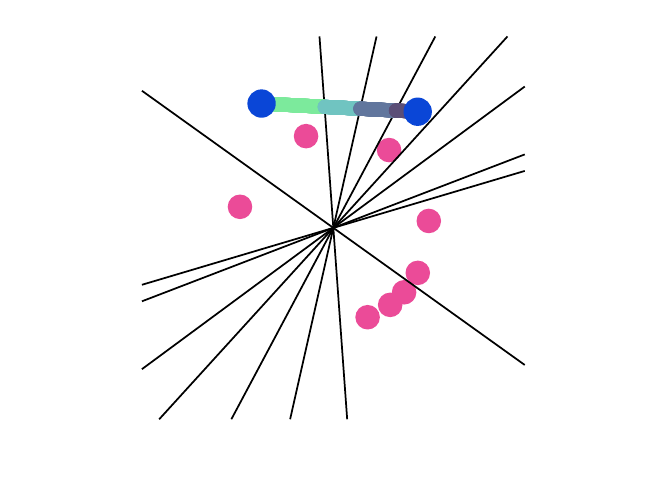}
        \subcaption{}
        \label{fig:dss-aid:distant}
    \end{subfigure}
    \caption{A Proof aid of the DSS condition when $W \in \Rea^{8 \times 2}$. The blue dots are $x_0$ and $x_1$, the pink are the weight matrix rows, black lines denote the boundaries between adjacent wedges, and the multi-colored line is $\ell^{x_0,x_1}(t)$. This line changes color each time it crosses into a new wedge.}
    \label{fig:dss-aid}
\end{figure}

We divide the proof into parts for ease of understanding.

\begin{lemma}[Inverse Lipschitz Constant: Face Adjacent Wedges] \label{lem:inv-lip:adjacent:face-intersect}
    Let $W \in \Rea^{m \times n}$ have a DSS w.r.t. every $x \in \Rea^n$ and $x_0, x_1 \in \Rea^n$.
    Define 
    \begin{align}
        \forall t \in [0, 1], \quad \ell^{x_0,x_1}(t) = (1 - t)x_0 + t x_1
    \end{align}
    and suppose that $x_0$ and $x_1$ are such that there is a $t' \in (0, 1)$ such that for all $t \in [0, 1]$ 
    \begin{align}
        \label{eqn:lem:inv-lip:adjacent:face-intersect:adj-condition}
        W|_{S(\ell^{x_0,x_1}(t),W)} = \begin{cases}
            W|_{S(x_0, W)} &\text{ if } t < t'\\
            W|_{S(x_1, W)} &\text{ if } t' < t\\
        \end{cases}.
    \end{align}
    Suppose further that $x_0, x_1$ are such that there is some $\delta > 0$ such that for all $\delta x \in \Rea^n$, if $\norm{\delta x}_2 < \delta$ then there is a $t' + \delta t$ such that 
    \begin{align}
        \label{eqn:lem:inv-lip:adjacent:face-intersect:face-condition}
        W|_{S(\ell^{x_0,x_1 + \delta x}(t),W) \geq 0} = \begin{cases}
            W|_{S(x_0, W)} &\text{ if } t < t' + \delta t\\
            W|_{S(x_1, W)} &\text{ if } t' + \delta t < t\\
        \end{cases}.
    \end{align}
    Then 
    \begin{align}
        \norm{\relu(Wx_0) - \relu(Wx_1)}_2 \geq \frac{1}{\sqrt{2}}\min(\sigma(W|_{S(x_0,W)}), \sigma(W|_{S(x_1,W)}))\norm{x_0 - x_1}_2 ,
    \end{align}
    where $\sigma(M)$ is the smallest singular value of the matrix $M$.
\end{lemma}

\begin{remark} \label{rmk:inv-lip:adjacent:face-intersect}
    The conditions of Lemma \ref{lem:inv-lip:adjacent:face-intersect} on $x_0$ and $x_1$ may look very odd, but they have a very natural geometric meaning. The DSS can be thought of as slicing $\Rea^n$ into wedges by a series of hyperplanes that have the rows of $W$ as normals.
    
    The condition in (\ref{eqn:lem:inv-lip:adjacent:face-intersect:adj-condition}) is interpreted as that the line segment that connects $x_0$ to $x_1$ passes from $x_0$'s wedge into $x_1$'s wedge without passing through any wedges in between (see Figure \ref{fig:dss-aid:adjacent}, as opposed to Figure \ref{fig:dss-aid:distant}). The implies that $x_0$ and $x_1$ must be in wedges that share a boundary. The condition in (\ref{eqn:lem:inv-lip:adjacent:face-intersect:face-condition}) requires that the wedges of $x_0$ and $x_1$ share a face, and not just a corner. 
\end{remark}

\begin{proof}
    We denote 
    \begin{align}
        W|_{S(x_0, W)} = W_0, W|_{S(x_1, W)} = W_1, W_{0\cap 1} = W_0 \cap W_1, W_0 = \begin{bmatrix} W_{0\cap 1}\\ W_{0 \setminus 1}\end{bmatrix}, W_1 = \begin{bmatrix} W_{0\cap 1}\\ W_{1 \setminus 0}\end{bmatrix}.
    \end{align}
    First we will show that if $w_i, w_{i'} \in W_{0 \setminus 1}$, then $w_i$ and $w_{i'}$ must be parallel. From equation (\ref{eqn:lem:inv-lip:adjacent:face-intersect:adj-condition}) and the continuity of $\relu(W\ell^{x_0,x_1}(t))$ w.r.t. $t$, we have that there is a $t' \in (0, 1)$ such that
    \begin{align}
        \innerprod{w_i}{(1 - t')x_0 + t'x_1} = 0 = \innerprod{w_{i'}}{(1 - t')x_0 + t'x_1}.
    \end{align}
    If $w_i$ and $w_{i'}$ are not parallel, then let $\delta x$ be some vector $0 < \norm{\delta x} < \delta$ that is perpendicular to $w_i$ but not $w_{i'}$ then
    \begin{align}
        \innerprod{w_i}{(1 - t')x_0 - t'\delta x + t'x_1} &= - t\innerprod{w_i}{\delta x} = 0 ,\\
        \innerprod{w_{i'}}{(1 - t')x_0 - t'\delta x + t'x_1} &= - t\innerprod{w_{i'}}{\delta x} \neq 0 ,
    \end{align}
    which contradicts (\ref{eqn:lem:inv-lip:adjacent:face-intersect:face-condition}). W.l.o.g. the same argument applies to $W_{1 \setminus 0}$ and also it is straight forward to see that all elements of $W_{0 \setminus 1}$ must be anti-parallel to all elements of $W_{1 \setminus 0}$. From $W_{1 \setminus 0}$ and $W_{0 \setminus 1}$ parallelism, there is a $c \geq 0$ such that for all $x \in \Rea^n$, 
    \begin{align}
        \norm{W_{1 \setminus 0}x} = c^2 \norm{W_{0 \setminus 1}x}.
    \end{align}
    Assume that $c \geq 1$, then
    \begin{align}
        \norm{W_0 x_0 - W_1 x_1}^2_2 &= \norm{W_{0 \cap 1} x_0 - W_{0\cap 1} x_1}^2_2 + \norm{W_{0 \setminus 1} x_0}^2_2 + \norm{W_{1 \setminus 0} x_1}^2_2 \\
        &= \norm{W_{0 \cap 1} x_0 - W_{0\cap 1} x_1}^2_2 + \norm{W_{0 \setminus 1} x_0}^2_2 + c^2\norm{W_{0 \setminus 1} x_1}^2_2 \\
        &\geq \norm{W_{0 \cap 1} x_0 - W_{0\cap 1} x_1}^2_2 + \norm{W_{0 \setminus 1} x_0}^2_2 + \norm{W_{0 \setminus 1} x_1}^2_2 \\
        &\geq \norm{W_{0 \cap 1} x_0 - W_{0\cap 1} x_1}^2_2 + \frac{1}{2}\norm{W_{0 \setminus 1} (x_0 - x_1)}^2_2 \\
        &\geq \frac{1}{2}\norm{W_{0} x_0 - W_{0} x_1}^2_2\\
        &\geq \frac{\sigma(W_0)^2}{2}\norm{x_0 - x_1}^2_2.
    \end{align}
    The antepenultimate inequality comes as from the definition of $W_{0 \setminus 1}$, we have that $W_{0 \setminus 1} x_0$ and $-W_{0 \setminus 1} x_1$ are the same sign, thus (\ref{eqn:lem-useful-ineq:1}) applies. In the case that $c < 1$, then the rolls of $W_{0 \setminus 1}$ and $W_{1 \setminus 0}$ can be switched, and the same result (with $\sigma(W_1)$ in place of $\sigma(W_0)$) is obtained. In either case,
    
    \begin{align}
        \norm{W_0 x_0 - W_1 x_1}^2_2 \geq \frac{1}{2}\min\left (\sigma(W|_{S(x_0,W)}^2), \sigma(W|_{S(x_1,W)})^2\right )\norm{x_0 - x_1}_2^2.
    \end{align}
\end{proof}

\begin{lemma}[Inverse Lipschitz Constant: Connected through Faces] \label{lem:inv-lip:con-through-faces}
    Let $W \in \Rea^{m \times n}$ have a DSS w.r.t. every $x \in \Rea^n$. Let $x_0, x_1$ be such that the line connecting passes through $n_t$ wedges, and through their faces (in the sense of Lemma \ref{rmk:inv-lip:adjacent:face-intersect}). Then 
    \begin{align}
        \norm{\relu(Wx_0) - \relu(Wx_1)}_2 \geq \frac{1}{\sqrt{2n_t}}\min_{t \in [0, 1]}\sigma(W|_{S(\ell^{x_0, x_1}(t),W)}) \norm{x_0 - x_1}_2.
    \end{align}
\end{lemma}

\begin{proof}
    Let $t_1 = 0$, $t_{n_t} = 1$, and let $T = \{ t_k\}_{k = 1}^{n_t}$ such that $t_{k} < t_{k + 1}$ for $k < n_t$ and $x_{k} \coloneqq \ell^{x_0,x_1}(t_k)$ are each in different wedges. Let $c = \frac{1}{\sqrt{2}}\min_{t \in [0, 1]}\sigma(W|_{S(\ell^{x_0, x_1}(t),W)})$, then 
    \begin{align} \label{eqn:lem:inv-lip:con-through-faces:1}
        c\norm{x_0 - x_1}_2 \leq \sum_{k=1}^{n_t - 1}c\norm{x_{t_k} - x_{t_{k+1}}}_2
    \end{align}
    then by Lemma \ref{lem:inv-lip:adjacent:face-intersect},
    \begin{align} \label{eqn:lem:inv-lip:con-through-faces}
        \sum_{k=1}^{n_t - 1}c\norm{x_{t_k} - x_{t_{k+1}}}_2 \leq \sum^{n_t-1}_{k = 1} \norm{\relu(Wx_{t_k}) - \relu(Wx_{t_{k+1}})}_2
    \end{align}
    and by Lemma \ref{lem:colinear-additivity} we have
    \begin{align} \label{eqn:lem:inv-lip:con-through-faces:3}
        \sum^{n_t-1}_{k = 1} \norm{\relu(Wx_{t_k}) - \relu(Wx_{t_{k+1}})}_2 \leq \sqrt{n_t}\norm{\relu(Wx_0) - \relu(Wx_1)}_2.
    \end{align}
    Combining (\ref{eqn:lem:inv-lip:con-through-faces:1}) - (\ref{eqn:lem:inv-lip:con-through-faces:3}) yields
    \begin{align}
        \norm{\relu(Wx_0) - \relu(Wx_1)}_2 \geq \frac{1}{\sqrt{2n_t}}\min_{t \in [0, 1]}\sigma(W|_{S(\ell^{x_0, x_1}(t),W)}) \norm{x_0 - x_1}_2.
    \end{align}
\end{proof}

\begin{proof}[Proof of Theorem \ref{thm:inv-lip:global}]
    Let $x_0$ and $x_1$ be given. If the line connecting $x_0$ and $x_1$ does not pass through any wedge corners, then we can apply Lemma \ref{lem:inv-lip:con-through-faces} directly, and get that
    \begin{align}
        \norm{\relu(Wx_0) - \relu(Wx_1)}_2 \geq \frac{1}{\sqrt{2n_t}}\min_{t \in [0, 1]}\sigma(W|_{S(\ell^{x_0, x_1}(t),W)}) \norm{x_0 - x_1}_2.
    \end{align}
    We now argue that the number of wedges that $\ell^{x_0,x_1}(t)$ passes through (and so $n_t$) is at most $m$. For each $j \in [[m]]$,
    \begin{align}
        \relu(W\ell^{x_0,x_1}(t))|_{j} = \innerprod{w_j}{\ell^{x_0,x_1}(t)}
    \end{align}
    is monotone increasing or decreasing. This implies that each $w_j \in W$ can enter or exit the DSS w.r.t. $\ell^{x_0, x_1}(t)$ at most once, therefore the total number of unique DSS w.r.t. $\ell^{x_0, x_1}(t)$ (i.e. wedges $\ell^{x_0, x_1}(t)$ pass through) is at most $m$. Hence, $n_t \leq m$. Clearly
    \begin{align}
        \min_{t \in [0, 1]}\sigma(W|_{S(\ell^{x_0, x_1}(t),W)}) \geq \min_{x \in \Rea^n} \sigma(W|_{S(x,W)}),
    \end{align}
    hence we have
    \begin{align}
        \norm{\relu(Wx_0) - \relu(Wx_1)}_2 \geq \frac{C}{\sqrt{m}}\norm{x_0 - x_1}_2.
    \end{align}
    
    Now we show that if Lemma \ref{lem:inv-lip:con-through-faces} does not apply to two points (namely, the line $\ell^{x_0, x_1}(t)$ passes through a corner in the sense of Remark \ref{rmk:inv-lip:adjacent:face-intersect}), then the two points can be perturbed an arbitrarily small amount, so that the perturbed points do satisfy Lemma \ref{lem:inv-lip:con-through-faces}.  The corners (again, as in Remark \ref{rmk:inv-lip:adjacent:face-intersect}) describe the points which are orthogonal to at least two $w_{j_1}, w_{j_2} \in W$. $w_{j_1}$ and $w_{j_2}$ must not be parallel to each other (otherwise Lemma \ref{lem:inv-lip:adjacent:face-intersect} would apply), thus the set of points orthogonal to both $w_{j_1}$ and $w_{j_2}$ constitute a $n-2$ dimensional linear space in $\Rea^n$.
    
    Let $x_0$ and $x_1$ be such that $\ell^{x_0, x_1}(t)$ intersects one of these corners. By considering  $\tilde x_0 = \delta x + x_0, \tilde x_1 = \delta x + x_1$ where $\delta x$ is perpendicular to $x_1 - x_0$, we can obtained a line $\ell^{\tilde x_0, \tilde x_1}(t)$ so that $\ell^{\tilde x_0, \tilde x_1}(t)$ and $\ell^{x_0, x_1}(t)$ do not intersect. The choice of $\delta x$ is $n-1$ dimensional, thus for every $\delta > 0$ and $w_{j_1}$ and $w_{j_2}$ (that are non-perpendicular) there is a $\delta x$ so that $\norm{\delta x}_2 < \delta$ and $\ell^{\tilde x_0, \tilde x_1}(t)$ does not intersect the corner of $w_{j_1}$ and $w_{j_2}$.
    
    Consider a sequence of $\tilde x^{(i)}_0, \tilde x^{(i)}_1, i = 1,\dots$ such that $\lim_{i \rightarrow \infty} (\tilde x^{(i)}_0, x^{(i)}_1) = (x_0, x_1)$ and $\ell^{\tilde x^{(i)}_0, x^{(i)}_1}(t)$ does not pass through a corner for any $i$. Given that $\norm{\cdot}_2$ and $\relu(W(\cdot))$ are continuous and so by Lemma \ref{lem:inv-lip:con-through-faces}
    \begin{align}
        \norm{\relu(W\tilde x^{(i)}_0) - \relu(W\tilde x^{(i)}_1)}_2 - \frac{C}{\sqrt{2m}}\norm{\tilde x^{(i)}_0  - \tilde x^{(i)}_1}_2 \geq 0,
    \end{align}
    thus
    \begin{align}
        &\norm{\relu(W x_0) - \relu(Wx_1)}_2 - \frac{C}{\sqrt{2m}}\norm{x_0  - x_1}_2, \\
        &=\lim_{i \rightarrow \infty} \norm{\relu(W\tilde  x^{(i)}_0) - \relu(W\tilde  x^{(i)}_1)}_2 - \frac{C}{\sqrt{2m}}\norm{\tilde  x^{(i)}_0 -  x^{(i)}_1}_2 \geq 0.
    \end{align}
    and so
    \begin{align}
        \norm{\relu(W x_0) - \relu(Wx_1)}_2 \geq \frac{C}{\sqrt{2m}}\norm{x_0  - x_1}_2.
    \end{align}
\end{proof}

\begin{remark}[Factor of $\frac{1}{\sqrt{m}}$ in Theorem \ref{thm:inv-lip:global}.]
    Throughout this paper we define the discrete norm of $y \in \Rea^d$ as
    \begin{align}
        \norm{y}_2 = \left (\sum^{d}_{j = 1}[y]_j^2\right )^{\frac{1}{2}}.
    \end{align}
    This is to be contrasted with the norm that arise from the discretization of the $L_2$ function norm on a finite domain. For example, if we instead thought of $y$ as a discrete sampling of a continuous function $\tilde y \in L_2([0, 1])$, such that $\forall j = 1,\dots, d$
    \begin{align}
        \tilde y\left (\frac{j - 1}{m}\right) = [y]_j ,
    \end{align}
    then we could approximate the $L_2([0, 1])$ norm of $\tilde y$ by
    \begin{align}
        \norm{\tilde y}_{L_2([0,1])} \approx \norm{y}_{l_2([0,1])} \coloneqq \frac{1}{\sqrt{m}}\left (\sum^{d}_{j = 1}[y]_j^2\right )^{\frac{1}{2}}.
    \end{align}
    If we express Theorem \ref{thm:inv-lip:global} in terms of $\norm{\cdot}_{l_2([0,1])}$, then it would become
    \begin{align}
        \norm{\relu(Wx_0) - \relu(Wx_1)}_{l_2(0,1)} \geq \frac{C(W)}{m}\norm{x_0 - x_1}_2 .
    \end{align}
\end{remark}

\subsection{Theorem \ref{thm:suff-conv-injectivity}}
\label{sec:conv-thm}

\begin{example}[Applying Theorem \ref{thm:suff-conv-injectivity}, One Channel] \label{ex:conv-injectivity-1}
    Consider a layer of the form\\
    \(
        \text{Reshape}(\relu(Wx))
    \)
    where $W = [C_1^T, \cdots, C_q^T]^T$ and each $C_k$ is a convolution operator with kernel $c_k$. Suppose further that $W \in \Rea^{4\times4\times1024 \times 100} = \Rea^{16384 \times 100}$ (as in \cite{radford2015unsupervised}). The reshaping operator takes the $16384$ single-channel output of $W$ and transforms it into a multi-channel signal. This is necessary for subsequent convolutions, but plays no role in injectivity. Let $\numconv = 8$, and the $2 \times 2$ convolution kernels be given as
    \begin{align}
        c_1 = \begin{bmatrix}
        3 & -1\\
        -1 & -1
        \end{bmatrix}, \quad 
        &c_2 = \begin{bmatrix}
        -1 & 3\\
        -1 & -1
        \end{bmatrix}, \quad
        c_3 = \begin{bmatrix}
        -1 & -1\\
        3 & -1
        \end{bmatrix}, \quad
        c_4 = \begin{bmatrix}
        -1 & -1\\
        -1 & 3
        \end{bmatrix}\\
        & c_5 = -c_1, c_6 = -c_2, c_7 = -c_3, c_8 = -c_4.
    \end{align}
    Directly proving that a $16384 \times 100$ dimension operator has a DSS w.r.t. every $x \in \Rea^{100}$ is daunting. However, since each layer of the operator is given by one of only 8 simple convolutions, we can leverage Theorem \ref{thm:suff-conv-injectivity} to significantly simplify the problem. Choosing $P = (2,2)$, implies that $\cP_{(2,2)}(c_k) = c_k$, and so $W|_{\cP_{(2,2)}} = \bigcup_{k = 1}^8 c_k$. Further, it is easy to see, that $\{ c_1, c_2, c_3, c_4\}$ is a basis for $\Rea^{2 \times 2}$, so Corollary \ref{cor:minimal-expansivity} applies,  $W|_{\cP_P}$ has a DSS of $\Rea^4$, and by Theorem \ref{thm:suff-conv-injectivity}, $\relu(Wx)$ is injective.
    
    An example when a layer is injective but $P$ in Theorem \ref{thm:suff-conv-injectivity} must be greater than $O$ is when $W$ is a convolution of 4 kernels of width 3
    \begin{align}
        c_1 = \begin{bmatrix} 1& 0& -1\end{bmatrix}, \quad c_2 = \begin{bmatrix} 1& 0& 1\end{bmatrix}, \quad c_3 = \begin{bmatrix} -1& 0& 1\end{bmatrix}, \quad c_4 = \begin{bmatrix} -1& 0& -1\end{bmatrix}.
    \end{align}
    If we choose $P = (3)$, then $W|_{\cP_{(3)}} = \bigcup_{k = 1}^4 \{c_k\}.$ only has four elements, and so cannot have a DSS of $\Rea^{3}$ w.r.t. every $x \in \Rea^3$ by Corollary \ref{cor:minimal-expansivity}. If we however choose $P = (4)$, then $\cP_{(4)}(c_k) = \left \{\begin{bmatrix} c_k& 0 \end{bmatrix}, \begin{bmatrix} 0& c_k \end{bmatrix} \right \}$ and $W|_{\cP_{(4)}}$ has a DSS of $\Rea^4$ w.r.t. all $x \in \Rea^4$ (from Corollary \ref{cor:minimal-expansivity}), so $W$ has a DSS w.r.t. all $x \in \Rea^{N}$.
\end{example}

In order to prove Theorem \ref{thm:suff-conv-injectivity}, we must first prove Lemma \ref{lem:domain-decomp}, which we do here.

\begin{proof}
    For every $k = 1,\dots, K$ $W_k|_{S(x,W_k)}$ has a DSS of $\Omega_k$ with respect to $P_{\Omega_k}(x)$ where $P_{\Omega_k}$ is the orthogonal projection of $\Rea^n$ onto $\Omega_k$. For every $w_{k, \ell}$, 
    \begin{align}
        \innerprod{w_{k, \ell}}{P_{\Omega_k}(x)} = \innerprod{w_{k, \ell}}{x}
    \end{align}
    thus $S(x, W_k) \subset S(x, W)$. From $\Omega_k \subset \Span(W_k|_{S(x,W_k)})$ for each $k$ we have a set spanning $\Omega_k$ that lie in $W|_{S(x,W)}$, hence 
    \begin{align}
        \Rea^n  = \Span\left(\Omega_1, \dots, \Omega_K \right)\subset \Span\left(\bigcup_{k = 1}^{K}\bigcup^{N_k}_{\ell = 1}\{ w_{k, \ell}\}\right) = \Span(W)
    \end{align}
    contains a DSS of $\Rea^n$ w.r.t. $x$. The set $\bigcup_{k = 1}^{K}\{ w_{k, \ell}\}_{\ell = 1}^{N_k}$ has no dependence on $x$, thus it is true for all $x \in \Rea^n$. 
\end{proof}

\begin{lemma} \label{lem:augmented-conv-kernel}
    Given a convolution operator $C \in \Rea^{N \times N}$. Let $0 \leq V$ be such that $V + O \leq N$. For each $x \in C$,
    \begin{align}
        \mathrm{aug}_{V:V+O}(x) \in C ,\quad \text{ where } \quad 
        (\mathrm{aug}_{V:V+O}(x))_J = \begin{cases}
            (x)_{J - V} &\text{ if } 1 + V \leq J \leq V + O \\
            0 & \text{ otherwise}
        \end{cases}.
    \end{align}
    Note that $(\mathrm{aug}_{V:V+O}(x))_J$ restricted to the indices $1 + V$ through $V + P$ is exactly $x$.
\end{lemma}

The above is always well defined, as $c_{H-T+O+1}$ is well defined iff $1 \leq H-T+O+1 \leq O$ (recall that the kernel $c$ is of width $O$), and so
\begin{align}
    1 \leq H-T+O+1 \leq O \iff -O \leq H - T \leq -1 \iff 1 \leq T - H \leq O
\end{align}

Now we proceed with the proof of Theorem \ref{thm:suff-conv-injectivity}.

\begin{proof}[Proof of Theorem \ref{thm:suff-conv-injectivity}]
    The strategy for this proof will be to use Lemma \ref{lem:domain-decomp} by decomposing $\Rea^n$ into some number of different domains $\{ \Omega_v\}_{v = 1}^{n_v}$, each of which are a restriction of $\Rea^n$ to $P$ non-zero components. For each $\Omega_v$, the elements of $W$ that lie in $\Omega_v$ can be identified as the elements in $W|_{\cP_P}$. If for one $v$ the components of $W|_{\cP_P}$ form a DSS of $\Omega_v$ w.r.t. every $x \in \Omega_v$, then it has a DSS for every such $\Omega_k$, so we can apply Lemma \ref{lem:domain-decomp} and get that $W$ has a DSS of $\Rea^n$ w.r.t. all $x \in \Rea^n$.
    
    Given an offset $V' \geq 0$ such that $V' + P \leq N$, define $\Omega_{V'}$ as the subspace of all vectors $x \in \Rea^N$ such that 
    \begin{align}
        \Omega_{V'} = \left \{ x \in \Rea^N \colon x_J = 0 \text{ if } 1 + V' \not \leq J \text{ or } J \not \leq V' + P
        \right \} .
    \end{align}
    From Lemma \ref{lem:augmented-conv-kernel}, for any $k = 1, \dots, n_v$, if $x^{k,P} \in \cP_P(c_k)$, $C_k$ is a submatrix of $W$ and
    \begin{align}
        \mathrm{aug}_{V':V'+P}(x^{k,P}) \in C_k. 
    \end{align}
    Further, for any such $x^{k,P}$, 
    \begin{align}
        \mathrm{aug}_{V':V'+P}(x^{k,P}) \in \Omega_{V'}.
    \end{align}
    If $W|_{\cP_P}$ contains a DSS of $\R^P$ w.r.t. all $x \in \R^P$, then 
    \begin{align}
        \label{eqn:lem:suff-conv:w-restr-dss}
        \mathrm{aug}_{V':V'+P}(W|_{\cP_P}) \text{ contains a DSS of } \Omega_{V'} \text{ w.r.t. all } x \in \Omega_{V'}.
    \end{align}
    This follows from Lemma \ref{lem:domain-decomp}. From Lemma \ref{lem:augmented-conv-kernel} for any $V'$ such that
    \begin{align}
        \label{eqn:lem:suff-conv:w_z-aug}
        \mathrm{aug}_{V':V'+P}(W|_{\cP_P}) \in W,
    \end{align}
    if $W|_{\cP_P}$ has a DSS of $\R^P$ w.r.t. all $x \in \R^P$, then $W$ contains a DSS of $\Omega_{V'}$ for all $0 \leq V' \leq N - P$. Finally, note that $\Span(\{\Omega_{V'}\}_{V' = 0}^{N - P})$, and so using (\ref{eqn:lem:suff-conv:w-restr-dss}), (\ref{eqn:lem:suff-conv:w_z-aug}) we can apply Lemma \ref{lem:domain-decomp} and find that $W$ contains a DSS of $\Rea^N$ w.r.t. all $x \in \Rea^N$.
\end{proof}

For a multi-channel input (with $n_c$ channels) $x \in \underbrace{\Rea^N \times \Rea^N \times \dots \times \Rea^N}_{n_c\text{ times} }$, a multi-channel convolution $C$ on $x$ is given by $Cx = \sum_{q = 1}^{n_c} C_q x_q$ where $C_o$ is a convolution on $\Rea^N$ (defined by Definition \ref{def:conv-operator}) and $x_o \in \Rea^N$ is the restriction of $x$ to the $o$'th channel. Because of the additive structure of multi-channel convolutions a $n_c$ over $\Rea^N = \Rea^{N_1}\times \dots \times \Rea^{N_p}$ dimensional domain of width $O = (O_1, \dots, O_p)$ with kernels $c_1, \dots, c_{n_c}$ is equivalent to a single convolution of width $(O, n_c) = (O_1, \dots, O_p, n_c)$ over $\Rea^{(N, p)} = \Rea^{N_1}\times \dots\times \Rea^{N_p} \times \Rea^{n_c}$. This follows from

\begin{align*}
   (Cx)_J = \sum_{q = 1}^{n_c} (C_q x_q)_J = \sum_{q = 1}^{n_c}\sum_{I = 1}^O (c_q)_{O-I-1}(x_q)_{J+I} = \sum_{(I,n_c) = 1}^{(O, q)} c_{(O,n_c)-(I, q)-1}x_{(J, q)+(I, q)}.
\end{align*}

\section{Robustness to Layer, Instance and Group Normalization}

\label{sec:layer-instance-group-normalization}

As described in Section \ref{sec:norm-training-runtime}, layer, instance and group normalization take place during both training and execution and, unlike weight and spectral normalization, the normalization is done on the input/outputs of layers instead of on the weight matrices. For these normalizations (\ref{eqn:deep-network-definition}) is modified so that it becomes
\begin{align}
    \label{eqn:normalized-deep-network}
    N(z) = \phi_L(W_L M_L(\cdots \phi_2(W_2 M_2(\phi_1(W_1 z + b_1)) + b_2) \cdots + b_L))
\end{align}
where $M_\ell \colon \Rea^{n_{i + 1}} \rightarrow \Rea^{n_{i + 1}} $ are normalization functions that are many-to-one. In general $\phi_\ell(W_\ell M_\ell(\cdot) + b_\ell)$ will not be injective for any $\phi_\ell, W_\ell, b_\ell$ on account of $M_\ell$, but for all of the mentioned normalization techniques we can get near injectivity. Before we descend into the particular we make the following observation about normalization methods that obey a certain structure.  

\begin{definition}[Scalar-Augmented Injective Normalization] \label{def:scalar-aug-injectiv-norm}
    Let $M_\ell(x) \colon \Rea^n \rightarrow \Rea^n$ be a normalization function that is understood to be many-to-one. We say that $M_\ell(x)$ is scalar-augmented injective if there exists a function $m_\ell(x) \colon \Rea^n \rightarrow \Rea^k$ where $k \ll n$ and $\tilde M_\ell \colon \Rea^n\times\Rea^k \rightarrow \Rea^n$ such that
    \begin{align}
        M_\ell(x) \coloneqq \tilde M_\ell(x ; m_\ell(x))
    \end{align}
    and $\tilde M_\ell(x; m_\ell(x))$ is injective on $x$ given $m_\ell(x)$. 
\end{definition}
An example of a normalization function that is scalar-augmented injective is
\begin{align}
    \label{eqn:example-normalization-fun}
    M_\ell(x) = \frac{x}{\norm{x}_2}.
\end{align}
For this choice of $M_\ell$, $k = 1$, and 
\begin{align}
    \tilde M_\ell(x; c) = \frac{x}{c} \quad m_\ell(x) = \norm{x}_2.
\end{align}
With this definition, we can prove the following trivial but useful result

\begin{lemma}[Restricted Injectivity of Scalar-Augmented Normalized Networks] \label{lem:restr-injec-scalar-aug-norm}
    Let $N$ be a deep network of the form in (\ref{eqn:normalized-deep-network}) and let each $\phi_\ell(W_\ell \cdot )$ be layer-wise injective. Let the normalization functions $\{M_\ell\}_{\ell = 1, \dots ,L}$ each be scalar-augmented injective. Then given $\{m_\ell(x)\}_{\ell = 1,\dots, L}$, the network
    \begin{align}
    \tilde N(z; m_1, \dots, m_\ell) = \phi_L(W_L \tilde M_L(\cdots \phi_2(W_2 \tilde M_2(\phi_1(W_1 z + b_1);m_2) + b_2)\cdots + b_L;m_L))
    \end{align}
    is injective.
\end{lemma}

\begin{proof}
    The proof of Lemma \ref{lem:restr-injec-scalar-aug-norm} follows from a straightforward application of induction, combined with Definition \ref{def:scalar-aug-injectiv-norm}.
\end{proof}

\begin{remark}
    Note that Lemma \ref{lem:restr-injec-scalar-aug-norm} implies that for a fixed $\{m_\ell(x)\}_{\ell = 1,\dots, L}$, there is at most one value of $z$ such that 
    \begin{align}
        \label{eqn:rmk:inject-restr-cond}
        \tilde N(z; m_1,\dots,m_\ell) = N(z) ,
    \end{align}
    where $N(z)$ is given by (\ref{eqn:normalized-deep-network}), and that the $z$'s on both sides of (\ref{eqn:rmk:inject-restr-cond}) are the same. It is still entirely possible that there are is another choice of $z'$, $\{m_\ell(x)\}_{\ell = 1,\dots, L}$ such that 
    \begin{align}
        \tilde N(z; m_1,\dots,m_\ell) = \tilde N(z'; m_1',\dots,m_\ell').
    \end{align}
    An example of this would be if $M_\ell$ is of the form in (\ref{eqn:example-normalization-fun}), then
    \begin{align}
        \tilde M_\ell(x;m_\ell(x)) = \tilde M_\ell(2 x;2 m_\ell(x)).
    \end{align}
    In other words, Lemma \ref{lem:restr-injec-scalar-aug-norm} implies that the deep network is injective (in $z$) for a fixed $\{m_\ell(x)\}_{\ell = 1,\dots, L}$, but it may still not be injective for all $z$ and $\{m_\ell(x)\}_{\ell = 1,\dots, L}$.
\end{remark}

With Lemma \ref{lem:restr-injec-scalar-aug-norm} in tow, we can show that layer, instance, and group normalization are all scalar-augmented injective normalizations, so Lemma \ref{lem:restr-injec-scalar-aug-norm} applies and yields a kind of injectivity. Layer, instance and group normalization are all related insofar as they can all be expressed in the same abstract form. For a given input $x$, all three break $x$ up into $K$ parts denoted $\{x|_{S_k}\}_{k = 1, \dots, K}$ such that for each $k = 1,\dots, K$
\begin{align}
    \mu_k &= \frac{1}{m} \sum^m_{i = 1}(x|_{S_k})_i \quad \sigma^2_k = \frac{1}{m} \sum^m_{i = 1}((x|_{S_k})_i - \mu_k)^2\\
    (\hat{x}|_{S_k})_i &= \frac{(x|_{S_k})_i - \mu_k}{\sqrt{\sigma_k^2 + \epsilon}} \quad   M(x)|_{S_i} = \gamma_k (\hat{x}|_{S_k})_i + \beta_k.
\end{align}
The differences between the three normalization are how $\{S_k\}_{k = 1,\dots, K}$ is chosen. For layer normalization $K = 1$ and the normalization is applied to the entire input signal. For instance normalization, there is one $S_k$ for each channel, and the $x|_{S_k}$ restricts $x$ to just one channel of inputs, that is the normalization is done channel-wise. Group normalization is part way between these two, where $k$ is less than the number of channels, and channels are batched together.

In any case, for any of these normalization methods, they are all scalar-augmented injective normalization where 
\begin{align}
    M_\ell(x) = \tilde M_\ell(x; \{\sigma_{k,\ell}, \mu_{k,\ell}\}_{k = 1,\dots, K}).
\end{align}
Thus, by Lemma \ref{lem:restr-injec-scalar-aug-norm} their corresponding deep networks are all injective, provided that for each $\ell$, $\{\sigma_{k,\ell}, \mu_{k,\ell}\}_{k = 1,\dots, K}$ is saved.

\subsection{Pooling Operations}

Although pooling may have a different aim than typical normalization, we consider it in this section, as it is mathematically similar to (\ref{eqn:normalized-deep-network}). Pooling is similar to layer, instance and group normalization in the sense that they partition the input space into $K$ disjoint pieces, and then output a weighted average upon each piece. Specifically, if $M_p(x) \colon \Rea^n \rightarrow \Rea^K$ where for $k = 1, \dots, K$, 
\begin{align}
    M_p(x)|_{S_k} = \norm{x|_{S_k}}_p
\end{align}
where  $\norm{\cdot}_p$ is the discrete $p$ norm of $x$ restricted to the set $S_k$. For $p = 1$ this is the mean of the absolute value, for $p = 2$ this is the Euclidean mean and for $p = \infty$ it is the maximum of the absolute value. The injectivity of this operation in the cases where $p = 1,2,\infty$ is considered in the work \cite{bruna2013signal}.

\section{Proofs of Theorem \ref{thm:fun-approx-by-injective-nn}} \label{sec:proof-of-fun-approx}

Before we present the proof of Theorem \ref{thm:fun-approx-by-injective-nn}, we present the following stand-alone lemma which is used during the proof of Theorem \ref{thm:fun-approx-by-injective-nn}. This lemma shows that, generically, an injective piecewise-affine neural network followed by an orthogonal projector is still injective, provided that the orthogonal projector projects onto a space of sufficient dimension.

Let $\mathcal V(k,D)$ denote the set of $k$-tuples $(v_1,v_2,\dots,v_k)$ where $v_j$ are orthonormal vectors in $\R^D$. Such vectors span a $k$-dimensional linear space. Furthermore, let $G(k,D)$ denote the set of $k$-dimensional linear subspaces of $\Rea^D$, and for $V\in G(k,D)$, $V=\hbox{span}(v_1,v_2,\dots,v_k)$, let $P_V=P_{(v_1,v_2,\dots,v_k)}:\Rea^D\to \Rea^D$ be an orthogonal projection which image is the space $V$. As the dimension of the orthogonal group $O(k)$ is $k(k-1)/2$ and by \citep{Milnor}, the set $G(k,D)$, called the Grassmannian, is a smooth algebraic variety, of dimension $k(D-k)$ and the dimension of $\mathcal V(k,D)$ is  
$k(D-k)+k(k-1)/2=k(2D-k-1)/2.$
When $V,W\in  G(k,D)$ are two $k$-dimensional linear subspaces of $\R^D$, we define their distance in $G(k,D)$ to be the operator norm
\begin{equation}\label{Grasmannian distance}
    d_{G(k,D)}(V,W)=\|P_V-P_W\|
\end{equation}
where $P_V$ and $P_W$ are orthogonal projections in $\R^D$ onto $V$ and $W$, respectively (see \cite{Mattila}, p. 49.

\begin{lemma}\label{lem: projections}
    Let $H_\theta\in \mathcal N\mathcal N(n,D)$, $D>m\geq 2n+1$ be a neural network such that $H_\theta:\Rea^n\to \Rea^{D}$ is injective. Let $X_\theta=\{V\in G(m,D):\ P_V\circ H_\theta:\Rea^n\to \Rea^{D}\hbox{ is injective}\}$. Then the set $X_\theta$ is an intersection of countably many  open and dense subsets of $G(m,D)$, that is, elements of  $X_\theta$ are generic. Moreover, the $m(D-m)$ dimensional Hausdorff measure of the complement of $X_\theta$ in $G(m,D)$ is zero.
\end{lemma}
    
Note that for $V\in X_\theta$, we have $P_V\circ H_\theta\in \mathcal N\mathcal N(n,D)$.
\begin{proof}
    We use that fact that $H_\theta:\Rea^n\to \Rea^D$ is injective and locally Lipschitz-smooth.  Recall that $f:\Rea^n\to \Rea^m$ is locally Lipschitz if for any compact set $\cC\subset \Rea^n$ there is $L_\cC>0$  such that $\mattimnorm{f(x)-f(y)}\le L_\cC\mattinnorm{x-y}$  for all $x,y\in \cC$.
    
    Let $v\in \Rea^m$ be a unit vector and let $Q_v:\Rea^D\to \Rea^D$  be the projection $Q_v(z)=z-(z\,\cdot v)v$. Let
    \begin{align*}
      \mathscr{P}=\{(x,y)\in \Rea^n\times \Rea^n:\ x\not =y\}.  
    \end{align*}
    As $H_\theta:\Rea^n\to \Rea^D$ is an injection, we can define $\mathbb S^{D-1}=\{w\in \Rea^D:\ \mattidnorm{w}=1\}$ and the map
    \begin{align} \label{s_theta formula}
         s_\theta:\mathscr{P}\to \mathbb S^{D-1},  \quad
         s_\theta(x,y)=\frac {H_\theta(x)-H_\theta(y)}{\mattidnorm{H_\theta(x)-H_\theta(y)}}.
    \end{align}
    Observe that $s_\theta(\mathscr{P})==\{s_\theta(x,y):\ (x,y)\in \mathscr{P}\}=-s_\theta(\mathscr{P})=\{-s_\theta(x,y):\ (x,y)\in \mathscr{P}\}.$
        
    As observed in the proof of the Whitney's embedding theorem \citep[Chapter 2, Theorem 3.5]{hirsch2012differential}, the map  $Q_w\circ H_\theta:\Rea^n\to \Rea^D$ is an injection when $w\not \in s_\theta(\mathscr{P})$. To see this,   assume that $w\in \mathbb S^{D-1}$  satisfies $w\not \in s_\theta(\mathscr{P})$, that is, $w$ is not in the image of $s_\theta$. Then, if there are $x,y\in  \Rea^n$, $x\not =y$ such that $Q_wH_\theta(x)=Q_wH_\theta(y)$, we see that there is $t\in \Rea$ such that $H_\theta(x)-H_\theta(y)=tw.$ By changing the roles of $x$ and $y$, in needed, without loss of generality we may assume that $t\ge 0$. As $w$ is a unit vector and $H_\theta(x)\not =H_\theta(y)$, this yields $t=\mattidnorm{H_\theta(x)-H_\theta(y)}\not =0$  and $$w=(H_\theta(x)-H_\theta(y))/t=    (H_\theta(x)-H_\theta(y))/{\mattidnorm{H_\theta(x)-H_\theta(y)}}=s_\theta(x,y),$$ which is in contradiction with the assumption that $w\not \in s_\theta(\mathscr{P})$. Hence, $w\not \in s_\theta(\mathscr{P})$ yields that $Q_w\circ H_\theta:\Rea^n\to \Rea^D$ is an injection.
        
    We consider the image $s_\theta(\mathscr{P})\subset \mathbb S^{D-1}$. For $h>0$, let
    \begin{align*}
        \mathscr{P}_h=\{(x,y)\in \overline B^n(0,h^{-1})\times \overline B^n(0,h^{-1}):\ \mattidnorm{H_\theta(x)-H_\theta(y)}\ge h\}.
    \end{align*} 
    As $H_\theta:\overline B^n(0,h^{-1})\to \Rea^m$ is Lipschitz-smooth with some Lipschitz constant $L_h$, we see that the map $s_\theta:\mathscr{P}_h\to \mathbb S^{D-1}$ is Lipschitz-smooth. Since the set $\mathscr{P}$ has the Hausdorff dimension $2n$ and the map $s_\theta:\mathscr{P}_h\to \mathbb S^{D-1}$ is  Lipschitz-smooth, the Hausdorff dimension of the set $s_\theta(\mathscr{P}_h)$ is at most $2n$, see e.g. \citep[p. 26]{morgan2016geometric}. The set $\mathscr{P}$ is the union of all sets $\mathscr{P}_{h_j}$, where $h_j=1/j$ and $j\in \mathbb{Z}$. By \citep[p.\ 59]{Mattila} the Hausdorff dimension of a  countable union of sets $S_j$ is the supremum of the Hausdorff dimension of the sets $S_j$. Hence $s_\theta(\mathscr{P})=\bigcup_{j=1}^\infty s_\theta(\mathscr{P}_{h_j})$ has the Hausdorff dimension less or equal $2n$.
     
    Since the dimension $D-1$ of $\mathbb S^{D-1}$ is strictly larger than $2n$, we see that the set  $s_\theta(\mathscr{P}_{h_j})$ is closed, its complement is an open and dense set, and thus the set $Y_1(\theta):=\mathbb S^{D-1} \setminus s_\theta(\mathscr{P})$ is an intersection of countably many open and dense sets.
        
    Observe that as $Q_{w_1}$ is a linear map, the map $Q_{w_1}\circ H_\theta$ is also a neural network that belongs in $ \mathcal N\mathcal N(n,D)$, and we can denote $Q_{w_1}\circ H_{\theta}=H_{\theta_1}$ with some parameters $\theta_1$. Thus  we can repeat the above arguments using the map  $Q_{w_1}\circ H_\theta:\Rea^n\to \hbox{span}(w_1)^\perp\equiv \Rea^{D-1}$ instead of $H_\theta:\Rea^n\to \Rea^D$. Repeating the above arguments $D-m$ times, can choose orthonormal vectors $w_j\in  \mathbb S^{D-1}$,  $j=1,2,\dots, D-m$, and sets $Y_j(\theta,w_1,\dots,w_{j-1})\subset  \mathbb S^{D-1}\cap (w_1,\dots,w_{j-1})^\perp$, which   $D-j$ dimensional Hausdorff measures vanish and which complements are intersections of countably many open and dense sets. Let $B_\theta$ to be the set of all n-tuples $(w_1,\dots,w_{D-m})$ where $w_1\in Y_1(\theta)$ and $w_j\in Y_2(\theta,w_1,\dots,w_{j-1})$ for all $j=2,\dots, D-m$. Note that all such vectors $w_j$, $j=1,2,\dots, D-m$ are orthogonal vectors spanning a $D-m$ dimensional vector space $V$, and the map
    \begin{align}\label{combined projection}
        & P_V\circ H_\theta:\Rea^n\to \Rea^D,\quad \hbox{where }P_V=Q_{w_{D-m}}\circ\dots\circ Q_{w_2}\circ Q_{w_1}
    \end{align}
    is injective. By the above, construction  the $(D-(m+1)/2)m$ dimensional Hausdorff measure of the complement of $B_\theta$ in $\mathcal V(D,m)$ is zero and $B_\theta$ is generic set. As the dimension of the set of the orthogonal basis in a $m$-dimensional vector space is $(m-1)m/2$, we obtain the claim.
\end{proof}

Now that this lemma is established, we proceed to the proof of Theorem \ref{thm:fun-approx-by-injective-nn}.

Our proof is divided into two parts. First we prove the universality part (the statement that an $N_\theta$ exists), and then show that we can make a small perturbation to the $N_\theta$ to obtain $N'_{\theta'}$, with the prescribed inverse Lipschitz Constant.

To prove the universality result, we combine the approximation results for neural networks and the low regularity version of the generic orthogonal projector technique used to prove the easy version of the Whitney's embedding theorem \cite[Chapter 2, Theorem 3.5]{hirsch2012differential} that shows that a $C^2$-smooth manifold of dimension $n$ can be embedded in $\Rea^{2n+1}$ with an injective, $C^2$-smooth map. To prove the result,  we first approximate $f(x)$ by a ReLU-type neural network that is only Lipschitz-smooth, so the graph of the map ${{F}}_\theta$ is only a Lipschitz-manifold. We note that limited regularity often causes significant difficulties for embedding results, as for example for the  Lipschitz-smooth manifolds it is presently known only that a $n$-dimensional manifold can be embedded (without preserving distances) in the Euclidean space $\Rea^N$ of dimension $N=(n+1)^2$, and the classical Whitney problem, whether a  $n$-dimensional manifold can be embedded in $\Rea^{2n+1}$, is still open \cite{Luukkainen,LIPmap}. Due to this lack of smoothness, we recall the details how this generic projector technique works.

\begin{proof}[Proof of Theorem \ref{thm:fun-approx-by-injective-nn}, universality result]
Let $\varepsilon>0$ and $\cC\subset \Rea^n$ be a compact set. As a continuous function $f:\R^n\to \R^n$ can be uniformly approximated in a compact set by a $C^\infty$-smooth function (see e.g.\ \cite [Thm. 2.29]{Adams}), we can without loss of generality assume that $f$ smooth and therefore a locally Lipschitz function.
    
By classical results of approximation theory for shallow neural networks, see \cite{hornik1991approximation,leshno1993multilayer,pinkus1999approximation}, for any $L\ge 1$ there are $\vec n$  and a neural network ${{F}}_\theta\in \mathcal N\mathcal N(n,m,L,\vec n)$ such that
    \begin{align}\label{a priori estimate}
        \mattimnorm{f(x)-{{F}}_\theta(x)}\leq \frac 12 \varepsilon_1,\quad\hbox{for all }x\in \cC.
    \end{align}
We note that by using recent results for deep neural networks, e.g. by \citep{yarotsky2017error}, one can obtain efficient estimates on how a given accuracy $\varepsilon_1$ can be obtained using sufficiently large
$L$ and $\vec n$. Our aim is the perturb ${{F}}_\theta:\Rea^n\to \Rea^m$  so that it becomes injective.  

We assume that $\cC\subset B^n(0,r_1)$, where $B^n(0,r_1)\subset \Rea^n$ is an open ball having centre $0$ and radius $r_1>0$. We denote the closure of this ball by $\overline B^n(0,r_1)$.

Let $D=m+n$, $\alpha>0$, and define a map $H_\theta:\Rea^n\to \Rea^{m+n}$,
\begin{align} \label{eqn:thm-5-proof:h-theta-def}
    H_\theta(x)=(\alpha x,{{F}}_\theta(x))\in  \Rea^n\times \Rea^m=\Rea^D.
\end{align}
Observe that the map $H_\theta:\Rea^n\to \Rea^{D}$ is injective.

Now we continue with the proof Theorem \ref{thm:fun-approx-by-injective-nn} universality result. Let $V_0= \{(0,0,\dots,0)\} \times \Rea^m\subset \Rea^{D}.$ 
By Lemma \ref{lem: projections},
the complement of the set $ X_\theta$ in the manifold $G(m,D)$ has measure zero and thus it does not contain any open subset of the manifold
$G(m,D)$. Thus for any $\varepsilon_0>0$ there is $V\in X_\theta$
such that $d_{G(m,D)}(V,V_0)<\varepsilon_0$, see \eqref{Grasmannian distance}. Below, we use
$$
\varepsilon_0=\min\paren{\frac 1{16\paren{1+\alpha+\|F_\theta\|_{C(\cC)}}} \varepsilon_1,\frac 12}.$$
As the $m$-dimensional vector space $V$ satisfies $V\in X_\theta$, the map $P_V\circ H_\theta$ is injective.
Moreover,
$$
\|(P_V-P_{V_0})\circ H_\theta\|_{C(\cC)}
\leq \|P_V-P_{V_0}\| \cdot \|H_\theta\|_{C(\cC)}<\frac 14 \varepsilon_1.
$$
Moreover, as $\|P_V-P_{V_0}\|<1$, it is show in 
 \cite{Kato},  Section I.4.6, that there 
 is orthogonal matrix $R_V\in O(D)$  maps the subspace $V$ onto the subspace $V_0$.
 The matrix $R_V$ is obtained by considering the matrices $P_V$ and $P_{V_0}$ as linear operators in the complex vector space
 $\mathbb C^D$. Then $P_V$ and $P_{V_0}$ are orthogonal projectors in $\mathbb C^D$ onto the (complixified) linear subspaces
 $\hat V=\mathbb C V$ and  $\hat V_0=\mathbb C V_0$. Then, following  \cite{Kato},  Section I.4.6, formula (4.38), one
 defines the operators $A,B:\mathbb C^D\to \mathbb C^D$ by setting  $A=(P_V-P_{V_0})^2$ and
\begin{eqnarray*}
B=\bigg(\idmat-(P_{V}+P_{V_0})(P_{V_0}-P_V)\bigg)(\idmat-A)^{-1/2},\quad 
(\idmat-A)^{-1/2}=\sum_{n=0}^\infty \binom{-\frac 12}{n}A^n,
\end{eqnarray*}
where $\idmat$ is the identity function
and the series converges as $\|A\|\leq \varepsilon_0^2<\frac 14$.
It is shown in  \cite{Kato},  Section I.4.6, formula (4.2) that the map $B$ satisfies $P_{V_0}=BP_{V}B^{-1}$ and hence
$\hat V_0=\hbox{Ran}(P_{V_0})=B(\hbox{Ran}(P_{V}))=B(\hat V)$. Moreover,
by \cite{Kato}, Section I.6.8, it holds that $B^{-1}=B^*$. 
As $\|A\|\leq \varepsilon_0^2$, we see that 
\begin{eqnarray*}
\|\idmat-B\|
\leq 
 2\varepsilon_0\paren{
\sum_{n=0}^\infty \binom{-\frac 12}{n}\|A\|^n}+
\paren{\sum_{n=1}^\infty \binom{-\frac 12}{n}\|A\|^n}
\leq 2\varepsilon_0\frac 1{1-\varepsilon_0^2}+\frac{\varepsilon_0^2}{1-\varepsilon_0^2}
\leq 4 \varepsilon_0.
\end{eqnarray*}

Let $R_V:\R^D\to \R^D$ be the
restriction of $B$ in $\R^D$. Then the operator 
$R_V\in O(D)$ and $R_V$ maps the subspace $V=
\R^D\cap \hat V$ onto the subspace $V_0=\R^D\cap \hat V_0$. Moreover, it satisfies
\begin{align}
    \|R_{V}-\idmat\|_{\Rea^D\to \Rea^D}=  \|B-\idmat\|_{{\mathbb C}^D\to {\mathbb C}^D}< \frac 1{4(1+\alpha+\|F_\theta\|_{C(\cC)})} \varepsilon_1.
\end{align}
 Let $\pi_0:\Rea^D\to \Rea^m$ be the map $\pi_0(y',y'')=y''$ be the projection to the last $m$ coordinates. Then
\begin{align}
    \|R_{V}\circ P_V\circ H_\theta- P_V\circ H_\theta  \|_{C(\cC)}<\frac 14\varepsilon_1,\quad \quad \pi_0\circ H_\theta=F_\theta
\end{align}
This and \eqref{a priori estimate} imply that the result of the first half of Theorem \ref{thm:fun-approx-by-injective-nn} applies to the neural network $N_\theta=\pi_0\circ R_{V}\circ  P_V\circ H_\theta:\Rea^n\to \Rea^m$.
\end{proof}

Now we prove the latter half of Theorem \ref{thm:fun-approx-by-injective-nn}, the construction of the $N'_{\theta'}$ with the prescribed inverse Lipschitz constant.
    \begin{lemma}\label{lemma: distances to image}
    Assume that ${D_1}\ge 2n+2$ and $n\geq 1$. Let $v_1\in S^{{D_1}-1}$, $\R^n=\bigcup_{j=1}^J \overline B_j$, where $B_j$ are open disjoint sets and let  $H_\theta:\R^n\to \R^{D_1}$ be a continuous map whose restriction to the sets $B_j$ are affine maps, that is $H_{\theta}|_{B_j}(x)=W_jx+b_j$. Further, let $\delta,\rho \in \Rea$ be such that $0 < 8\rho \leq \delta < \pi$, then there exists $c_0\in (0,1)$ such that if
    \begin{align}
        c_0  \frac 1{J^{2/({D_1}-(2n+1))}} \delta\le  \rho
    \end{align}
    then there exists $w\in  \mathbb S^{{D_1}-1}$ such that
    \begin{align}
        \hbox{dist}(w,v_1)<\delta,\quad \hbox{dist}(w,s_\theta(\mathscr{P}))\ge  \rho,
    \end{align}
    where the distances are measured in $\Rea^{D_1}$. In other words, there exists a $w$ that is a slight perturbation of $v$ and is at least $\rho$ away from $s_\theta(\mathscr{P})$.
\end{lemma} 

\begin{proof}  
First, we decribe our proof technique in words. We show Lemma \ref{lemma: distances to image} by using sphere packing arguments (see \cite{schneider2015curvatures}) by bounding the size of various the manifolds $L_{jk}$ (defined in Eqn. \ref{eqn:matti:l-def}), which denote directions that we need to stay away from. We bound the size above by considering overlapping spheres, and then below by non-overlapping spheres of half of the radius. By showing that the size of $s_{\theta}(\mathscr{P})\subset \bigcup_{jk} L_{jk}$ is `small enough,' there is always room to find a $w$ suitable for the lemma that avoids being too close to $s_{\theta}(\mathscr{P})$.

We write 
$$
s_\theta(x,y)=P(r_\theta(x,y))
$$
where $P(z)=z/\norm{z}_{D_1}$, $P:\Rea^{D_1}\to \mathbb S^{{D_1}-1}$
and
\begin{align} 
    &r_\theta:\mathscr{P}\to \mathbb \Rea^{{D_1}},  \quad
     r_\theta(x,y)={H_\theta(x)-H_\theta(y)}.
\end{align}
Observe that as $H_\theta$  is a piecewise affine map, so is also $r_\theta$. By our assumptions, 
\begin{align*}
    \mathscr{P}=\bigcup_{j,k=1}^J (B_j\times B_k)\cap \mathscr{P}
\end{align*}
and $r_\theta$ is affine in all sets $(B_j\times B_k)\cap \mathscr{P}$. Let $\mattidifmat_{jk}$ be linear space that is spanned by the set
\begin{align}
\{W_jy-W_ky'+b_j-b_k:\ y,y'\in \Rea^n,\ t\in \Rea\},\quad j,k=1,2,\dots,J.\hspace{-1cm}
\end{align}
Note that the dimension $\mattidifmat_{jk}$ is at most $2n+1$. To simplify the notations below, we consider linear subspaces $T_{jk}\subset \Rea^{D_1}$  that have dimension $2n+1$ and satisfy $\mattidifmat_{jk}\subset T_{jk}$. Then for all $j$ and $k$ the sets $r_\theta((B_j\times B_k)\cap \mathscr{P})$  and $s_\theta((B_j\times B_k)\cap \mathscr{P})=P(r_\theta((B_j\times B_k)\cap \mathscr{P}))$  are subsets of $T_{jk}$. Let
\begin{align}
    \label{eqn:matti:l-def}
    L_{jk}=T_{jk}\cap \mathbb S^{{D_1}-1}.
\end{align}
Note that $s_\theta(\mathscr{P}) \subset \bigcup_{jk}L_{jk}$. Observe also that as  $L_{jk}$ is the intersection of a linear subspace $T_{jk}\subset \Rea^{D_1}$ of dimension  $2n+1$ and the unit sphere, we have that  the dimension of $L_{jk}$  is  $2n+1-1=2n$.
We recall that ${D_1}\ge 2n+2$.  Then dimension of $L_{jk}$  is $2n< {D_1}-1$.

Next we estimate the maximal distance of point, in a metric $\delta$-ball of the sphere $\mathbb S^{{D_1}-1}$, to the union of the sets  $L_{jk}$.

Let $\tilde \rho=\frac 18\rho$ and
$B_\delta=B_{\mathbb S^{{D_1}-1}}(v_1,\delta)$  be a metric ball of radius $\delta$ centred at $v_1$ on the sphere $\mathbb S^{{D_1}-1}\subset \mathbb R^{D_1}$.
Let $\xi_i$, $i=1,2,\dots,K$ be a 
maximal ${\tilde \rho}$-separated subset of $B_\delta$.
Then we see that ${\tilde \rho}$-balls $
B_{\mathbb S^{{D_1}-1}}(\xi_i,{\tilde \rho})$ cover the set $B_\delta$, that is
\begin{align*}
    B_\delta\subset \bigcup_{i=1}^K B_{\mathbb S^{{D_1}-1}}(\xi_i,{\tilde \rho}).
\end{align*}

Thus the sum of the $({D_1}-1)$-volumes of the balls $B_{\mathbb S^{{D_1}-1}}(\xi_i,{\tilde \rho})$ is larger or equal to the volume of the ball $B_\delta$. Using the comparision estimates for manifolds with a bounded sectional curvature \cite[Ch.\ 6, Cor. 2.4]{petersen2006riemannian}, we see that if $c_3=\sin(1)$ then ${c_3^{{D_1}-1}} \delta^{{D_1}-1}\leq K{\tilde \rho}^{{D_1}-1}$, that is 
\begin{align}\label{K at least}
    K\ge {c_3^{{D_1}-1}} \left( \frac{\delta}{\tilde \rho}\right)^{{D_1}-1}.
\end{align}
In other words, a maximal ${\tilde \rho}$-separated set in $B_\delta$ contains at least ${c_3^{{D_1}-1}} ( \delta/{\tilde \rho})^{{D_1}-1}$ points.

On the other hand, we see that ${\tilde \rho}/2$-balls $B_{\mathbb S^{{D_1}-1}}(\xi_i,{\tilde \rho}/2)$ are disjoint subsets of the ball $B_{\mathbb S^{{D_1}-1}}(v_1,2\delta)$ of radius $2\delta$. Thus the sum of volumes of the balls $B_{\mathbb S^{{D_1}-1}}(\xi_i,{\tilde \rho}/2)$ is smaller or equal to the volume of $B_{\mathbb S^{{D_1}-1}}(v_1,2\delta)$. Again by using  \cite[Ch.\ 6, Cor. 2.4]{petersen2006riemannian}, we see that, when we denote $c_4=4c_3^{-1}\ge 1$, we have
\begin{align}\label{K at most all K}
    K\le {c_4^{{D_1}-1}}  \left( \frac{\delta}{\tilde \rho}\right)^{{D_1}-1}.
\end{align}
In other words,  a maximal ${\tilde \rho}$-separated set in $B_\delta$ contains at most ${c_4^{{D_1}-1}}  ( \delta/{\tilde \rho})^{{D_1}-1}$ points.

Next, for any $j,k=1,2,\dots, J$, let $\{x_p^{jk}:\ p=1,\dots,K_{jk}\}$, consisting of $K_{jk}$ points, be a  maximal ${\tilde \rho}/2$ separated subset set in $L_{jk}$. Recall that $L_{jk}=T_{jk}\cap \mathbb S^{{D_1}-1}$ and that $T_{jk}$ are linear spaces of dimension $2n+1$. Thus  $L_{jk}$  are isometric to the $2n$ dimensional sphere $\mathbb S^{2n}$. By the above considerations in \eqref{K at most all K}, we have that
\begin{align}
    \label{K at most}
    K_{jk}\leq c_4^{2n} (2\delta/{\tilde \rho})^{2n},\quad j,k=1,2,\dots, J.
\end{align}

Next we show that if
\begin{align}\label{goal}
    {\tilde \rho}<\frac 18 c_0 \frac 1{J^{2/({D_1}-(2n+1))}} \delta,
\end{align}
then the ${\tilde \rho}/2$-neighborhood of the set $\bigcup_{j,k=1}^J L_{jk}$ does not contain the set $B_\delta$. To show this, assume the opposite that the ${\tilde \rho}/2$-neighborhood of the set $\bigcup_{j,k=1}^J L_{jk}$ contains the set $B_\delta$. Then the ${\tilde \rho}$-neighborhood of the set $Y=\bigcup_{j,k=1}^J\{x_p^{jk}:\ p=1,\dots,K_{jk}\}$ contains the set $B_\delta$, and hence the set $Y$  contains a maximal $2{\tilde \rho}$-separated subset of $B_\delta$. Let $Y_1$ denote this set. Let $K$ be the number of the points in the  set $Y_1$.

From Eqn. \ref{K at least} we have ${c_3^{{D_1}-1}} ( \delta/{\tilde \rho})^{{D_1}-1}\le K$. From $Y_1 \subset Y$ we have $K\leq \sum_{j,k=1}^JK_{jk}$. From Eqn. \ref{K at most} we have $\sum_{j,k=1}^JK_{jk}\le J^2  c_4^{2n} (2\delta/{\tilde \rho})^{2n}$, this in total we have
\begin{align}
    {c_3^{{D_1}-1}} ( \delta/{\tilde \rho})^{{D_1}-1}\le J^2  c_4^{2n} (2\delta/{\tilde \rho})^{2n}
\end{align}
or equivalently,
\begin{align}
    c_3^{{D_1}} ( \delta/{\tilde \rho})^{{D_1}-(2n+1)}\le J^2 (2c_4)^{2n} ,
\end{align}
 or equivalently,
\begin{align}
    {\tilde \rho}\ge \frac{{c_3^{({D_1}-1)/({D_1}-(2n+1))}}} { (2 c_4)^{2n/({D_1}-(2n+1))} } \frac 1{J^{2/({D_1}-(2n+1))}} \delta.
\end{align}
And so
\begin{align}\label{improved estimate}
    {\tilde \rho}\ge  c_0   \frac 1{J^{2/({D_1}-(2n+1))}}\delta,\quad c_0=\frac{c_3^{({D_1}-1)/({D_1}-(2n+1))}}{(2 c_4)^{2n/({D_1}-(2n+1))}} .
\end{align}

Summarizing, the above shows that if
\begin{align}
    {\rho}< c_0 \frac 1{J^{2/({D_1}-(2n+1))}} \delta,
\end{align}
so that also ${\tilde \rho}\leq \rho < c_0 \frac 1{J^{2/({D_1}-(2n+1))}} \delta$, then  there exists $w\in B_\delta$ such that $w$  is not in the ${\tilde \rho}/2$-neighborhood of the set $\bigcup_{j,k=1}^J L_{jk}$ in $\mathbb S^{{D_1}-1}$. As $\bigcup_{j,k=1}^J L_{jk}$ contains the sets $s_\theta(\mathscr{P})$ and for $w_1,w_2\in B_\delta\subset
\mathbb S^{{D_1}-1}\subset \Rea^{D_1}$, 
 \begin{align*}
     \|w_1-w_2\|_{\Rea^{D_1}}\leq \hbox{dist}_{\mathbb S^{{D_1}-1}}(w_1,w_2)\leq 4\|w_1-w_2\|_{\Rea^{D_1}},
 \end{align*}
 the claim follows.

\end{proof}

Next we apply that above lemma to obtain inverse Lipschitz estimates for
the map $Q_w\circ H_\theta$.

\begin{lemma}\label{Lemma inverse Lip estimate}
    Let $H_\theta:\Rea^n\to \Rea^D$ satisfy $\mattidnorm{H_\theta(x)-H_\theta(y)}\ge
    \alpha \mattinnorm{x-y}$ and assume that  $w\in \mathbb S^{D-1}$  satisfies $\hbox{dist}(w, s_\theta(\mathscr{P}))\ge \rho.$
    Then
    \begin{align}
        \mattidnorm{Q_wH_\theta(x)-Q_wH_\theta(y)}\ge  \frac 1{\sqrt 2} \rho\alpha \mattinnorm{x-y}.
    \end{align}
\end{lemma}

\begin{proof}
Assume that $w\in \mathbb S^{D-1}$  satisfies 
\begin{align*}
    \hbox{dist}_{\Rea^D}(w, s_\theta(\mathscr{P}))\ge \rho.
\end{align*}

Let $\beta\ge \frac 1{{\sqrt 2} } \rho\alpha$.
Let $x,y\in  \Rea^n$, $x\not =y$ and
\begin{align*}
    v=Q_wH_\theta(x)-Q_wH_\theta(y)\in \Rea^D.
\end{align*}
Note that $\mattidnorm{v}<\beta\mattinnorm{x-y}$. Then, we see that there is $t\in \Rea$ 
\begin{align}
    H_\theta(x)-H_\theta(y)=tw+v.
\end{align}
Let $r=\mattidnorm{H_\theta(x)-H_\theta(y)}\not =0$. Again, by changing roles of $x$ and $y$, if needed, we can assume that $t \ge 0$. Note that $w\perp v$. 

The above yields 
$$
\frac tr w=(H_\theta(x)-H_\theta(y)-v)/r=s_\theta(x,y)-\frac vr.
$$
Then,
$$
\frac tr w+\frac vr=s:=s_\theta(x,y)\in s_\theta(\mathscr{P}).
$$

Assume first that $\langle s,w\rangle\ge 0$. Then,
as $w$ and $s$  are unit vectors, we have for any $b\in \Rea$
$$
\mattidnorm{s-bw}^2\ge \mattidnorm{s-\langle s,w\rangle w}^2=\mattidnorm{s}^2-2\langle s,\langle s,w\rangle w\rangle+\langle s,w\rangle^2=1-\langle s,w\rangle^2
$$
and as $0\le\langle s,w\rangle\leq 1$,
$$
\mattidnorm{s-w}^2= \mattidnorm{s}^2-2\langle s,w\rangle+\mattidnorm{w}^2=2(1-\langle s,w\rangle)\le 2(1-\langle s,w\rangle^2)
$$
and hence
$
\mattidnorm{s-w}^2\le 2\inf_{b\in \Rea}\mattidnorm{s-bw}^2.
$
Thus,
$$
\mattidnorm{w-s} \le \sqrt 2\inf_{b\in \Rea}\mattidnorm{s-bw}
\le \sqrt 2\mattidnorm{s- \frac tr w}=\sqrt 2\mattidnorm{\frac vr}
.
$$
Hence
\begin{align}\label{kappa eq}
\rho&\leq \hbox{dist}(w,s_\theta(\mathscr{P}))
\\
\nonumber &\leq  \sqrt 2\,
\mattidnorm{\frac vr}
\\ \nonumber
    &\le \frac {\sqrt 2} r \beta\mattinnorm{x-y}.\nonumber
\end{align}
Recall that by our assumptions in the claim
\begin{align}\label{inverse2}
    \mattidnorm{H_\theta(x)-H_\theta(y)} &\ge \alpha \mattinnorm{x-y}
\end{align}
Hence, by \eqref{kappa eq},
\begin{align}\label{kappa eq2}
    \rho<\frac {\sqrt 2} r \beta\mattinnorm{x-y}\le \frac {\sqrt 2}  {\alpha \mattinnorm{x-y}} \beta\mattinnorm{x-y}=\frac {{\sqrt 2} \,\beta} {\alpha }. 
\end{align}
This shows that for all $x,y\in \Rea^n$
\begin{align*}
    \mattidnorm{Q_wH_\theta(x)-Q_wH_\theta(y)}\ge  \frac 1{\sqrt 2} \rho\alpha \mattinnorm{x-y}.
\end{align*}

\end{proof}

Next we apply Lemma \ref{Lemma inverse Lip estimate} for the neural networks $H^{(0)}_\theta=H_\theta$ and
\begin{align}
H^{(j)}_\theta= Q_{w_j}\circ\dots \circ Q_{w_1} H_\theta:\Rea^n\to \Rea^D,\quad
j=1,2,\dots,n,\quad D=n+m,
\end{align} 
with suitably chosen orthogonal vectors $w_j$. We choose these vectors using a recurrent procedure: Let $v_1,\dots,v_{n+m}$ be the unit coordinate vectors in $\Rea^{n+m}$
. Our initial choice of $w_1$ is arbitrary. Once the first $w_1,w_2,\dots, w_{k-1}$ are chosen, we choose $w_k\in \{w_1,w_2,\dots, w_{k-1}\}^\perp\subset \Rea^D$. 
Next, we use the isomorphism
$$
A_{k}:\{w_1,w_2,\dots, w_{k-1}\}^\perp\to \Rea^{D-(k-1)},\quad A_1=I:\Rea^{D}\to \Rea^{D}
$$ 
where $\{w_1,w_2,\dots, w_{k-1}\}^\perp\subset \Rea^{D}$.
The map $A_k$ is used to identify the linear space $\{w_1,w_2,\dots, w_{k-1}\}^\perp\subset \R^D$ with a $D-(k-1)$ dimensional Euclidean space, and it can be also considered as a map that gives coordinates in the space $\{w_1,w_2,\dots, w_{k-1}\}^\perp.$
By applying Lemma \ref{lemma: distances to image} for  the map 
$A_{k}\circ H^{(k-1)}_\theta:\Rea^n\to
 \Rea^{D-(k-1)}$, that is
a map from $n$ dimensional space to an Euclidean space of dimension $d_{k-1}=D-(k-1)\ge n+2n+1-(n-1)\ge 2n+2$. 
In this, we use the values 
$$
\rho_k=\frac {c_0}{2J^{2/((D-(k-1))-(2n+1))}} \delta=\frac {c_0}{2J^{2/(m-n-k)}} \delta
$$ 
and 
\begin{align}
\alpha_{k}=\bigg(\prod_{j=1}^{k}\frac 1{\sqrt 2} \rho_j\bigg)\alpha,
\end{align}
so that $\alpha_{k}=(\frac 1{\sqrt 2} \rho_k)\alpha_{k-1}$ and $\alpha_0=\alpha$.

By replacing the vector $v_1$  by $A_{k}v_{k}$,  $\rho$ by $\rho_k$, and $\alpha$ by $\alpha_{k+1}$, respectively,
Lemma \ref{lemma: distances to image} implies that there is 
a vector $w\in  \mathbb S^{{d_{k-1}}-1}$ such that
    \begin{eqnarray}\label{addded d-formula}
        \hbox{dist}(w,A_{k}v_{k})<\delta,\quad \hbox{dist}(w,s_\theta(\mathscr{P}))\ge  \rho_k,
    \end{eqnarray}
    where distances are measured in  $\Rea^{d_{k-1}}$ and $s_\theta(\mathscr{P})$ is defined in formula \eqref{s_theta formula} with $H_\theta$ replaced by the neural network $A_{k}\circ H^{(k-1)}_\theta:\Rea^n\to
 \Rea^{d_{k-1}}$.
We define 
$w_k=A_{k}^{-1}w$.
Then by Lemma \ref{Lemma inverse Lip estimate},
 $H^{(k)}_\theta=Q_{w_k}\circ  H^{(k-1)}_\theta$ 
 satisfies
\begin{align}
    \label{eqn:matti:h-theta-inv-lip-bound}
\mattidnorm{H^{(k)}_\theta(x)-H^{(k)}_\theta(y)}\ge  \alpha_k \mattinnorm{x-y}.
\end{align}

Observe that as $Q_{w_j}$ are linear maps, the restrictions of  all maps $H^{(j)}_\theta$ are  affine in the the same sets $B_1,\dots,B_J$. Then, apply map the vectors $v_{n+m},\dots,v_{n+1},w_n,\dots,w_1$ to the vectors $v_{n+m},\dots,v_1$. The above yields that
\begin{align}
    \mattidnorm{H^{(n)}_\theta(x)-H^{(n)}_\theta(y)}\ge a \mattinnorm{x-y},   \quad  a=\alpha_n,\quad C_0= \frac 1{2\sqrt 2}  {c_0}.
\end{align} 
Moreover, we have 
\begin{align}
\mattidnorm{H_\theta(x)}\leq (C_R+\alpha R),\quad C_R=\|F_\theta\|_{C(B(R))},
\end{align}
and as $\|Q_{w_j}\|\le 1$,
\begin{align}
\mattidnorm{H^{(j)}_\theta(x)}\leq (C_R+\alpha R).
\end{align}
Moreover, for $j=1,2,\dots,n$ we have
\begin{align}
\mattidnorm{H^{(j)}_\theta(x)-Q_{v_j}H^{(j-1)}_\theta(x)}&\leq \|Q_{w_j}-Q_{v_j}\|\, \mattidnorm{H^{(j-1)}_\theta(x)}
\\ \nonumber
&\leq 2\mattidnorm{w_j-v_j}\, \mattidnorm{H^{(j-1)}_\theta(x)}.
\end{align}

Let $R_{w_n,\dots,w_1,v_n,\dots,v_1}\in \Rea^{(n+m)\times(n+m)}$ be an invertable matrix that maps the vectors $v_{n+m},v_{n+m-1},\dots, v_{n+1},w_n,\dots,w_1$ to $v_{n+m},v_{n+m-1},\dots,v_{n+1},v_n,\dots,v_1$, respectively.
We recall that $\mattinnorm{w_j-v_j}<\delta$. For any $x \in \Rea^{m+n}$, let $x = \sum^{n}_{i = 1}\alpha_i v_i + \sum^{m+n}_{i = n+1}\alpha_i v_i$, then
\begin{align*}
    \mattidnorm{\paren{R_{w_n,\dots,w_1,v_n,\dots,v_1}^{-1} - I}x} &= \mattidnorm{\sum^{n}_{i = 1}\alpha_i \paren{R_{w_n\dots w_1,v_n\dots v_1}^{-1}v_i - v_i} + \sum^{m+n}_{i = n+1}\alpha_i \paren{R_{w_n\dots w_1,v_n\dots v_1}^{-1} v_i - v_i}}\\
    &= \mattidnorm{\sum^{n}_{i = 1}\alpha_i\left (w_i - v_i\right )} \leq \sum^{n}_{i = 1}\mattidnorm{\alpha_i\left (v_i - w_i\right )}\\
    &< \sum^{n}_{i = 1}\alpha_i \delta \leq \delta \norm{x}_{l^1(\Rea^{m+n})}\\
    & \leq \sqrt{m+n}\delta \mattidnorm{x}
\end{align*}
where $\norm{x}_{l^1(\Rea^{m+n})}$ denotes the discrete $L_1$ norm of $x$ in $\Rea^{m+n}$. Further
\begin{align*}
    \mattidnorm{R_{w_n,\dots,w_1,v_n,\dots,v_1} - I} &= \mattidnorm{\paren{I - R_{w_n,\dots,w_1,v_n,\dots,v_1}^{-1}}R_{w_n,\dots,w_1,v_n,\dots,v_1}}\\
    &\leq \mattidnorm{R_{w_n,\dots,w_1,v_n,\dots,v_1}^{-1} - I}\mattidnorm{R_{w_n,\dots,w_1,v_n,\dots,v_1}}\\
    &\leq \frac{\sqrt{m+n}}{1 - \sqrt{m+n}\delta}\delta
\end{align*}
where the last inequality comes from $\mattidnorm{R_{w_n,\dots,w_1,v_n,\dots,v_1}} = \mattidnorm{\paren{I + \paren{R_{w_n,\dots,w_1,v_n,\dots,v_1}^{-1}-I}}^{-1}}$. Since $Q_{v_n}\dots Q_{v_1}H_\theta(x)=\tilde F_\theta(x):=(F_\theta(x),0,\dots,0)$ 
we see that
\begin{align}
\mattimnorm{\tilde F_\theta(x)-H^{(n)}_\theta(x)}&\leq 2n\delta (C_R+\alpha R).
\end{align}
Then,
\begin{align}
    \label{eqn:matti:f-theta-def}
    f_\theta(x)=P_0R_{w_n,\dots,w_1,v_n,\dots,v_1}H^{(n)}_\theta(x)
\end{align}

where $P_0:\Rea^{m}\times \Rea^{n}\to \Rea^m$ maps $P_0(x,z)=x$. The above yields
\begin{align*}
\mattimnorm{f_\theta(x)-F_\theta(x)}&\leq 
2n\delta (C_R+\alpha R) +\frac{\sqrt{m+n}}{1 - \sqrt{m+n}\delta}\delta(C_R + \alpha R) \\
&= \paren{2n + \frac{\sqrt{m+n}}{1 - \sqrt{m+n}\delta}}\delta\paren{C_R + \alpha_R}
\end{align*}
for $x\in B(R)$.

Further, from $\mattidnorm{R} \geq \mattidnorm{I + \paren{R^{-1} - I}}$, the fact that 
$R_{w_n,\dots,w_1,v_n,\dots,v_1}H^{(n)}_\theta(x)$ belongs in the image of $P_0$, and from \eqref{eqn:matti:h-theta-inv-lip-bound} \& \eqref{eqn:matti:f-theta-def} we have
\begin{align}
    \mattimnorm{f_\theta(x)-f_\theta(y)}\ge a \mattinnorm{x-y}, \quad  a= \frac{\alpha_n}{1 + \sqrt{m+n}\delta} \geq \alpha_n.
\end{align} 

Above, using the definition Euler's constant and the harmonic series, we see that
there are $c_6,c_5\in \R_+$ such that 
$$
\ln(m)+c_6\leq \sum_{j=1}^m \frac 1{j}\leq \ln(m)+c_7
$$
and as $\ln(m+1)+c_6-\ln 2\leq
\ln(2m)+c_6-\ln 2\leq
\ln(m)+c_6$ and $\ln(m)+c_7\leq \ln(m+1)+c_7$,  
 we see that  for $C_2=2(c_7+c_6+\ln 2)$
$$
\sum_{j=1}^n \frac 1{m-n-j}
=\sum_{j=1}^{m-n-1} \frac 1j-
\sum_{j=1}^{m-2n-1} \frac 1j
\leq \ln(m-n)-\ln(m-2n)+C_2/2
$$
and
\begin{align*}
    a =\frac{\alpha}{1 + \sqrt{m+n}\delta} C_0^{nD} \delta^nJ^{-2(\sum_{j=1}^n \frac 1{m-n-j})} \ge \alpha C_0^{n(n+m)} \delta^nJ^{-2\ln( (m-n)/(m-2n))-C_2}.
\end{align*}

We choose
\begin{align*}
    \varepsilon_2= \paren{2n + \frac{\sqrt{m+n}}{1 - \sqrt{m+n}\delta}}\delta\paren{C_R + \alpha_R}
\end{align*}
so that $\delta= \paren{ \paren{2n + \frac{\sqrt{m+n}}{1 - \sqrt{m+n}\delta}}\paren{C_R + \alpha_R}}^{-1}\varepsilon_2$ and
\begin{align*}
    &\paren{ \frac {C_0^D}{\paren{2n + \sqrt{m+n}}\paren{C_R + \alpha_R}}\varepsilon_2}^nJ^{-2\ln( (m-n)/(m-2n))-C_2} \alpha\\
    &\leq \paren{ \frac {C_0^D}{\paren{2n + \frac{\sqrt{m+n}}{1 - \sqrt{m+n}\delta}}\paren{C_R + \alpha_R}}\varepsilon_2}^nJ^{-2\ln( (m-n)/(m-2n))-C_2} \alpha\\ &\leq a .
\end{align*}
When we use above $\alpha=1$ and $C_1=C_0/10$, this yields the final part of Theorem \ref{thm:fun-approx-by-injective-nn}.
\bigskip

Lemma \ref{lem: projections} used in the above proof yields also Corollary \ref{cor: decreasing dimension}.

 \begin{proof}[Proof of Corollary \ref{cor: decreasing dimension}]
Observe that a measure $\mu$ that is absolutely continuous with respect to the Lebesgue $\prod_{j=1}^k\Rea^{d_{2j}\times d_{2j-1}}$ is absolutely continuous also with respect to the normalized Gaussian distribution, and that if the set $S=\{F_k\hbox{ is not injective}\}$ has measure zero with respect to the normalized Gaussian distribution, then its $\mu$-measure is also zero. Thus we can assume without loss of generality that the elements of matrices $B_j$  are independent and have normalized Gaussian distributions.

Assume next that we have shown that $F_{j-1}:\Rea^n\to \Rea^{d_{2j-2}}$ is injective almost surely. Then, $f^{(j)}_\theta\circ F_{j-1}:\Rea^n\to \Rea^{d_{2j-1}}$ is injective almost surely. If $d_{2j}>d_{2j-1}$, the matrix $B_j$ is almost surely injective and so is $F_{j}=B_j\circ f^{(j)}_\theta\circ F_{j-1}:\Rea^n\to \Rea^{d_{2j}}$. Thus, it is enough to consider the case when $d_{2j-1}\ge d_{2j}\geq 2n+1$. Then the matrix $B_j:\Rea^{d_{2j-1}}\to\Rea^{d_{2j}}$ has almost surely rank  $d_{2j}$. By using the singular value decomposition, we can write $B_{j}=R^1_jD_jR^2_j$ where $R_j^1\in O(d_{2j}),$ $R_j^2\in O(d_{2j-1})$, $D_j\in \Rea^{d_{2j}\times d_{2j-1}}$  is  matrix which principal diagonal elements are almost surely strictly positive and the other elements are zeros. Let $V_j=(R^2_j)^{-1}(\R^{d_{2j}}\times \{(0,0,\dots,0)\})\subset \R^{d_{2j-1}}$. Let $P_{V_j}$ be an orthogonal projector in $ \R^{d_{2j-1}}$ onto  the space $V_j$  of dimension $d_{2j}$. As the distribution of the matrix $B_j$  is invariant in rotations of the space, so is the distribution of linear space $V_j$ in $G(d_{2j-1},d_{2j})$. By Lemma \ref{lem: projections}, we see that the map $P_{V_j}\circ f^{(j)}_\theta\circ F_{j-1}:\Rea^n\to \Rea^{d_{2j-1}}$ is injective almost surely. This implies that  the map $B_{j}\circ f^{(j)}_\theta\circ F_{j-1}:\Rea^n\to \Rea^{d_{2j}}$ is injective almost surely, that is, the map $F_{j}:\Rea^n\to \Rea^{d_{2j}}$ is injective almost surely. The claim follows by induction.
\end{proof}

In \citep{Baraniuk-Wakin}, see also \citep{Hedge-Wakin-Baraniuk,Iven-Maggioni}, \citep{Broomhead,Broomhead2}, the authors study manifold learning using random projectors. These results are related to the proof of the first, universal, half of Theorem \ref{thm:fun-approx-by-injective-nn} above. Let $H_\theta$ be given by (\ref{eqn:thm-5-proof:h-theta-def}), a ReLU-based neural network whose graph $M\subset \Rea^d$, $d={2n+m}$. When $P_V$ is a random projector in $\Rea^d$ onto a $m$-dimensional linear subspace $V$, the injectivity of  the neural network $P_V\circ H_\theta$ is closely related to the property that $P_V(M)$ is an $n$-dimensional submanifold with a large probability. In \cite{Broomhead,Broomhead2} the authors use Whitney embedding results for $C^2$-smooth manifold for dimension reduction of data. Our proof applies similar techniques for Lipschitz-smooth maps. In \cite{Baraniuk-Wakin,Hedge-Wakin-Baraniuk,Iven-Maggioni}, the authors apply the result that when $M\subset \R^D$ is a submanifold  and $D$  is large enough, a random $m$-dimensional projector $P_V$ satisfies on $M$ the restricted isometry property with a large probability. In this case, $P_V\circ H_\theta$  is not only an injection but its inverse map is also a local Lipschitz map. In this sense, the techniques in \citep{Baraniuk-Wakin,Hedge-Wakin-Baraniuk,Iven-Maggioni} would give improved results to the generic projection technique used in this paper. The results in \citep{Baraniuk-Wakin,Hedge-Wakin-Baraniuk,Iven-Maggioni}, however, require that the dimension $m$ of image space of the map $f:\Rea^n\to \Rea^m$ satisfies $m\ge C\log(\varepsilon^{-1})$, where $\varepsilon$ is the precision parameter in the inequality \eqref{approximation errror}. Our result, Theorem \ref{thm:fun-approx-by-injective-nn} requires only that $m\ge 2n+1$.

\section{Miscellaneous Lemmas}

\begin{lemma}\label{lem:relu-arithmetic:pos-addition}
    Let $x, y, z \in \Rea$ and $z \geq 0$. Then 
    \begin{align}
        \relu(x + z) = \relu(y + z) \implies \relu(x)  = \relu(y) .
    \end{align}
\end{lemma}

\begin{lemma}[Useful Inequalities]
    The following geometric inequalities are useful and have straight forward proofs. Suppose that $a,b,c \in \Rea^n$, then
    \begin{enumerate}
        \item For $i = 1,\dots, k$, let $a_i \in \Rea^m$. If
            \begin{align}
                \sum_{i = 1}^k\norm{a_i}_2^2 \leq \norm{\sum_{i = 1}^k a_i}^2_2 ,
            \end{align}
            then if for each $j = 1, \dots, m$, $a_{i}|_j \cdot a_{i'}|_j \geq 0$ for each pair $i, i' \in [[m]]$, then
            \begin{align} \label{eqn:lem-useful-ineq:1}
                \sum_{i = 1}^k\norm{a_i}_2 \leq \sqrt{k} \norm{\sum_{i = 1}^k a_i}_2 .
            \end{align}
        \item If $a,b \in \Rea^n$ and $\innerprod{a}{b} \geq 0$, then
            \begin{align} \label{eqn:lem-useful-ineq:2}
                \norm{a}^2_2 + \norm{b}^2_2 \leq \norm{a + b}^2_2 .
            \end{align}
    \end{enumerate}
\end{lemma}

\begin{lemma}[Co-linear Additivity of $\relu(W\cdot)$]\label{lem:colinear-additivity}
    Let $W \in \Rea^{m \times n}$, $x_1, x_2\in \Rea^n$, $\ell^{x_1, x_2}(t) = (1-t)x_1 + t x_2$. Let there be
    \begin{align}
        0 = t_1 \leq t_2 \leq \dots \leq t_{n_t-1} \leq t_{n_t} = 1
    \end{align}
    then 
    \begin{align} \label{eqn:lem:colinear-additivity:squared}
        \sum^{n_t-1}_{k = 1} \norm{\relu(W x_{t_k}) - \relu(W x_{t_{k+1}})}_2^2 \leq \norm{\relu(W x_1) - \relu(W x_2)}_2^2.
    \end{align}
    and 
    \begin{align} \label{eqn:lem:colinear-additivity:norm}
        \sum^{n_t-1}_{k = 1} \norm{\relu(W x_{t_k}) - \relu(W x_{t_{k+1}})}_2 \leq \sqrt{n_t} \norm{\relu(W x_1) - \relu(W x_2)}_2.
    \end{align}
\end{lemma}

\begin{proof}
    Let $x_t = \ell^{x_1, x_2}(t)$, then as a function of $t$, the $j$'th component of $\relu(Wx_t)$ is either increasing (if $\innerprod{w_j}{x_2 - x_1} \geq 0$) or decreasing (if $\innerprod{w_j}{x_2 - x_1} \leq 0$). In either case, it is clear that for each $j = 1, \dots, m$
    \begin{align}
        (\relu(W x_1) - \relu(W x_2))|_j \cdot (\relu(W x_2) - \relu(W x_3))|_j \geq 0
    \end{align}
    hence clearly 
    \begin{align}
        \innerprod{(\relu(W x_1) - \relu(W x_2))}{(\relu(W x_2) - \relu(W x_3))} \geq 0
    \end{align}
    thus we can apply (\ref{eqn:lem-useful-ineq:2}) and we obtain (\ref{eqn:lem:colinear-additivity:squared}). Applying (\ref{eqn:lem-useful-ineq:1}) then yields (\ref{eqn:lem:colinear-additivity:norm}).
\end{proof}

\section{Detailed Comparison to Prior Work}

\subsection{Comparison to Bruna \textit{et al.}} \label{sec:comparison-w-bruna}

In \citep[Proposition 2.2.]{bruna2013signal} the authors give a result  invoking a condition similar to our DSS condition (Definition \ref{def:directed-spanning-set}). It also concerns injectivity of a $\relu$ layer in terms of the injectivity of the weight matrix restricted to certain rows. The authors also compute a bi-Lipschitz bound for a layer (similar to our Theorem \ref{thm:inv-lip:global}), though as we show in the following examples their analysis is in some cases not precisely aligned with injectivity.

Their criterion is given in two parts. For a weight matrix, they first define a notion of admissible set which indicates the points where the weight matrix's injectivity must be tested. Injectivity follows provided that the weight matrix is non-singular when restricted to each admissible set. Given a weight matrix $W \in \Rea^{M \times N}$ and bias $b$, the authors say that $\Omega \subset \{ 1,\dots, M\}$ is admissible if
\begin{align}
    \label{eqn:bruna:admissible-def}
    \bigcap_{i \in \Omega} \{ x \colon \innerprod{x}{w_i} > b_i\} \cap \bigcap_{i \not \in \Omega} \{ x \colon \innerprod{x}{w_i} < b_i\}
\end{align}
is not empty. For our analysis we focus on the case when $b \equiv 0$. In this case $\Omega$ is admissible if and only if
\begin{align}
    \label{eqn:bruna:admissable-zero-bias}
    \exists x \in \Rea^{n} \quad \text{such that} \quad \innerprod{x}{w_i} \begin{cases} &> 0 \text{ if } i \in \Omega\\
         &< 0 \text{ if } i \not\in \Omega
     \end{cases}.
\end{align}
Note that the inequality in (\ref{eqn:bruna:admissable-zero-bias}) is strict, unlike (\ref{eqn:deep-network-definition}). If, for example, $W$ has a column that is the zero vector, then there are no admissible $\Omega$. The authors use the notation $\overline{\Omega}$ to denote all admissible sets for a given weight matrix. In their notation $F$ is the transpose of our weight matrix $W$, $F_\Omega$ are the $\Omega$ rows of the weight matrix, $F_\Omega|_{V_\Omega}$ is the subspace generated by the $\Omega$ rows of $W$. The authors also call the $\relu$ function the half-rectification function. $\lambda_{-}(F)$ and $\lambda_{+}(F)$ denote the lower and upper frame bounds of $F$ respectively. The injectivity criterion from \citep{bruna2013signal} is
\begin{proposition} \label{prop:bruna:injectivity}
    Let $A_0 = \min_{\Omega \in \overline{\Omega}} \lambda_{-}(F_\Omega |_{V_\Omega})$. Then the half-rectification operator $M_b(x) = \relu(F^Tx + b)$ is injective if and only if $A_0 > 0$. Moreover, it satisfies
    \begin{align} \label{eqn:bruna:bi-lip}
        \forall x, x', A_0 \norm{x - x'} \leq \norm{M_b(x) - M_b(x')} \leq B_0 \norm{x - x'}
    \end{align}
    with $B_0 = \max_{\Omega \in \overline{\Omega}}\lambda_+(F_\Omega) \leq \lambda_+(F)$.
\end{proposition}

We now show that Proposition \ref{prop:bruna:injectivity} does not precisely align with injectivity of $\relu(W(\cdot))$. We construct a weight matrix for which $A_0 > 0$, but does not yield an injective $\relu(W(\cdot))$. If 
\begin{align}
    W = \begin{bmatrix}
    1 & 0\\
    0 & 1\\
    -1 & 0\\
    \end{bmatrix}
\end{align}
then clearly $\relu(Wx)$ is not injective (for all $\alpha < 0, \relu(W\begin{bmatrix} 0\\\alpha\end{bmatrix}) = \overline{0}$). The only admissible sets are $\bar\Omega = \left \{\{1\}, \{3\},\{1,2\},\{2,3\}\right \}$ (notably $\{ 2\}$ is not admissible). $W$ is full rank on all $\Omega \in \bar \Omega$, so $A_0 > 0$ so Proposition \ref{prop:bruna:injectivity} implies that $\relu(W(\cdot))$ is injective. Now consider the case when
\begin{align}
    \label{eqn:burna:injec-counter-example:injec}
    W = \begin{bmatrix}
        B\\
        -DB\\
        0
    \end{bmatrix}
\end{align}
where $B$ is a basis of $\Rea^n$, $D$ is a strictly positive diagonal matrix, and $0$ is the zero row vector. From Corollary \ref{cor:minimal-expansivity}, $W$ satisfies Theorem \ref{thm:relu-w-injectivity}, and so $\relu(W(\cdot))$ is injective. On account of the zero row vector in (\ref{eqn:burna:injec-counter-example:injec}), $\forall x \in \Rea^n$, $\innerprod{x}{0} = 0$ so there are no $\Omega$ that are admissible $\Omega$ according to (\ref{eqn:bruna:admissible-def}). Thus $A_0$ is undefined.

Now we construct an example of a $W$ and $x,x' \in \Rea^n$ for which \\$A_0\norm{x - x'} > \norm{\relu(Wx) - \relu(Wx')}$. Let 
\begin{align}
    W = \begin{bmatrix}
    1& 0\\
    0& 1\\
    -1& 0\\
    0& -1
    \end{bmatrix} ,\quad x = \frac{1}{\sqrt{2}}\begin{bmatrix} 1 \\ 1\end{bmatrix} ,\quad x' = \frac{1}{\sqrt{2}}\begin{bmatrix} -1 \\ 1\end{bmatrix}.
\end{align}
Clearly on every admissible set $\lambda_-(W_\Omega|_{V_\Omega}) = 1$, so
\begin{align}
    \relu(Wx) &= \frac{1}{\sqrt{2}}\relu\left( \begin{bmatrix} 1\\1\\-1\\-1\end{bmatrix} \right ) = \frac{1}{\sqrt{2}}\begin{bmatrix} 1\\1\\0\\0\end{bmatrix} ,
    \nonumber \\ 
    \relu(Wx') &= \frac{1}{\sqrt{2}}\relu\left( \begin{bmatrix} -1\\1\\1\\-1\end{bmatrix} \right ) = \frac{1}{\sqrt{2}}\begin{bmatrix} 0\\1\\1\\0\end{bmatrix} ,
\end{align}
hence,
\begin{align}
    \norm{\relu(Wx) - \relu(Wx')} = \frac{1}{\sqrt{2}} \norm{\begin{bmatrix} 1\\0\\-1\\0\end{bmatrix}} = 1
\end{align}
and
\begin{align}
    \norm{x - x'} = \frac{1}{\sqrt{2}}\norm{\begin{bmatrix} 2 \\ 0\end{bmatrix}} = \sqrt{2}.
\end{align}
From this we have
\begin{align}
     \sqrt{2} = A_0\norm{x - x'} > \norm{\relu(Wx) - \relu(Wx')} = 1.
\end{align}
On the other hand, substituting this into (\ref{eqn:inv-lip:global}) yields
\begin{align}
    \frac{C}{\sqrt{m}}\norm{x - x'} = \frac{1}{\sqrt{8}}\sqrt{2} = \frac{1}{2} \leq 1 ,
\end{align}
which does hold, suggesting that the lower bound in Proposition 1 is not pessimistic enough.

\subsection{Relationship to \citep{mallat2018phase}}

In \citep{mallat2018phase} the authors consider a construction analogous to our convolutional construction (in Definition \ref{def:discrte-convolution}) defined on a continuum (i.e. infinite-dimensional function defined on an interval) rather than on a vector (i.e. discrete finite dimensional function) defined on a subset of $\Rea^n$. The authors posit that CNNs first learn a layer of filters localized in frequency varied in phase. The authors also show that a a $\relu$ activation function acts as a filter on the phase of the convolution of the filters against the input signal and that, provided that the filters are sufficiently different in phase and satisfy a frame condition then the layer is bi-Lipschitz, and hence is injective. Their analysis a particularized version of ours, and can be straight forwardly subsumed by our work.

The frame condition is given by Proposition 2.6 in \citep{mallat2018phase} that the weight matrix must satisfy in order ensure that $W$ is invertible and stable. In the notation of \citep{mallat2018phase} the filters $\hat \psi_\lambda$ are analogous to the Fourier transform of the kernels in Definition \ref{def:conv-operator}, and the condition in (2.25) in \citep{mallat2018phase} is one natural way to generalize the notion of a basis to a continuous signal. Hence, Proposition 2.6 in \citep{mallat2018phase} can be loosely interpreted as a statement that the kernels in a given layer of width $P$ form a basis of $\Rea^P$.

The second condition is that kernels are given in terms of belonging to a family, and members of this family are related to each other in the sense that members of the same family are centered in the same Fourier domain, and act as phase offsets of differing phase. Equation 2.14 of \citep{mallat2018phase} describes that a phase filter $H \colon \bbC \times [0, 2,\pi]\rightarrow \bbC$ is defined by

\begin{align}
    \label{eqn:mallat:def-of-h}
    \forall z \in \bbC, \alpha \in [0, 2\pi], \quad  Hz(\alpha) = |z| h(\alpha - \varphi(z))
\end{align}
where $|z|$ standard modulus of a complex number, $\varphi(z)$ is the complex phase, and $h(\alpha) = \relu(\cos(\alpha))$. If we consider just the value of $Hz(0)$ and $Hz(\pi)$, then we find

\begin{align}
    \label{eqn:mallat:hz-zero}
    Hz(0) =& |z|\relu(\cos(\varphi(z))) = \relu(|z| \cos(\varphi(z))) = \relu(\mathscr{R}z) ,\\
    \label{eqn:mallat:hz-pi}
    Hz(\pi) =& |z|\relu(-\cos(\varphi(z))) = \relu(-|z| \cos(\varphi(z))) = \relu(-\mathscr{R}z) ,
\end{align}
where $\mathscr{R}z$ is the real part of $z$. If (as in \citep{mallat2018phase}) $z$ is given by 
$z = x \star c_k$, where $x$ is a real signal and $\star c_k$ denotes the convolution against a kernel $c_k$, then (\ref{eqn:mallat:hz-zero}) and (\ref{eqn:mallat:hz-pi}) imply that
\begin{align}
    Hz(0) &= \relu(x\star \mathscr{R}c_k) ,\\
    Hz(\pi) &= \relu(x\star(-\mathscr{R}c_k)) ,
\end{align}
that is, that for every kernel $c_k$ in a layer, the kernel $-c_k$ is also in that layer.-

Combining the two logical conditions above implies that the kernels of width $c_k$ form a basis of $\Rea_P$ and that for every kernel $c_k$ there is also a kernel $-c_k$. Together these two (by Corollary \ref{cor:minimal-expansivity}) that the $c_k$ form a DSS of $\Rea^P$, and thus by Theorem \ref{thm:suff-conv-injectivity} the entire layer is injective.

\section{Architecture details for experiments} \label{app:arch_dets}

\textbf{Generator network}: We train a generator with 5 convolutional layers. The input latent code is 256-dimensional which is treated by the network as a $1\times1\times256$ size tensor. The first layer is a transposed convolution with a kernel size of $4\times 4$ with stride $1$ and $1024$ output channels. This is followed by a leaky ReLU. We follow this up by $3$ conv layers each of which halve the number of channels and double the image size (i.e. we go from $N/2 \times N/2 \times C$ to $N\times N\times C/2$ tensor) giving an expansivity of two, the minimum required for injectivity of ReLU networks. Each of these 3 convolution layers has kernel size $3$, stride $2$ and is followed by the ReLU activation. These layers are made injective by having half the filters as $w$ and the other half as $-s^2w$. Here, $w$ and $s$ are trainable parameters. The biases in these layers are kept at zero. We do not employ any normalization schemes. Lastly, we have a convolution layer at the end to get to $3$ channels and required image size. This layer is followed by the sigmoidal activation. We compare this to a regular GAN which has all the same architectural components including nonlinearities except the filters are not chosen as $w$ and $-s^2w$ and we also allow biases. We found that our injective architecture was able to achieve the same performance visually as a `off the shelf' \cite{bora2017compressed} non-injective counterpart.

\textbf{Critic network}: The discriminator has 5 convolution layers with $128, 256, 512, 1024$ and $1$ channels per layer. Each convolution layer has $4\times 4$ kernels with stride $2$. Each layer is followed by the leaky-ReLU activation function. The last layer of the network is followed by identity.

\textbf{Inference network}: The inference network has the same architecture as the first 4 convolution layers of the discriminator. This is followed by 3 fully-connected layers of size $512, 256$ and $256$. The first 2 fully-connected layers have a Leaky ReLU activation while the last layer has identity activation function. The inference net is trained in tandem with the GAN.

We use the Wasserstein loss with gradient penalty~\citep{gulrajani2017improved} to train our networks. We train for $40$ epochs on a data set of size $80000$ samples. We use a batch size of $64$ and Adam optimizer for training with learning rate of $10^{-4}$. 

We report FID~\citep{heusel2017gans} and Inception score~\citep{salimans2016improved} using $10000$ generated samples. The standard deviation was calculated using 5 sets of $10000$ generated samples. In order to calculate the mean and covariance of generated distributions, we sample 50000 codes.

\end{appendices}

\end{document}

%% file: common_packages.tex
\usepackage{amssymb}
\usepackage{amsmath}
\usepackage{graphicx}
\usepackage[mathscr]{euscript}
\usepackage{subcaption}
\usepackage[utf8]{inputenc}
\usepackage{float}
\usepackage{multicol}
\usepackage{lipsum}
\usepackage{mathtools}
\usepackage{soul}
\setstcolor{red}
\usepackage{bm}
\usepackage{amsthm}
\usepackage{xargs}                      
\usepackage{xcolor}
\usepackage{bbm}
\usepackage{hyperref}
\usepackage{wrapfig}

%% file: common_commands.tex
\usepackage{mathtools}

\newcommand{\Rea}{\mathbb{R}}
\newcommand{\R}{\mathbb{R}}
\newcommand{\sign}{\operatornamewithlimits{sign}}

\newcommand{\norm}[1]{\left\lVert#1\right\rVert}

\newcommand{\Span}{\operatornamewithlimits{span}}
\newcommand{\innerprod}[2]{\left \langle#1,#2\right\rangle}

\newcommand{\Nat}{\mathbb{N}}

\newcommand{\wassone}[2]{\wassone \left( #1, #2\right)}

\newcommand{\relu}{\operatorname{ReLU}}
\newcommand{\lrelu}{\operatorname{LReLU}}
\newcommand{\set}[1]{\left \{#1\right\}}
\newcommand{\abs}[1]{\left | #1\right |}

\makeatletter
\def\mathcolor#1#{\@mathcolor{#1}}
\def\@mathcolor#1#2#3{%
  \protect\leavevmode
  \begingroup
    \color#1{#2}#3%
  \endgroup
}
\makeatother

\newcommand\restr[2]{{
  \left.\kern-\nulldelimiterspace 
  #1 
  \vphantom{
  \big|
  } 
  \right|_{#2} 
  }}